\documentclass{article} 
\usepackage{iclr2023_conference,times}

\usepackage{amsmath,amsfonts,bm}
\usepackage{mathrsfs}

\def\eqref#1{equation~\ref{#1}}

\def\1{\bm{1}}

\def\vepsilon{{{\epsilon}}}

\DeclareMathAlphabet{\mathsfit}{\encodingdefault}{\sfdefault}{m}{sl}
\SetMathAlphabet{\mathsfit}{bold}{\encodingdefault}{\sfdefault}{bx}{n}

\def\gN{{\mathcal{N}}}

\def\gP{{\mathcal{P}}}

\def\gW{{\mathcal{W}}}

\def\sR{{\mathbb{R}}}

\newcommand{\E}{\mathbb{E}}

\let\log\relax
\DeclareMathOperator{\log}{ln}

\newcommand{\KL}{\mathrm{KL}}

\DeclareMathOperator*{\argmin}{arg\,min}

\usepackage{hyperref}
\usepackage{url}
\usepackage{minipage}
\usepackage{graphicx}
\usepackage{wrapfig}

\usepackage{algorithm}
\usepackage{multirow}
\usepackage{algpseudocode}
\usepackage{adjustbox}
\usepackage{scalerel}
\usepackage{lscape}

\def\cossqablation{1}

\title{Denoising Diffusion Samplers}

\author{Francisco Vargas$^{1}$\thanks{Work done while at DeepMind}~~, Will Grathwohl$^2$ \& Arnaud Doucet$^2$ \\
$^{1}$ University of Cambridge, $^{2}$ DeepMind
}

\usepackage{listings}
\usepackage{xcolor}
\usepackage{amsthm}
\usepackage{booktabs}

\makeatletter
\newtheorem*{rep@theorem}{\rep@title}
\newcommand{\newreptheorem}[2]{%
\newenvironment{rep#1}[1]{%
 \def\rep@title{#2 \ref{##1}}%
 \begin{rep@theorem}}%
 {\end{rep@theorem}}}
\makeatother

\newreptheorem{proposition}{Proposition}
\newreptheorem{corollary}{Corollary}
\newreptheorem{theorem}{Theorem}
\newreptheorem{lemma}{Lemma}
\newreptheorem{observation}{Observation}
\newreptheorem{remark}{Remark}

\newtheorem{proposition}{Proposition}
\newtheorem{corollary}{Corollary}

 \iclrfinalcopy 

\begin{document}

\maketitle

\begin{abstract}
Denoising diffusion models are a popular class of generative models providing state-of-the-art results in many domains. One adds gradually noise to data using a diffusion to transform the data distribution into a Gaussian distribution. Samples from the generative model are then obtained by simulating an approximation of the time-reversal of this diffusion initialized by Gaussian samples. Practically, the intractable score terms appearing in the time-reversed process are approximated using score matching techniques. We explore here a similar idea to sample approximately from unnormalized probability density functions and estimate their normalizing constants. We consider a process where the target density diffuses towards a Gaussian. Denoising Diffusion Samplers (DDS) are obtained by approximating the corresponding time-reversal.  While score matching is not applicable in this context, we can leverage many of the ideas introduced in generative modeling for Monte Carlo sampling. Existing theoretical results from denoising diffusion models also provide theoretical guarantees for DDS. We discuss the connections between DDS, optimal control and Schr\"odinger bridges and finally demonstrate DDS experimentally on a variety of challenging sampling tasks.
\end{abstract}

\section{Introduction}
Let $\pi$ be a probability density on $\mathbb{R}^d$ of the form
\begin{equation}
    \pi(x)=\frac{\gamma(x)}{Z},\qquad Z=\int_{\mathbb{R}^d} \gamma(x) \mathrm{d}x,
\end{equation}
where $\gamma:\mathbb{R}^d \rightarrow \mathbb{R}^{+}$ can be evaluated pointwise but the normalizing constant $Z$ is intractable. We are here interested in both estimating $Z$ and obtaining approximate samples from $\pi$.

A large variety of Monte Carlo techniques has been developed to address this problem. In particular Annealed Importance Sampling (AIS) \citep{neal2001annealed} and its Sequential Monte Carlo (SMC) extensions \citep{del2006sequential} are often regarded as the gold standard to compute normalizing constants. Variational techniques are a popular alternative to Markov Chain Monte Carlo (MCMC) and SMC where one considers a flexible family of easy-to-sample distributions $q^{\theta}$ whose parameters are optimized by minimizing a suitable metric, typically the reverse Kullback--Leibler discrepancy $\KL(q^{\theta}||\pi)$. Typical choices for $q^{\theta}$ include mean-field approximation \citep{Wainwright:2008} or normalizing flows \citep{papamakarios2019normalizing}. To be able to model complex variational distributions, it is often useful to model $q^{\theta}(x)$ as the marginal of an auxiliary extended distribution; i.e. $q^{\theta}(x)=\int q^{\theta}(x,u)\mathrm{d}u$. As this marginal is typically intractable, $\theta$ is then learned by minimizing a discrepancy measure between $q^{\theta}(x,u)$ and an extended target $p^{\theta}(x,u)=\pi(x)p^\theta(u|x)$ where $p^\theta(u|x)$ is an auxiliary conditional distribution \citep{Agakov2004}.

Over recent years, Monte Carlo techniques have also been fruitfully combined to variational techniques. For example, AIS can be thought of a procedure where $q^{\theta}(x,u)$ is the joint distribution of a Markov chain defined by a sequence of MCMC kernels whose final state is $x$ while $p^{\theta}(x,u)$ is the corresponding AIS extended target \citep{neal2001annealed}. The parameters $\theta$ of these kernels can then be optimized by minimizing $\KL(q^{\theta}||p^{\theta})$ using stochastic gradient descent \citep{wunoe2020stochastic,Geffner:2021,Thin:2021,ZhangAIS2021,doucet2022annealed,geffner2022langevin}.

Instead of following an AIS-type approach to define a flexible variational family, we follow here an approach inspired by Denoising Diffusion Probabilistic Models (DDPM), a powerful class of generative models \citep{sohl2015deep,ho2020denoising,song2020score}. In this context, one adds noise progressively to data using a diffusion to transform the complex data distribution into a Gaussian distribution. The time-reversal of this diffusion can then be used to transform a Gaussian sample into a sample from the target. While superficially similar to Langevin dynamics, this process mixes fast even in high dimensions as it inherits the mixing properties of the forward diffusion \citep[Theorem 1]{debortoli2021diffusion}. However, as the time-reversal involves the derivatives of the logarithms of the intractable marginal densities of the forward diffusion, these so-called scores are practically approximated using score matching techniques. If the score estimation error is small, the approximate time-reversal still enjoys remarkable theoretical properties  \citep{debortoli2022convergence,ChenChewiSinhosamplingaseasyaslearningscore2022,lee2022convergencenocondition}.

These results motivate us to introduce Denoising Diffusion Samplers (DDS). Like DDPM, we consider a forward diffusion which progressively transforms the target $\pi$ into a Gaussian distribution. This defines an extended target distribution $p(x,u)=\pi(x)p(u|x)$. DDS are obtained by approximating the time-reversal of this diffusion using a process of distribution $q^{\theta}(x,u)$. What distinguishes DDS from DDPM is that we cannot simulate sample paths from the diffusion we want to time-reverse, as we cannot sample its initial state $x$ from $\pi$. Hence score matching ideas cannot be used to approximate the score terms. 

We focus on minimizing $\KL(q^{\theta}||p)$, equivalently maximizing an Evidence Lower Bound (ELBO), as in variational inference. We leverage a representation of this KL discrepancy based on the introduction of a suitable auxiliary reference process that provides low variance estimate of this objective and its gradient. We can exploit the many similarities between DDS and DDPM to leverage some of the ideas developed in generative modeling for Monte Carlo sampling. This includes using the probability flow ordinary differential equation (ODE) \citep{song2020score} to derive novel normalizing flows and the use of underdamped Langevin diffusions as a forward noising diffusion \citep{Dockhorn2022}. The implementation of these samplers requires designing numerical integrators for the resulting stochastic differential equations (SDE) and ODE. However, simple  integrators such as the standard Euler--Maryuama scheme do not yield a valid ELBO in discrete-time. So as to guarantee one obtains a valid ELBO, DDS relies instead on an integrator for an auxiliary stationary reference process which preserves its invariant distribution as well as an integrator for the approximate time-reversal inducing a distribution absolutely continuous w.r.t. the distribution of the discretized reference process. Finally we compare experimentally DDS to AIS, SMC and other state-of-the-art Monte Carlo methods on a variety of sampling tasks.

\section{Denoising Diffusion Samplers: Continuous Time}\label{sec:DDSCT}
We start here by formulating DDS in continuous-time to gain insight on the structure of the time-reversal we want to approximate. We introduce $\mathcal{C}=C([0,T],\mathbb{R}^d)$ the space of continuous functions from $[0,T]$ to $\mathbb{R}^d$ and $\mathcal{B}(\mathcal{C})$ the Borel sets on $\mathcal{C}$. We will consider in this section path measures which are probability measures on $(\mathcal{C},\mathcal{B}(\mathcal{C}))$, see \cite{leonard2013survey} for a formal definition. Numerical integrators are discussed in the following section. 
\subsection{Forward diffusion and its time-reversal}
Consider the forward noising diffusion given by an Ornstein--Uhlenbeck (OU) process\footnote{This is referred to as a Variance Preserving diffusion by \cite{song2020score}.}
\begin{equation}\label{eq:forwarddiffusionP}
    \mathrm{d}x_t=-\beta_t x_t \mathrm{d}t+\sigma \sqrt{2\beta_t}\mathrm{d}B_t,\qquad x_0 \sim \pi,
\end{equation}
where $(B_t)_{t\in[0,T]}$ is a $d$-dimension Brownian motion and $t \rightarrow \beta_t$ is a non-decreasing positive function. This diffusion induces the path measure $\mathcal{P}$ on the time interval $[0,T]$ and the marginal density of $x_t$ is denoted $p_t$. The transition density of this diffusion is given by $p_{t|0}(x_t|x_0)=\mathcal{N}(x_t;\sqrt{1-\lambda_t}x_0,\sigma^2 \lambda_t I)$ where $\lambda_t=1-\exp(-2\int^t_0\beta_s \mathrm{d}s)$.

From now on, we will always consider a scenario where  $\int_0^T \beta_s \mathrm{d}s \gg 1$ so that $p_T(x_T)\approx \mathcal{N}(x_T;0,\sigma^2 I)$. We can thus think of (\ref{eq:forwarddiffusionP}) as transporting approximately the target density $\pi$ to this Gaussian density.

From \citep{haussmann1986time}, its time-reversal $(y_t)_{t\in[0,T]}=(x_{T-t})_{t\in[0,T]}$, where equality is here in distribution, satisfies
\begin{equation}\label{eq:exacttimereversalCT}
    \mathrm{d}y_t=\beta_{T-t}\{y_t+2\sigma^2 \nabla \log p_{T-t}(y_t)\} \mathrm{d}t+\sigma \sqrt{2\beta_{T-t}}\mathrm{d}W_t,\qquad y_0 \sim p_{T},
\end{equation}
where $(W_t)_{t\in[0,T]}$ is another $d$-dimensional Brownian motion. By definition this time-reversal starts from $y_0 \sim p_T(y_0)\approx \mathcal{N}(y_0;0,\sigma^2 I)$ and is such that $y_T \sim \pi$. This suggests that if we could approximately simulate the diffusion (\ref{eq:exacttimereversalCT}), then we would obtain approximate samples from $\pi$. 

However, putting this idea in practice requires being able to approximate the intractable scores $(\nabla \log p_t(x))_{t \in [0,T]}$. To achieve this, DDPM rely on score matching techniques \citep{Hyvarinen:2005a,vincent2011connection} as it is easy to sample from (\ref{eq:forwarddiffusionP}). This is impossible in our scenario as sampling from (\ref{eq:forwarddiffusionP}) requires sampling $x_0 \sim \pi$ which is impossible by assumption.

\subsection{Reference diffusion and value function}\label{sec:refdiffusionvaluefunction}
In our context, it is useful to introduce a \emph{reference} process defined by the diffusion following (\ref{eq:forwarddiffusionP}) but initialized at $p^{\textup{ref}}_0(x_0)=\mathcal{N}(x_0;0,\sigma^2 I)$ rather than $\pi(x_0)$ thus ensuring that the marginals of the resulting path measure $\mathcal{P}^{\textup{ref}}$ all satisfy $p^{\textup{ref}}_t(x_t)=\mathcal{N}(x_t;0,\sigma^2 I)$. From the chain rule for KL for path measures (see e.g. \cite[Theorem 2.4]{leonard2014some} and \cite[Proposition 24]{debortoli2021diffusion}), one has $\KL(\mathcal{Q}||\mathcal{P}^{\textup{ref}})=\KL(q_0||p^{\textup{ref}}_0)+\mathbb{E}_{x_0 \sim q_0}[\KL(\mathcal{Q}(\cdot|x_0)||\mathcal{P}^{\textup{ref}}(\cdot|x_0))]$. Thus $\mathcal{P}$ can be identified as the path measure minimizing the KL discrepancy w.r.t. $\mathcal{P}^{\textup{ref}}$ over the set of path measures $\mathcal{Q}$ with marginal $q_0(x_0)=\pi(x_0)$ at time $t=0$, i.e. $ \mathcal{P}= \argmin_\mathcal{Q} \{\KL(\mathcal{Q}||\mathcal{P}^{\textup{ref}}): q_0=\pi\}$.

A time-reversal representation of $\mathcal{P}^{\textup{ref}}$ is given by $(y_t)_{t\in[0,T]}=(x_{T-t})_{t\in[0,T]}$ satisfying 
\begin{equation}\label{eq:timereversalrefprocessCT}
    \mathrm{d}y_t =-\beta_{T-t} y_t \mathrm{d}t+\sigma \sqrt{2\beta_{T-t}} \mathrm{d}W_t,\qquad y_0 \sim p^{\textup{ref}}_0.
\end{equation}
As $\beta_{T-t} y_t +2\sigma^2 \nabla \log p^{\textup{ref}}_{T-t}(y_t)=-\beta_{T-t} y_t$, we can rewrite the time-reversal (\ref{eq:exacttimereversalCT}) of $\mathcal{P}$ as 
\begin{equation}\label{eq:diffusionvaluefunction}
    \mathrm{d}y_t= -\beta_{T-t} \{y_t -2\sigma^2 \nabla \log \phi_{T-t}(y_t)\} \mathrm{d}t+\sigma \sqrt{2\beta_{T-t}} \mathrm{d}W_t,\qquad y_0 \sim p_T,
\end{equation}
where $\phi_t(x)=p_t(x)/p^{\textup{ref}}_t(x)$ is a so-called value function which can be shown to satisfy a Kolmogorov equation such that $\phi_t(x_t)=\mathbb{E}_{\mathcal{P}^{\textup{ref}}}[\phi_0(x_0)|x_t]$, the expectation being w.r.t. $\mathcal{P}^{\textup{ref}}$.

\subsection{Learning the time-reversal}
To approximate the time-reversal (\ref{eq:exacttimereversalCT}) of $\mathcal{P}$, consider a path measure $\mathcal{Q}^\theta$ whose time-reversal is induced by
\begin{equation}\label{eq:Qthetascore}
\mathrm{d}y_t=\beta_{T-t}\{y_t+2\sigma^2 s_{\theta}(T-t,y_t) \} \mathrm{d}t+\sigma \sqrt{2\beta_{T-t}}\mathrm{d}W_t,\qquad y_0  \sim \mathcal{N}(0,\sigma^2 I),
\end{equation}
so that $y_t \sim q^{\theta}_{T-t}$. To obtain $s_{\theta}(t,x) \approx \nabla \log p_t(x)$, we parameterize $s_{\theta}(t,x)$ by a neural network whose parameters are obtained by minimizing
\begin{align}\label{eq:scorematchinglike}
 \KL(\mathcal{Q}^\theta||\mathcal{P})&=\KL(\mathcal{N}(0,\sigma^2 I)||p_T)+\mathbb{E}_{y_0\sim \mathcal{N}(0,\sigma^2 I)}[\KL(\mathcal{Q}^\theta(\cdot |y_0)||\mathcal{P}(\cdot|y_0))]\\
    &=\KL(\mathcal{N}(0,\sigma^2 I)||p_T)+\sigma^2 \mathbb{E}_{\mathcal{Q}^\theta}\Bigr[\scaleobj{.8}{\int_0^T} \beta_{T-t}||s_\theta(T-t,y_t)-\nabla \log p_{T-t}(y_t)||^2 \mathrm{d}t \Bigr],\nonumber
\end{align}
where we have used the chain rule for KL then Girsanov's theorem (see e.g. \citep{klebaner2012introduction}). This expression of the KL is reminiscent of the expression obtained in \cite[Theorem 1]{song2021maximum} in the context of DDPM. However, the expectation appearing in (\ref{eq:scorematchinglike}) is here w.r.t. $\mathcal{Q}^\theta$ and not w.r.t.  $\mathcal{P}$ and we cannot get rid of the intractable scores $(\nabla \log p_{t}(x))_{t \in [0,T]}$ using score matching ideas.

Instead, taking inspiration from (\ref{eq:diffusionvaluefunction}), we reparameterize the time-reversal of  $\mathcal{Q}^\theta$ using
\begin{equation}\label{eq:approximatetimereversalCT}
    \mathrm{d}y_t=-\beta_{T-t}\{y_t -2\sigma^2 f_{\theta}(T-t,y_t)\} \mathrm{d}t+\sigma \sqrt{2\beta_{T-t}} \mathrm{d}W_t,\qquad y_0 \sim \mathcal{N}(0,\sigma^2 I),
\end{equation}
i.e. $f_\theta(t,x)=s_\theta(t,x)-\nabla \log p^{\textup{ref}}_t(x)=s_{\theta}(t,x)+x/\sigma^2$. This reparameterization allows us to express  $\KL(\mathcal{Q}^\theta||\mathcal{P})$ in a compact form.

\begin{proposition}\label{prop:RNcontinuous}
The Radon--Nikodym derivative $\frac{\mathrm{d}\mathcal{Q}^{\theta}}{\mathrm{d}\mathcal{P}^{\textup{ref}}}(y_{[0,T]})$ satisfies under $\mathcal{Q}^{\theta}$
\begin{equation}\label{eq:RadonNykodyn}
   \log \left(\frac{\mathrm{d}\mathcal{Q}^{\theta}}{\mathrm{d}\mathcal{P}^{\textup{ref}}}\right)=\sigma^2 \scaleobj{.8}{\int_0^T}  \beta_{T-t} ||f_\theta(T-t,y_t)||^2 \mathrm{d}t+\sigma \scaleobj{.8}{\int_0^T} \sqrt{2 \beta_{T-t}} f_\theta(T-t,y_t)^\top \mathrm{d}W_t.
\end{equation}
From the identity $\KL(\mathcal{Q}^\theta||\mathcal{P})=\KL(\mathcal{Q}^\theta||\mathcal{P}^{\textup{ref}})+\mathbb{E}_{y_T \sim q^{\theta}_0}[\ln \left(\frac{p^{\textup{ref}}_0(y_T)}{p_0(y_T)}\right)]$, it follows that
\begin{align}
\KL(\mathcal{Q}^\theta||\mathcal{P})&
=\mathbb{E}_{\mathcal{Q}^\theta} \Bigr[ \sigma^2 \scaleobj{.8}{\int_0^T} \beta_{T-t} ||f_\theta(T-t,y_t)||^2 \mathrm{d}t
     +\scaleobj{.8}{\ln \left(\frac{\gN(y_T; 0, \sigma^2 I)}{\pi(y_T)}\right)} \Bigr]\label{eq:KLpathintegral}.
\end{align}
\end{proposition}

For $\theta$ minimizing (\ref{eq:KLpathintegral}), approximate samples from $\pi$ can be obtained by simulating (\ref{eq:approximatetimereversalCT}) and returning $y_T\! \sim \!q^{\theta}_0$. We can obtain an unbiased estimate of $Z$ via the following importance sampling identity
\begin{equation}\label{eq:hatZCT}
 \hat{Z}=\frac{\gamma(y_{T})}{\mathcal{N}(y_T;0,\sigma^2 I)} \frac{\mathrm{d}\mathcal{P}^{\textup{ref}}}{\mathrm{d}\mathcal{Q}^{\theta}}(y_{[0,T]}),
\end{equation}
where the second term can be computed directly from (\ref{eq:RadonNykodyn}) and $y_{[0,T]}$ is obtained by solving (\ref{eq:approximatetimereversalCT}). 

\subsection{Continuous-time normalizing flow}\label{sec:CTNF}
To approximate the log likelihood of DDPM, it was proposed by \cite{song2020score} to rely on the probability flow ODE. This is an ODE admitting the same marginal distributions $(p_t)_{t\in[0,T]}$ as the noising diffusion given by $\mathrm{d}x_t=-\beta_t \big\{ x_t + \sigma^2 \nabla \log p_t(x_t)\big\}\mathrm{d}t$. We use here this ODE to design a continuous-time normalizing flow to sample from $\pi$. For $\theta$ such that $f_{\theta}(t,x) \approx \nabla \log \phi_t(x)$ (obtained by minimization of the KL in (\ref{eq:KLpathintegral})), we have $x_t+ \sigma^2 \nabla \log p_t(x) \approx \sigma^2 f_{\theta}(t,x)$. So it is possible to sample approximately from $\pi$ by integrating the following ODE over $[0,T]$
\begin{equation}\label{eq:ODEproposal}
   \mathrm{d}y_t=\sigma^2 \beta_{T-t} f_{\theta}(T-t,y_t) \mathrm{d}t,\qquad y_0\sim  \mathcal{N}(0,\sigma^2 I).
\end{equation}
If denote by $\bar{q}^{\theta}_{T-t}$ the distribution of $y_t$ then $y_T \sim \bar{q}^{\theta}_{0}$ is an approximate sample from $\pi$. We can use this normalizing flow to obtain an unbiased estimate of $Z$ using importance sampling $\hat{Z}=\gamma(y_{T})/\bar{q}^{\theta}_{0}(y_T)$ for $y_T\sim \bar{q}^{\theta}_{0}$.
Indeed, contrary to the proposal $q^{\theta}_0$ induced by (\ref{eq:approximatetimereversalCT}), we can compute pointwise $\bar{q}^{\theta}_{0}$ using the instantaneous change of variables formula \citep{chen2018neural} such that $\log \bar{q}^{\theta}_0(y_T) = \log \mathcal{N}(y_0;0,\sigma^2 I)-\sigma^2 \scaleobj{.8}{\int^T_0} \beta_{T-t} \nabla \cdot f_{\theta}(T-t,y_t) \mathrm{d}t$ and where $(y_t)_{t\in [0,T]}$ arises from the ODE (\ref{eq:ODEproposal}). 
The expensive divergence term can be estimated using the Hutchinson estimator \citep{grathwohl2018ffjord,song2020score}.

\subsection{Extension to Underdamped Langevin Dynamics}\label{subsec:underdampedLangevin}
For DDPM, it has been proposed to extend the original state $x \in \mathbb{R}^d$ by a momentum variable $p \in \mathbb{R}^d$. One then diffuses the data distribution augmented by a Gaussian distribution on the momentum using an underdamped Langevin dynamics targeting $\mathcal{N}(x;0,\sigma^2I)\mathcal{N}(m;0,M)$ where $M$ is a positive definite mass matrix \citep{Dockhorn2022}. It was demonstrated empirically that the resulting scores are smoother, hence easier to estimate and this leads to improved performance. We adapt here this approach to Monte Carlo sampling; see Section \ref{sec:underdampeddetailsSM} for more details.

We diffuse $\bar{\pi}(x,m)=\pi(x)\mathcal{N}(m;0,M)$ using an underdamped Langevin dynamics, i.e.
\begin{equation}\label{eq:forwardunderdamped}
\mathrm{d}x_t=M^{-1}m_t\mathrm{d}t,\quad \mathrm{d}m_t=-\frac{x_t}{\sigma^2}\mathrm{d}t -\beta_t m_t\mathrm{d}t +\sqrt{2\beta_t}M^{1/2}\mathrm{d}B_t,\quad (x_0,m_0)\sim \bar{\pi}.
\end{equation}
The resulting path measure on $[0,T]$ is denoted $\mathcal{P}$ and the marginal of $(x_t,m_t)$ is denoted $\eta_t$.

Here, the reference process $\mathcal{P}^{\textup{ref}}$ is defined by the diffusion (\ref{eq:forwardunderdamped}) initialized using $x_0 \sim \mathcal{N}(0,\sigma^2 I), m_0\sim \mathcal{N}(0,M)$. This ensures that $\eta^{\textup{ref}}_t(x_t,m_t)=\mathcal{N}(x_t;0,\sigma^2I)\mathcal{N}(m_t;0,M)$ for all $t$ and the time-reversal process $(y_t,n_t)_{t\in[0,T]}=(x_{T-t},m_{T-t})_{t\in[0,T]}$ (where equality is in distribution) of this stationary diffusion satisfies
\begin{align}\label{eq:reverseunderdampedref}
\mathrm{d}y_t&=-M^{-1}n_t\mathrm{d}t,\quad
\mathrm{d}n_t
=\frac{y_t}{\sigma^2}\mathrm{d}t -\beta_{T-t} n_t\mathrm{d}t+\sqrt{2\beta_{T-t}}M^{1/2}\mathrm{d}W_t. 
\end{align}

Using manipulations identical to Section \ref{sec:refdiffusionvaluefunction}, the time-reversal of $\mathcal{P}$ can be also be written for $\phi_t(x,m):=\eta_t(x,m)/\eta_t^{\textup{ref}}(x,m)$ as $\mathrm{d}y_t=-M^{-1}n_t\mathrm{d}t$ and
\begin{align}\label{eq:reverseunderdampedphi}
\mathrm{d}n_t=
\frac{y_t}{\sigma^2}\mathrm{d}t - \beta_{T-t}n_t\mathrm{d}t +2\beta_{T-t}\nabla_{n_t} \log \phi_{T-t}(y_t,n_t)\mathrm{d}t+\sqrt{2\beta_{T-t}}M^{1/2}\mathrm{d}W_t. 
\end{align}
To approximate $\mathcal{P}$, we will consider a parameterized path measure $\mathcal{Q}^\theta$ whose time reversal is defined for $(y_0,n_0)\sim \mathcal{N}(y_0;0,\sigma^2 I)\mathcal{N}(n_0;0,M)$ by $\mathrm{d}y_t=-M^{-1}n_t\mathrm{d}t$ and 
\begin{align}\label{eq:reverseunderdampedQ}
\mathrm{d}n_t=\frac{y_t}{\sigma^2}\mathrm{d}t - \beta_{T-t} n_t\mathrm{d}t +2\beta_{T-t}M f_{\theta}(T-t,y_t,n_t)\mathrm{d}t+\sqrt{2\beta_{T-t}}M^{1/2}\mathrm{d}W_t.
\end{align}
We can then express the KL of interest in a compact way similar to Proposition \ref{prop:RNcontinuous}, see Appendix \ref{appdx:udmpcont}. A normalising flow formulation is presented in Appendix \ref{app:NFunderdamped}.

\begin{algorithm*}[t!]
\caption{DDS Training}\label{alg:dds_alg}
\begin{algorithmic}[1]
\State {\textbf{Input: }  $\sigma >0$, $\gamma: \sR^d \rightarrow \sR^+$, $(\beta_k)_{k=1}^K \in (\sR^+)^K$, $\theta \in  \sR^p$, $\lambda>0$}
\For{$i = 1, \dots, \mathrm{training\;iterations}$}
    \For{$n = 1, \dots, N$} \textcolor{gray}{ $\quad$ \# Iterate over training batch size.}
        \State{Sample~$y_{0,n} \sim \gN(0,\sigma^2 I)$ and set $r_0=0$}
        \For{$k = 0, \dots, K-1$} \textcolor{gray}{ $\quad$ \# Iterate over integration steps.}
            \State{$\lambda_{K-k}:=1-\sqrt{1-\alpha_{K-k}}$}
            \State{$ y_{k+1,n}\!=\!\sqrt{1\!\!-\!\alpha_{K\!-\!k}} y_{k,n}+2\sigma^2\lambda_{K\!-k} f_\theta(K\!\!-k,y_{k,n}\!) +\sigma \!\sqrt{\alpha_{K-k}} \varepsilon_{k,n},\; \varepsilon_{k,n}\! \overset{\textup{i.i.d.}}{\sim}\! \mathcal{N}(0,I)$}
            \State{$r_{k+1,n} = r_{k,n} + \frac{2\sigma^2 \lambda_{K-k}^2}{\alpha_{K-k}} || f_\theta(K-k,y_{k,n})||^2$}
        \EndFor
        \State{$\theta \leftarrow \theta - \lambda \nabla_\theta \frac{1}{N} \sum_{n=1}^N \left(r_{K,n} + \ln \left(\frac{\gN(y_{K,n}; 0, \sigma^2 I)}{\pi(y_{K,n})}\right)\right)$}
    \EndFor    
\EndFor
\State{\textbf{Return: }  $\theta$}
\end{algorithmic}
\end{algorithm*}

\section{Denoising Diffusion Samplers: Discrete-Time}\label{sec:DDSDT}
We introduce here discrete-time integrators for the SDEs and ODEs introduced in Section \ref{sec:DDSCT}. Contrary to DDPM, we not only need an integrator for $\mathcal{Q}_{\theta}$ but also for $\mathcal{P}^{\textup{ref}}$ to be able to compute an approximation of the Radon--Nikodym derivative $\mathrm{d}\mathcal{Q}_{\theta}/\mathrm{d}\mathcal{P}^{\textup{ref}}$. Additionally this integrator needs to be carefully designed as explained below to preserve an ELBO. For sake of simplicity, we consider a constant discretization step $\delta\!>\!0$ such that $K\!=T/\delta$ is an integer. In the indices, we write $k$ for $t_k=k\delta$ to simplify notation; e.g. $x_k$ for $x_{k\delta}$.

\subsection{Integrator for $\gP^{\mathrm{ref}}$}

Proposition \ref{prop:RNcontinuous} shows that learning the time reversal in continuous-time can be achieved by minimising the objective $\KL(\mathcal{Q}^\theta||\mathcal{P})$ given in (\ref{eq:KLpathintegral}). This expression is obtained by using the identity $\KL(\mathcal{Q}^\theta||\mathcal{P})=\KL(\mathcal{Q}^\theta||\mathcal{P}^{\textup{ref}})+\mathbb{E}_{y_T \sim q^{\theta}_0}[\ln \left(p^{\textup{ref}}_0(y_T)/p_0(y_T)\right)]$. This is also equivalent to maximizing the corresponding ELBO, $\mathbb{E}_{\mathcal{Q}^\theta}[\log \hat{Z}]$, for the unbiased estimate $\hat{Z}$ of $Z$ given in (\ref{eq:hatZCT})

An Euler--Maruyama (EM) discretisation of the corresponding SDEs is obviously applicable. However, it is problematic as established below.
\begin{proposition}\label{prop:notelbo}
Consider integrators of $\mathcal{P}^{\textup{ref}}$ and $\mathcal{Q}^{\theta}$ leading to approximations $p^{\textup{ref}}(x_{0:K})$ and $q^{\theta}(x_{0:K})$. Then $\KL(q^\theta||{p}^{\textup{ref}})+\mathbb{E}_{y_K \sim q^{\theta}_0}[\ln \left(p^{\textup{ref}}_0(y_K)/\pi(y_K)\right)]$ is guaranteed to be non-negative and $\mathbb{E}_{\mathcal{Q}^\theta}[\log \hat{Z}]$ is an ELBO for $\hat{Z}=(\gamma(y_{K})p^{\textup{ref}}(y_{0:K}))/(\mathcal{N}(y_K;0,\sigma^2) q^\theta(y_{0:K}))$ if one has $p^{\mathrm{ref}}_K(y_K) = p_0^{\mathrm{ref}}(y_K)$. The EM discretisation of $\mathcal{P}^{\textup{ref}}$ does \emph{not} satisfy this condition.
\end{proposition}

A direct consequence of this result is that the estimator for $\ln Z$ that uses the EM discretisation can be such that $\mathbb{E}_{\mathcal{Q}^\theta}[\log \hat{Z}] \geq \ln Z$ as observed empirically in Table \ref{fig:euvsexp}. To resolve this issue, we  derive in the next section an integrator for $\gP^{\mathrm{ref}}$ which ensures that $p^{\mathrm{ref}}_k(y) = p_0^{\mathrm{ref}}(y)$ for all $k$. 

\subsection{Ornstein--Uhlenbeck}
We can integrate exactly (\ref{eq:timereversalrefprocessCT}). This is given by $y_0\sim \mathcal{N}(0,\sigma^2 I)$ and
\begin{equation}\label{eq:OUrefintegrated}
    y_{k+1}=\sqrt{1-\alpha_{K-k}} y_k+\sigma \sqrt{\alpha_{K-k}} \varepsilon_k,\quad \varepsilon_k \overset{\textup{i.i.d.}}{\sim} \mathcal{N}(0,I),
\end{equation}
for $\alpha_k=1-\exp\left(-2 \int_{(k-1)\delta}^{k\delta}\beta_s \mathrm{d}s\right)$~\footnote{Henceforth, it is assumed that $\alpha_k$ can be computed exactly.}. This defines the discrete-time reference process.
We propose the following exponential type integrator \citep{debortoli2022convergence} for (\ref{eq:approximatetimereversalCT}) initialized using $y_0 \sim \mathcal{N}(0,\sigma^2 I)$
\begin{equation}\label{eq:integratorapproxrevverse}
  y_{k+1}=\sqrt{1-\alpha_{K-k}} y_k+2\sigma^2 (1-\sqrt{1-\alpha_{K-k}}) f_\theta(K-k,y_k)+\sigma \sqrt{\alpha_{K-k}}
    \varepsilon_k,
\end{equation}
where $\varepsilon_k \overset{\textup{i.i.d.}}{\sim} \mathcal{N}(0,I)$. Equations (\ref{eq:OUrefintegrated}) and (\ref{eq:integratorapproxrevverse}) define the time reversals of the reference process $p^{\textup{ref}}(x_{0:K})$ and proposal $q^{\theta}(x_{0:K})$. For its time reversal, we write $q^{\theta}(y_{0:K})=\mathcal{N}(y_0;0,\sigma^2 I)\prod_{k=1}^K q^{\theta}_{k-1|k}(y_{K-k+1}|y_{K-k})$ abusing slightly notation, and similarly for $p^{\textup{ref}}(y_{0:K})$.
We will be relying on the discrete-time counterpart of (\ref{eq:RadonNykodyn}).
\begin{proposition}\label{prop:KL}
The log density ratio between $q^\theta(y_{0:K})$ and $p^{\textup{ref}}(y_{0:K})$ satisfies for $y_{0:K} \sim q^\theta(y_{0:K})$
\begin{equation}\label{eq:logDTRN}
 \log  \left(\frac{q^\theta(y_{0:K})}{p^{\textup{ref}}(y_{0:K})}\right)
    =2 \sigma^2 \sum_{k=1}^K \frac{\lambda_k^2}{\alpha_k}||f_\theta(k,y_{K-k})||^2 +2 \sigma \sum_{k=1}^K \frac{\lambda_k}{\sqrt{\alpha_k}}f_\theta(k,y_{K-k})^\top \varepsilon_k 
\end{equation}
where $\lambda_k:=1-\sqrt{1-\alpha_k}$ and $\varepsilon_k$ defined through (\ref{eq:integratorapproxrevverse}) is such that $\varepsilon_k \overset{\textup{i.i.d.}}{\sim} \mathcal{N}(0,I)$. In particular, one obtains from $\KL(q^{\theta}||p)=\KL(q^{\theta}||p^{\textup{ref}})+\mathbb{E}_{q^{\theta}_0}\left[\log \left( \frac{\pi(y_K)}{\mathcal{N}(y_K;0,\sigma^2 I)}\right)\right]$ that
\begin{align}
    \KL(q^{\theta}||p)&
    =\E_{q^\theta}\left[2 \sigma^2 \sum_{k=1}^K \frac{\lambda_k^2}{\alpha_k}||f_\theta(k,y_{K-k})||^2  + \scaleobj{.8}{\ln \left(\frac{\gN(y_K; 0, \sigma^2 I)}{\pi(y_K)}\right)}\right].\label{eq:KLneat}
\end{align}
\end{proposition}
We compute an unbiased gradient of this objective using the reparameterization trick and the JAX software package \citep{jax2018github}. The training procedure is summarized in Algorithm \ref{alg:dds_alg} in Appendix \ref{apdx:discrete}. Unfortunately, contrary to DDPM, $q^{\theta}_{k|K}$ is not available in closed form for $k<K-1$ so we can neither mini-batches over the time index $k$ without having to simulate the process until the minimum sampled time nor reparameterize $x_k$ as in \cite{ho2020denoising}. 
Once we obtain the parameter $\theta$ minimizing (\ref{eq:KLneat}), DDS samples from $q^{\theta}$ using (\ref{eq:integratorapproxrevverse}). The final sample $y_K$ has a distribution $q^{\theta}_0$ approximating $\pi$ by design. By using importance sampling, we obtain an unbiased estimate of the normalizing constant $\hat{Z}=(\gamma(y_{K})p^{\textup{ref}}(y_{0:K}))/(\mathcal{N}(y_K;0,\sigma^2) q^\theta(y_{0:K}))$ for $y_{0:K}\sim q^\theta(y_{0:K})$. Finally Appendices \ref{subsec:integratorunderdamped} and \ref{apdx:KLunderdampedDT} extend this approach to provide a similar approach to discretize the underdamped dynamics proposed in Section \ref{subsec:underdampedLangevin}. In this context, the proposed integrators rely on a leapfrog scheme \citep{leimkuhler2016molecular}.

\subsection{Theoretical Guarantees}
Motivated by DDPM, bounds on the total variation between the target distribution and the distribution of the samples generated by a time-discretization of an approximate time reversal of a forward noising diffusion have been first obtained in \citep{debortoli2021diffusion} then refined in \citep{debortoli2022convergence,ChenChewiSinhosamplingaseasyaslearningscore2022,lee2022convergencenocondition} and extended to the Wasserstein metric. These results are directly applicable to DDS because their proofs rely on assumptions on the score approximation error but not on the way these approximations are learned, i.e. via score matching for DDPM and reverse KL for DDS. For example, the main result in \cite{ChenChewiSinhosamplingaseasyaslearningscore2022} shows that if the true scores are $L$-Lipschitz, the $L_2(p_t)$ error on the scores is bounded and other mild integrability assumptions then DDS outputs samples $\epsilon$-close in total variation to $\pi$ in $O(L^2d/\epsilon^2)$ time steps. As pointed out by the authors, this matches state-of-the-art complexity bounds for Langevin Monte Carlo algorithm for sampling targets satisfying a log-Sobolev inequality \emph{without} having to make any log-concavity assumption on $\pi$.
However, the assumption on the approximation error for score estimates is less realistic for DDS than DDPM as we do not observe realizations of the forward diffusion.

\begin{figure}
    \centering
    \includegraphics[width=0.8\linewidth, height=3.8cm]{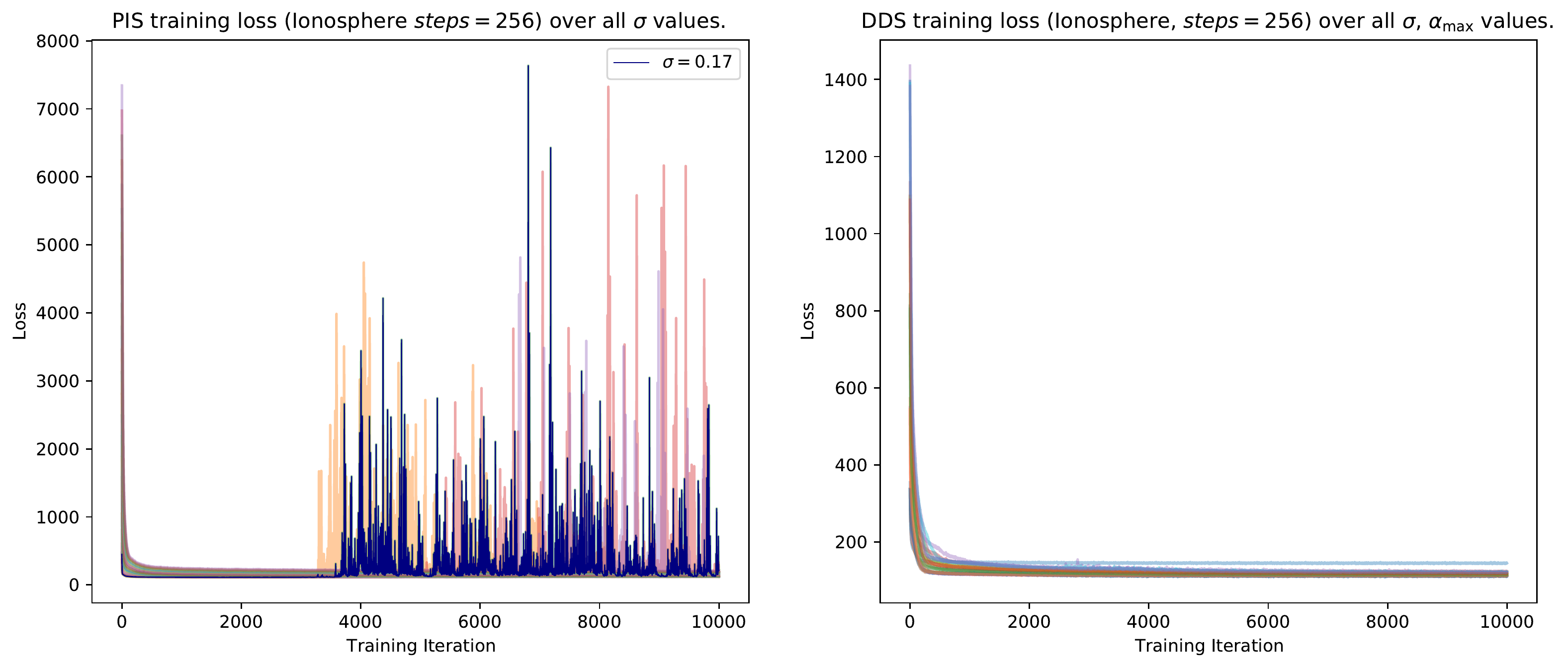}
    \caption{Training loss per hyperparameter: PIS (left) vs DDS (right).}
    \label{fig:stability2}
\end{figure}

\subsection{Intepretation, Related Work and Extensions}

\textbf{DDS as KL and path integral control.} The reverse KL we minimize can be expressed as
\begin{equation}
  \KL(q^{\theta}||p)=\mathbb{E}_{q^{\theta}}\left[\scaleobj{.9}{\biggl\{ \ln \left(\frac{\gN(x_0; 0, \sigma^2 I)}{\pi(x_0)}\right)+ \sum_{k=1}^K 
    \log \left( \frac{q^\theta_{k-1|k}(x_{k-1}|x_k)}{p^{\textup{ref}}_{k-1|k}(x_{k-1}|x_k)} \right)  \biggl\}}\right].
\end{equation}
This objective is a specific example of KL control problem \citep{kappen2012optimal}. In continuous-time, (\ref{eq:KLpathintegral}) corresponds to a path integral control problem; see e.g. \citep{kappen2016adaptive}. 

\begin{figure*}[t!]
    \centering
    \begin{minipage}{0.33\linewidth }
    \centering
    \includegraphics[width=\linewidth ]{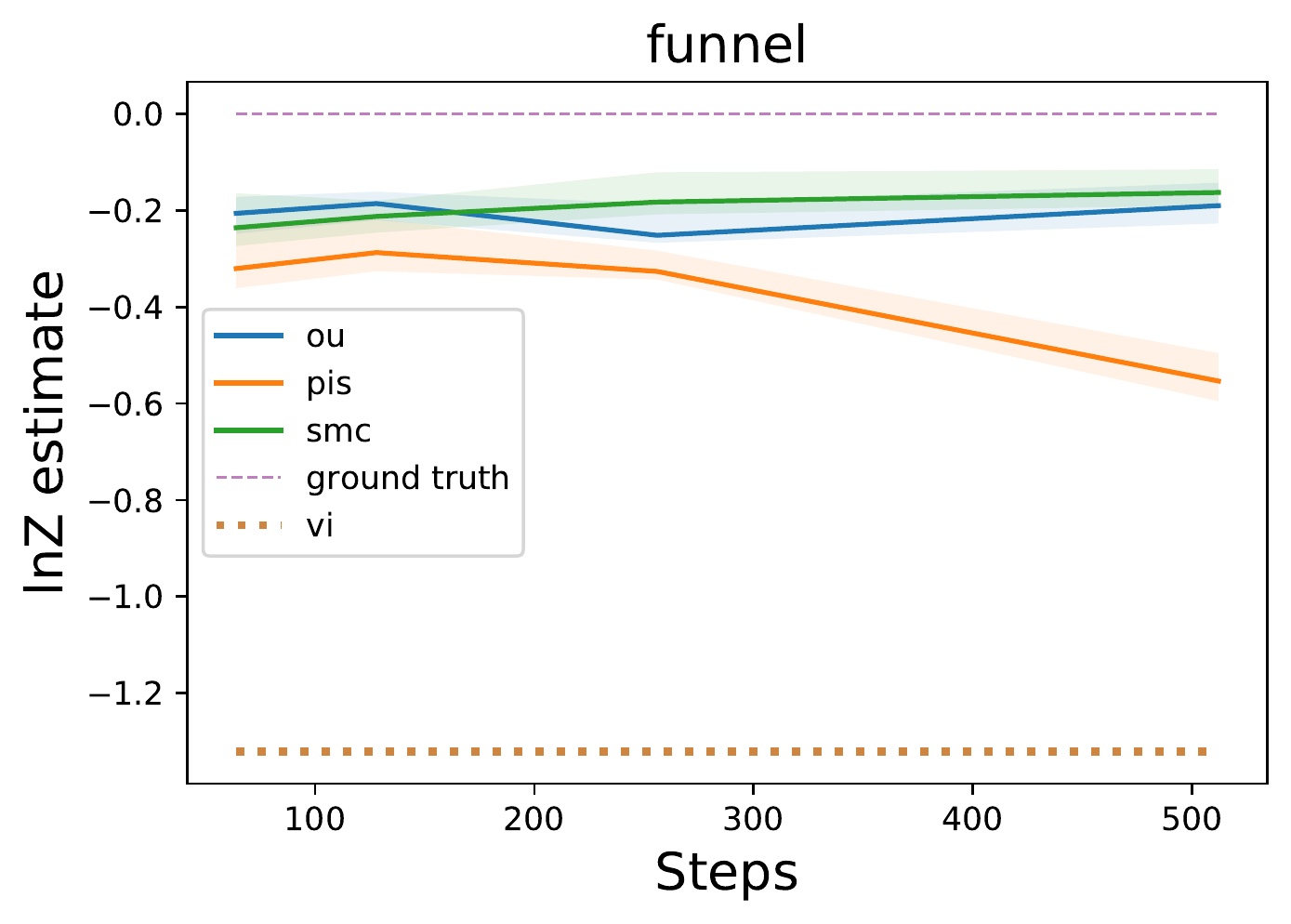}
    \end{minipage}
    \hspace*{\fill}\begin{minipage}{0.32\linewidth }
    \centering
    \includegraphics[width=\linewidth ]{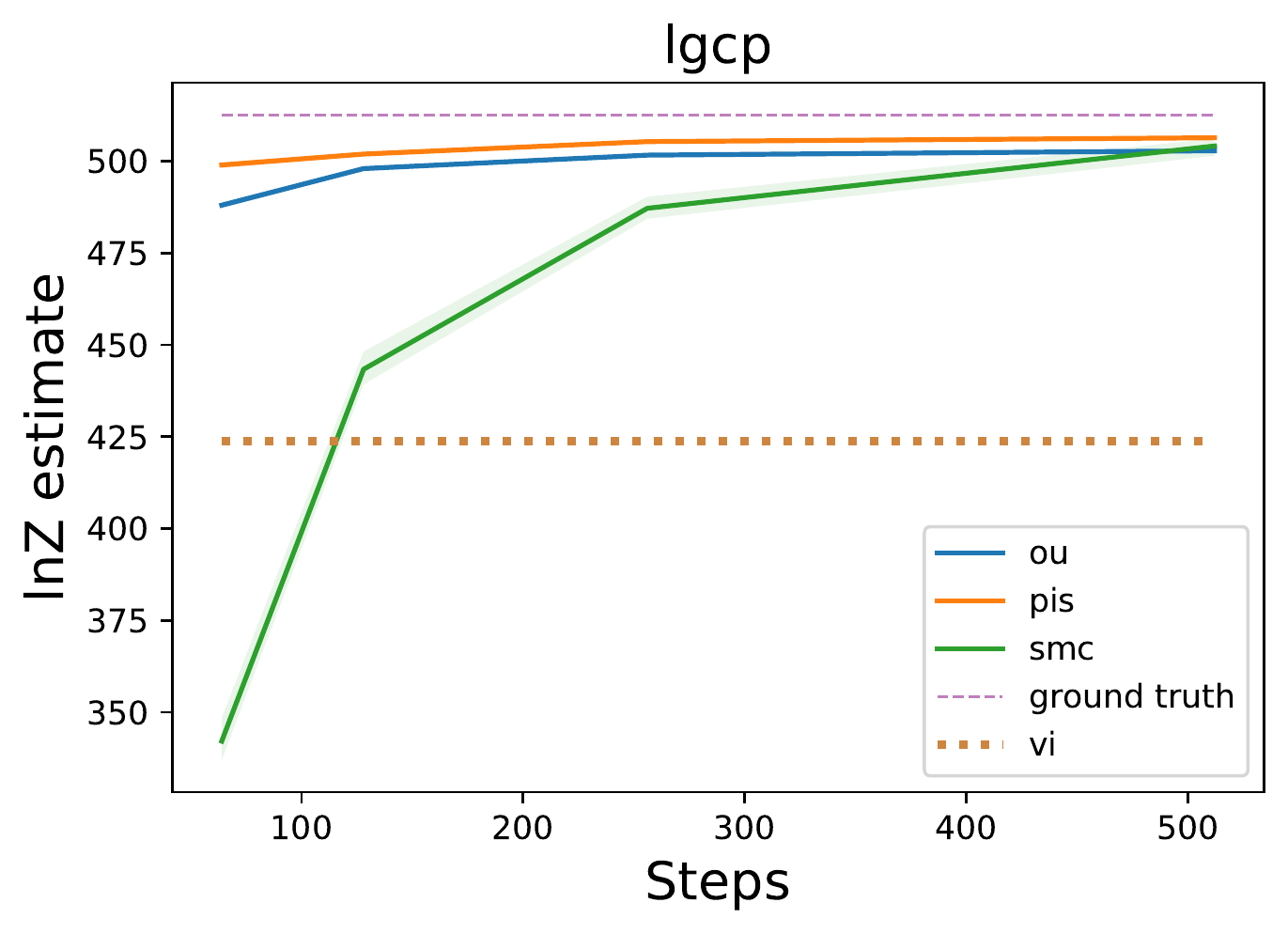}
    \end{minipage}
     \hspace*{\fill}
    \hspace*{\fill}\begin{minipage}{0.32\linewidth }
    \centering
    \includegraphics[width=\linewidth ]{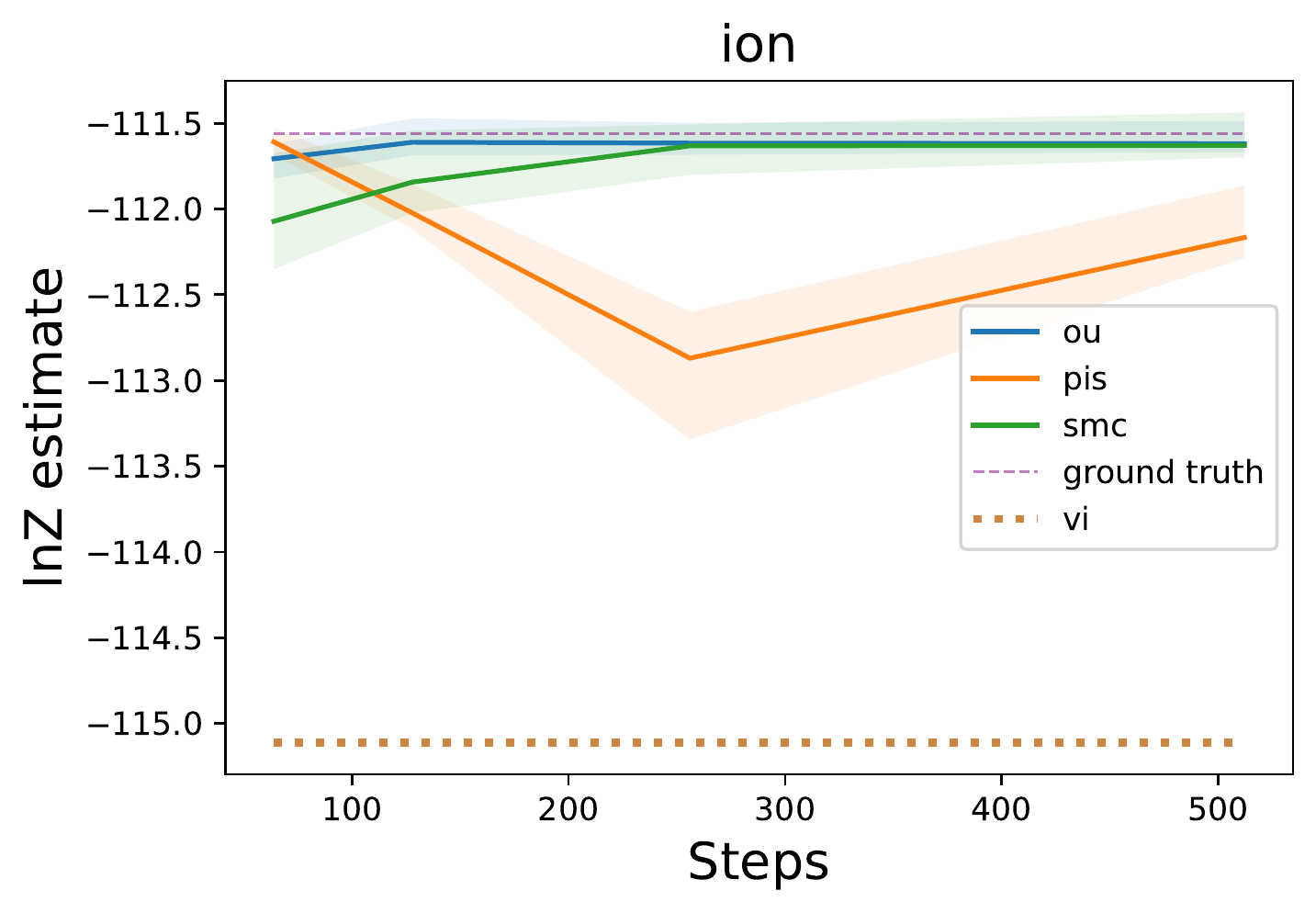}
    \end{minipage}
    
    \caption{$\log Z$ estimate (median plus upper/lower quartiles) as a function of number of steps $K$ - a) Funnel , b) LGCP, c) Logistic Ionosphere dataset. Yellow dotted line is MF-VI and dashed magenta is the gold standard. \label{fig:base_targets}} 
\end{figure*}

\cite{heng2020controlled} use KL control ideas so as to sample from a target $\pi$ and estimate $Z$. However, their algorithm relies on $p^{\textup{ref}}(x_{0:K})$ being defined by a discretized non-homogeneous Langevin dynamics such that $p_K(x_K)$ is typically not approximating a known distribution. Additionally it approximates the value functions $(\phi_k)_{k=0}^K$ using regression using simple linear/quadratic functions. Finally it relies on a good initial estimate of $p(x_{0:K})$ obtained through SMC. This limits the applicability of this methodology to a restricted class of models.

\textbf{Connections to Schr\"odinger Bridges.}
The Schr\"odinger Bridge (SB) problem \citep{leonard2013survey,debortoli2021diffusion} takes the following form in discrete time. Given a reference density $p^{\textup{ref}}(x_{0:K})$, we want to find the density $p^{\textup{sb}}(x_{0:K})$ s.t. $p^{\textup{sb}}=\argmin_q \{\KL(q||p^{\textup{ref}}): q_0=\mu_0,~~q_K=\mu_K\}$ where $\mu_0,\mu_K$ are prescribed distributions. This problem can be solved using iterative proportional fitting (IPF) which is defined by the following recursion with initialization $p^1=p^{\textup{ref}}$
\begin{equation}
p^{2n}=\argmin_q \{\KL(q||p^{2n-1}):q_0=\mu_0 \},~~
    p^{2n+1}=\argmin_q \{\KL(q||p^{2n}):q_K=\mu_K \}.
\end{equation}
Consider the SB problem where $\mu_0(x_0)=\pi(x_0)$, $\mu_K(x_K)=\mathcal{N}(x_K;0,\sigma^2 I)$ and the time-reversal of $p^{\textup{ref}}(x_{0:K})$ is defined through (\ref{eq:OUrefintegrated}). In this case, $p^2=p$ corresponds to the discrete-time version of the noising process and $p^3$ to the time-reversal of $p$ but initialized at $\mu_K$ instead of $p_K$. This is the process DDS is approximating. As $p_K \approx \mu_K$ for $K$ large enough, we have approximately $p^{\textup{sb}} \approx p_3 \approx p_2$. We can thus think of DDS as approximating the solution to this SB problem.

Consider now another SB problem where $\mu_0(x_0)=\pi(x_0)$, $\mu_K(x_K)=\delta_0(x_K)$ and $p^{\textup{ref}}(x_{0:K})=\delta_0(x_K)\prod_{k=0}^{K-1} \mathcal{N}(x_k;x_{k+1},\delta \sigma^2 I)$, i.e. $p^{\textup{ref}}$ is a pinned Brownian motion running backwards in time. This SB problem was discussed in discrete-time in  \citep{beghi1996relative} and in continuous-time in \citep{follmer1984entropy,daipra1991stochastic,tzen2019theoretical}. In this case, $p^2(x_{0:K})=\pi(x_0)p^{\textup{ref}}(x_{1:K}|x_0)$ is a modified ``noising" process that transports $\pi$ to the degenerate measure $\delta_0$ and it is easily shown that $p^{\textup{sb}}=p^2$. 
Sampling from the time-reversal of this measure would generate samples from $\pi$ starting from $\delta_0$. Algorithmically, \cite{zhangmarzouk2021sampling} proposed approximating this time-reversal by some projection on some eigenfunctions. In parallel, \cite{barrlamacraft2020quantum}, \cite{vargas2021bayesian} and \cite{zhangyongxinchen2021path} approximated this SB by using a neural network parameterization of the gradient of the logarithm of the corresponding value function trained by minimizing a reverse KL. We will adopt the Path Integral Sampler (PIS) terminology proposed by \cite{zhangyongxinchen2021path} for this approach. DDS can thus be seen as an alternative to PIS which relies on a reference dynamics corresponding to an overdamped or underdamped OU process instead of a pinned Brownian motion. Theoretically the drift of the resulting time-reversal for DDS is not as a steep as for PIS (see Appendix \ref{app:comparisonDDSPIS}) and empirically this significantly improves numerical stability of the training procedure; see Figure \ref{fig:stability2}. The use of a pinned Brownian motion is also not amenable to the construction of normalizing flows.

\begin{figure*}[t!]
    \centering
    \begin{minipage}{0.33\linewidth }
    \centering
    \includegraphics[width=\linewidth ]{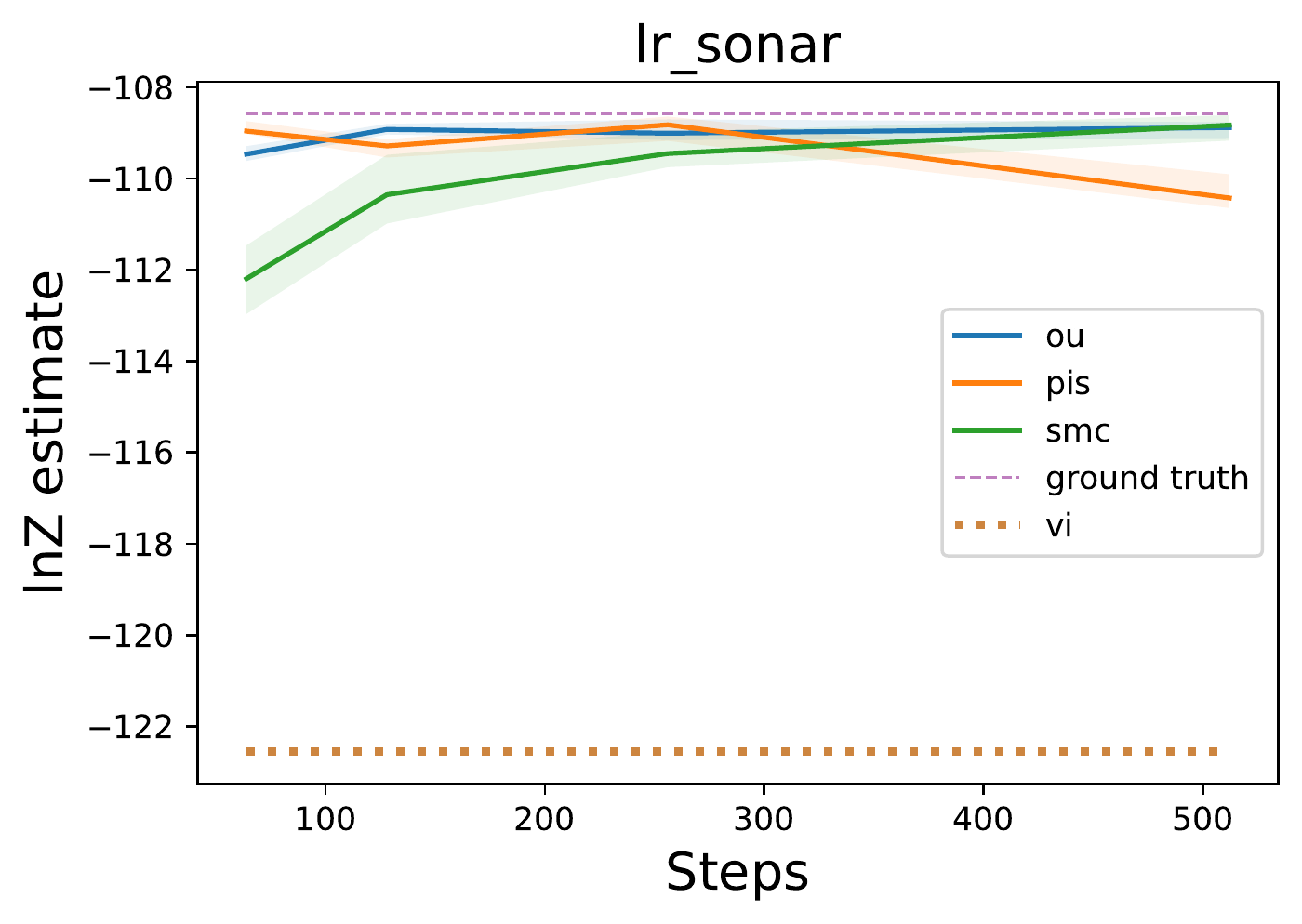}
    \end{minipage}
    \hspace*{\fill}\begin{minipage}{0.32\linewidth }
    \centering
    \includegraphics[width=\linewidth ]{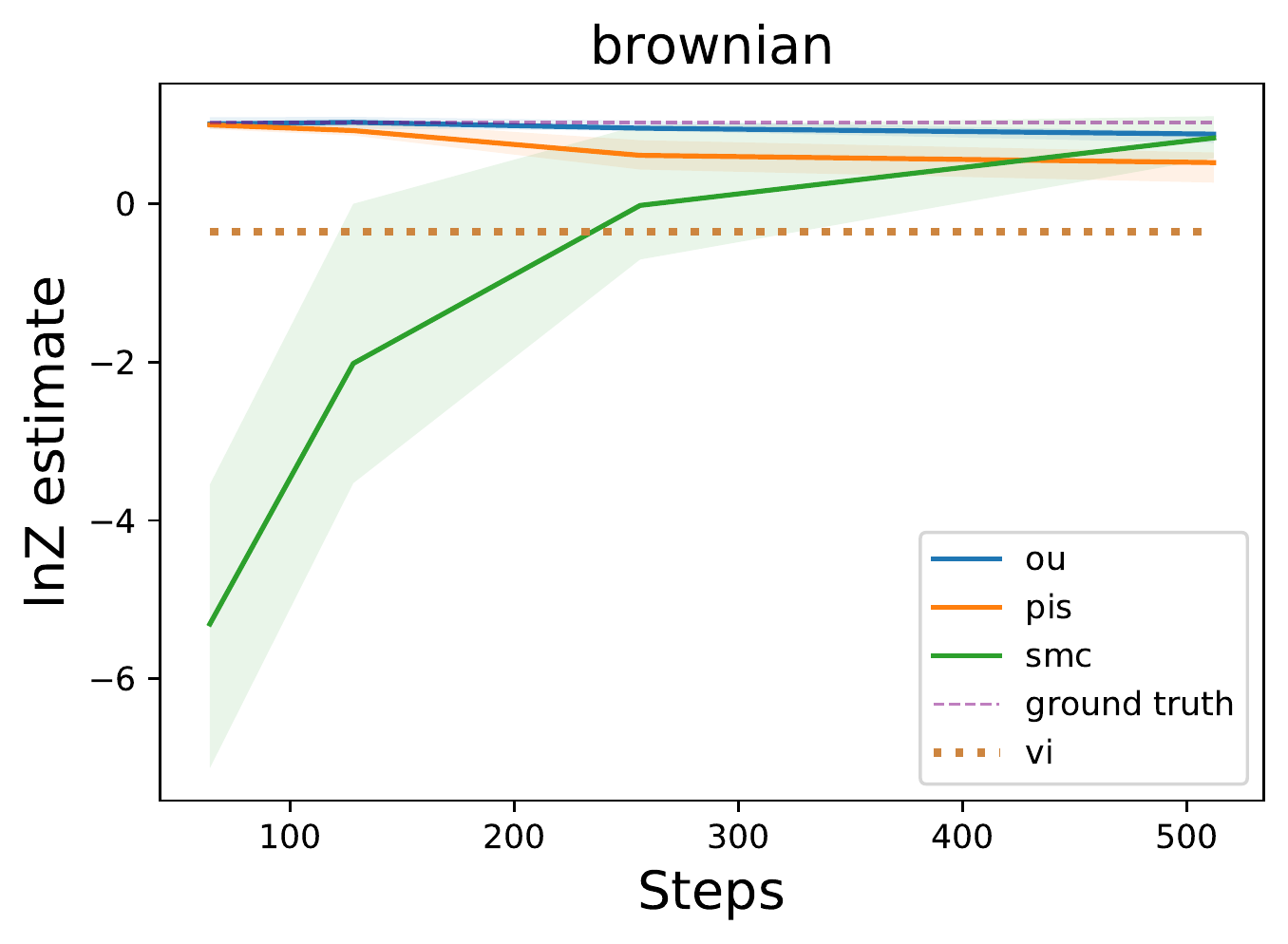}
    \end{minipage}
     \hspace*{\fill}
    \hspace*{\fill}\begin{minipage}{0.32\linewidth }
    \centering
    \includegraphics[width=\linewidth ]{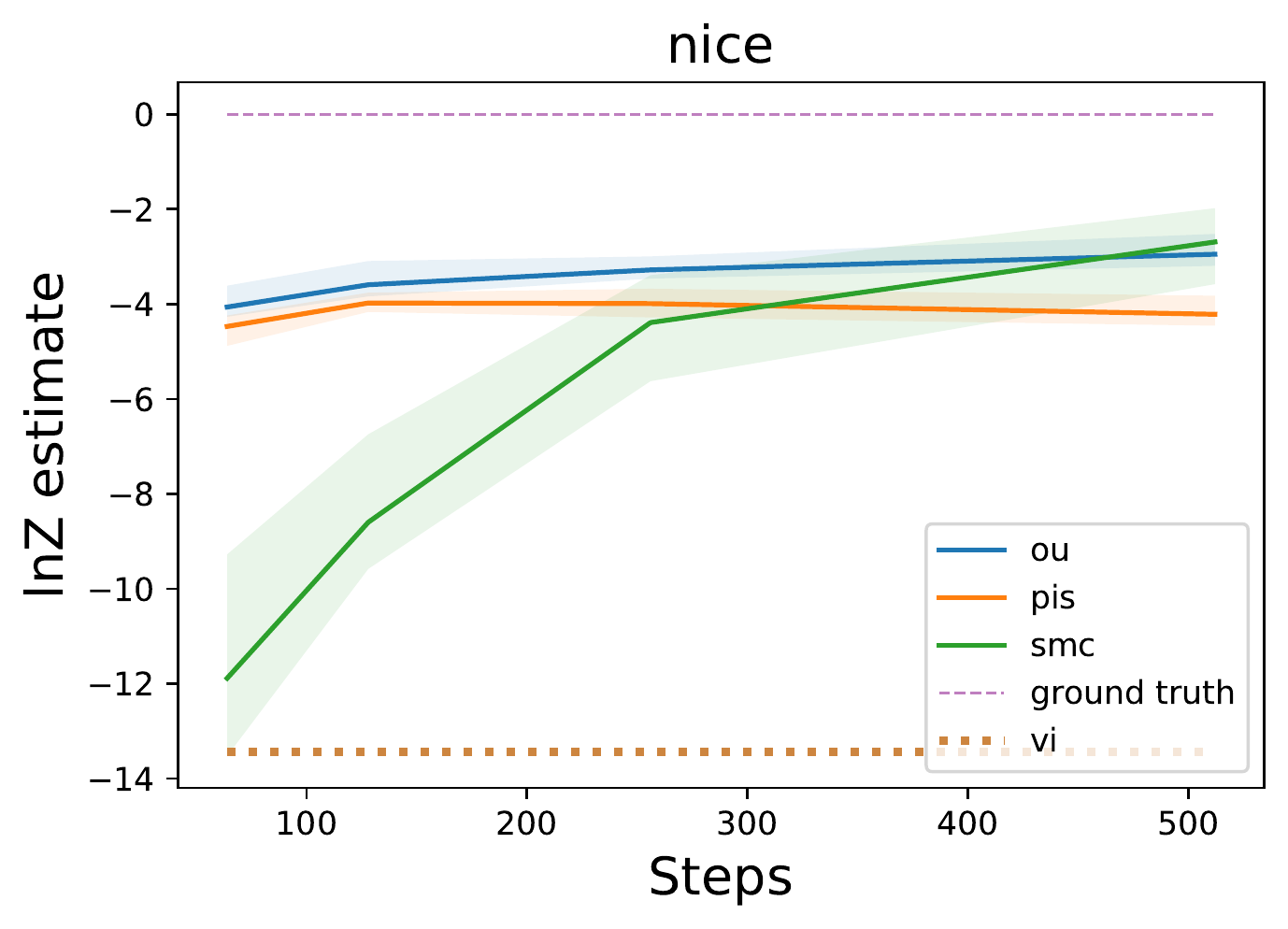}
    \end{minipage}
    
    \caption{$\log Z$ estimate as a function of number of steps $K$ - a) Logistic Sonar dataset, b) Brownian motion, c) NICE. Yellow dotted line is MF-VI and dashed magenta is the gold standard. \label{fig:logreg_targets}} 
\end{figure*}

\textbf{Forward KL minimization.}
In \cite{jing2022torsional}, diffusion models ideas are also use to sample unnormalized probability densities. The criterion being minimized therein is the forward KL as in DDPM, that is $\KL(p||q^{\theta})$. As samples from $p$ are not available, an importance sampling approximation of $p$ based on samples from $q_\theta$ is used to obtain an estimate of this KL and its gradient w.r.t. $\theta$. The method is shown to perform well in low dimensional examples but is expected to degrade significantly in high dimension as the importance sampling approximation will be typically poor in the first training iterations.

\section{Experiments}\label{sec:experiments}
We present here experiments for Algorithm \ref{alg:dds_alg}. In our implementation, $f_\theta$ follows the PIS-GRAD network proposed in \citep{zhangyongxinchen2021path}: $ f_\theta(k, x) \!=\! \mathrm{NN}_1(k, x;\theta) + \mathrm{NN}_2(k;\theta) \odot \nabla \ln  \pi(x)$. Across all experiments we use a two layer architecture with 64 hidden units each (for both networks), as in \cite{zhangyongxinchen2021path}, with the exception of the NICE \citep{dinh2014nice} target where we use 3 layers with 512, 256, and 64 units respectively. The final layers are initialised to $0$ in order to make the path regularisation term null. We use $\alpha_k^{1/2} \!\propto\! \alpha_{\max}^{1/2}\cos^2\left(\frac{\pi}{2}\frac{1- k  /K  + s}{1 +s}\right)$ with $s\!=\!0.008$ as in \citep{nichol2021improved}.
We found that detaching the target score stabilised optimization in both approaches without affecting the final result. We adopt this across experiments, an ablation of this feature can be seen in Appendix \ref{sec:detach_exp}. 

Across all tasks we compare DDS to SMC \citep{del2006sequential,zhou2016toward}, PIS \citep{barrlamacraft2020quantum,vargas2021bayesian,zhangyongxinchen2021path}, and Mean Field-VI (MF-VI) with a Gaussian variational distribution.  Finally we explore a task introduced in \citep{doucet2022annealed} that uses a pre-trained normalising flow as a target. Within this setting we propose a benchmarking criterion that allows us to assess mode collapse in high dimensions and explore the benefits of incorporating inductive biases into $f_\theta$. We carefully tuned the hyper-parameters of all algorithms (e.g. step size, diffusion coefficient, and such), details can be found in Appendix \ref{apdx:tune}. Finally training time can be found in Appendix \ref{sec:training_time}.
Additional experiments for the normalizing flows are presented in Appendix \ref{app:probaflowODE} and for the underdamped approach in Appendix \ref{apdx:undexp}. We note that these extensions did \emph{not} bring any benefit compared to Algorithm \ref{alg:dds_alg}.


\subsection{Benchmarking Targets}
We first discuss two standard target distributions which are often used to benchmark methods; see e.g. \citep{Neal:2003,arbel2021annealed,heng2020controlled,zhangyongxinchen2021path}. Results are presented in Figure \ref{fig:base_targets}.

\textbf{Funnel Distribution:} This 10-dimensional challenging distribution is given by $\gamma(x_{1:10})= \mathcal{N}(x_1;0, \sigma_{f}^2)\mathcal{N}(x_{2:10};0, \exp(x_1) I)$,
where $\sigma_{f}^2=9$ \citep{Neal:2003}.
\footnote{\cite{zhangyongxinchen2021path} inadvertently considered the more favourable $\sigma_f=1$ scenario for PIS but used $\sigma_f=3$ for other methods. This explains the significant differences between their results and ours.}

\textbf{Log Gaussian Cox process:} This model arises in spatial statistics \citep{Moller:1998}. We use a $d\!=\!M \times M\!=\!1600$ grid, resulting in the unnormalized target density $\gamma(x) = \mathcal{N}(x ; \mu, K) \textstyle\prod_{i \in [1:M]^{2}} \exp(x_{i} y_{i} - a \exp(x_i) )$.
\subsection{Bayesian Models}
We explore two standard Bayesian models and compare them with standard AIS and MCD benchmarks (See  Appendix \ref{app:mcd}) in addition to the SMC and VI benchmarks presented so far. Results for Ionosphere can be found in the 3rd pane of Figure \ref{fig:base_targets} whilst SONAR and Brownian are in Figure \ref{fig:logreg_targets}.

\textbf{Logistic Regression:} We set $ x \sim \mathcal{N}(0, \sigma_w^2 I),y_i \sim \mathrm{Bernoulli}(\mathrm{sigmoid}( x^\top u_i))$. This Bayesian logistic model is evaluated on two datasets, Ionosphere ($d=32$) and Sonar ($d=61$).

\textbf{Brownian Motion:} We consider a discretised Brownian motion with a Gaussian observation model and a latent volatility as a time series model, $d=32$. This model, proposed in the  software package developed by \cite{inferencegym2020}, is specified in more detail in Appendix \ref{apdx:brownian}.

\subsection{Mode Collapse In High Dimensions}
\paragraph{Normalizing Flow Evaluation:} Following \cite{doucet2022annealed} we train NICE \citep{dinh2014nice} on a down-sampled $d=14 \times 14$ variant of MNIST \citep{mnist} and use the trained model as our target. As we can generate samples from our target, we evaluate the methods samples by measuring the Sinkhorn distance between true and estimated samples. This evaluation criteria allows to asses mode collapse for samplers in high dimensional settings. Results can be seen in Table \ref{tab:was} and the third pane of Figure \ref{fig:logreg_targets}.

\begin{table}[H]
\centering
\begin{tabular}{@{}ccccc@{}}
\toprule
Method & DDS  & PIS & SMC & MCD \\ \midrule
$\ln Z$ Estimate     & $-3.204 \pm 0.645$  & $-3.933 \pm 0.754$  & $-4.255 \pm 2.043$ &       $-6.25 $                                  \\
$\gW^2_{\gamma=0.01}(\pi_{\mathrm{true}}, \hat{\pi})$    & $658.079$  & $658.778$ & $750.245$ &  NA \\
 \bottomrule
\end{tabular}
\caption{Results on NICE target, we performed 30 runs with different seeds for each approach. For MCD we used the results from \citep{doucet2022annealed}. \label{tab:was}}
\end{table}
\section{Discussion}
We have explored the use of DDPM ideas to sample unnormalized probability distributions and estimate their normalizing constants. 

The DDS in Algorithm \ref{alg:dds_alg} is empirically competitive with state-of-the-art SMC and numerically more stable than PIS. This comes at the cost of a non-negligible training time compared to SMC. When accounting for it, SMC often provide better performance on simple targets. However, in the challenging multimodal NICE example, even a carefully tuned SMC sampler using Hamiltonian Monte Carlo transitions was not competitive to DDS. This is despite the fact that DDS (and PIS) are prone to mode dropping as any method relying on the reverse KL.

We have also investigated normalizing flows based on the probability flow ODE as well as DDS based on an underdamped dynamics. Our experimental results were disappointing in both cases in high dimensional scenarios. We conjecture that more sophisticated numerical integrators need to be developed for the normalizing flows to be competitive and that the neural network parameterization used in the underdamped scenario should be improved and better leverage the structure of the logarithmic derivative of the  value functions.   

Overall DDS are a class of algorithms worth investigating and further developing. There has much work devoted to improving successfully DDPM over the past two years including, among many others, modified forward noising mechanisms \citep{hoogeboom2022blurring}, denoising diffusion implicit models \citep{song2020denoising} and sophisticated numerical integrators \citep{karras2022elucidating}. It is highly likely that some of these techniques will lead to more powerful DDS. Advances in KL and path integral control might also be adapted to DDS to provide better training procedures; see e.g. \cite{thalmeier2020adaptive}.  

\bibliography{iclr2023_conference,references}

\begin{thebibliography}{63}
\providecommand{\natexlab}[1]{#1}
\providecommand{\url}[1]{\texttt{#1}}
\expandafter\ifx\csname urlstyle\endcsname\relax
  \providecommand{\doi}[1]{doi: #1}\else
  \providecommand{\doi}{doi: \begingroup \urlstyle{rm}\Url}\fi

\bibitem[Agakov \& Barber(2004)Agakov and Barber]{Agakov2004}
Felix~V Agakov and David Barber.
\newblock An auxiliary variational method.
\newblock In \emph{Advances in Neural Information Processing Systems}, 2004.

\bibitem[Arbel et~al.(2021)Arbel, Matthews, and Doucet]{arbel2021annealed}
Michael Arbel, Alex Matthews, and Arnaud Doucet.
\newblock Annealed flow transport {M}onte {C}arlo.
\newblock In \emph{International Conference on Machine Learning}, 2021.

\bibitem[Barr et~al.(2020)Barr, Gispen, and
  Lamacraft]{barrlamacraft2020quantum}
Ariel Barr, Willem Gispen, and Austen Lamacraft.
\newblock Quantum ground states from reinforcement learning.
\newblock In \emph{Mathematical and Scientific Machine Learning}, 2020.

\bibitem[Beghi(1996)]{beghi1996relative}
Alessandro Beghi.
\newblock On the relative entropy of discrete-time {M}arkov processes with
  given end-point densities.
\newblock \emph{IEEE Transactions on Information Theory}, 42\penalty0
  (5):\penalty0 1529--1535, 1996.

\bibitem[Bradbury et~al.(2018)Bradbury, Frostig, Hawkins, Johnson, Leary,
  Maclaurin, Necula, Paszke, Vander{P}las, Wanderman-{M}ilne, and
  Zhang]{jax2018github}
James Bradbury, Roy Frostig, Peter Hawkins, Matthew~James Johnson, Chris Leary,
  Dougal Maclaurin, George Necula, Adam Paszke, Jake Vander{P}las, Skye
  Wanderman-{M}ilne, and Qiao Zhang.
\newblock {JAX}: composable transformations of {P}ython+{N}um{P}y programs,
  2018.
\newblock URL \url{http://github.com/google/jax}.

\bibitem[Chen et~al.(2018)Chen, Rubanova, Bettencourt, and
  Duvenaud]{chen2018neural}
Ricky~TQ Chen, Yulia Rubanova, Jesse Bettencourt, and David~K Duvenaud.
\newblock Neural ordinary differential equations.
\newblock \emph{Advances in Neural Information Processing Systems}, 2018.

\bibitem[Chen et~al.(2022)Chen, Chewi, Li, Li, Salim, and
  Zhang]{ChenChewiSinhosamplingaseasyaslearningscore2022}
Sitan Chen, Sinho Chewi, Jerry Li, Yuanzhi Li, Adil Salim, and Anru~R. Zhang.
\newblock Sampling is as easy as learning the score: theory for diffusion
  models with minimal data assumptions.
\newblock \emph{arXiv preprint arXiv:22209.11215}, 2022.

\bibitem[Dai~Pra(1991)]{daipra1991stochastic}
Paolo Dai~Pra.
\newblock A stochastic control approach to reciprocal diffusion processes.
\newblock \emph{Applied Mathematics and Optimization}, 23\penalty0
  (1):\penalty0 313--329, 1991.

\bibitem[De~Bortoli(2022)]{debortoli2022convergence}
Valentin De~Bortoli.
\newblock Convergence of denoising diffusion models under the manifold
  hypothesis.
\newblock \emph{Transactions on Machine Learning Research}, 2022.

\bibitem[De~Bortoli et~al.(2021)De~Bortoli, Thornton, Heng, and
  Doucet]{debortoli2021diffusion}
Valentin De~Bortoli, James Thornton, Jeremy Heng, and Arnaud Doucet.
\newblock Diffusion {S}chr{\"o}dinger bridge with applications to score-based
  generative modeling.
\newblock In \emph{Advances in Neural Information Processing Systems}, 2021.

\bibitem[Del~Moral et~al.(2006)Del~Moral, Doucet, and Jasra]{del2006sequential}
Pierre Del~Moral, Arnaud Doucet, and Ajay Jasra.
\newblock Sequential {M}onte {C}arlo samplers.
\newblock \emph{Journal of the Royal Statistical Society: Series B (Statistical
  Methodology)}, 68\penalty0 (3):\penalty0 411--436, 2006.

\bibitem[Dinh et~al.(2014)Dinh, Krueger, and Bengio]{dinh2014nice}
Laurent Dinh, David Krueger, and Yoshua Bengio.
\newblock Nice: Non-linear independent components estimation.
\newblock \emph{arXiv preprint arXiv:1410.8516}, 2014.

\bibitem[Dockhorn et~al.(2022)Dockhorn, Vahdat, and Kreis]{Dockhorn2022}
Tim Dockhorn, Arash Vahdat, and Karsten Kreis.
\newblock Score-based generative modeling with critically-damped {L}angevin
  diffusion.
\newblock In \emph{International Conference on Learning Representations}, 2022.

\bibitem[Doucet et~al.(2022)Doucet, Grathwohl, Matthews, and
  Strathmann]{doucet2022annealed}
Arnaud Doucet, Will Grathwohl, Alexander G de~G Matthews, and Heiko Strathmann.
\newblock Score-based diffusion meets annealed importance sampling.
\newblock In \emph{Advances in Neural Information Processing Systems}, 2022.

\bibitem[F\"ollmer(1984)]{follmer1984entropy}
Hans F\"ollmer.
\newblock An entropy approach to the time reversal of diffusion processes.
\newblock \emph{Lecture Notes in Control and Information Sciences},
  69:\penalty0 156--163, 1984.

\bibitem[{Geffner} \& {Domke}(2021){Geffner} and {Domke}]{Geffner:2021}
Tomas {Geffner} and Justin {Domke}.
\newblock {MCMC} variational inference via uncorrected {H}amiltonian annealing.
\newblock In \emph{Advances in Neural Information Processing Systems}, 2021.

\bibitem[Geffner \& Domke(2022)Geffner and Domke]{geffner2022langevin}
Tomas Geffner and Justin Domke.
\newblock Langevin diffusion variational inference.
\newblock \emph{arXiv preprint arXiv:2208.07743}, 2022.

\bibitem[Grathwohl et~al.(2018)Grathwohl, Chen, Bettencourt, Sutskever, and
  Duvenaud]{grathwohl2018ffjord}
Will Grathwohl, Ricky~TQ Chen, Jesse Bettencourt, Ilya Sutskever, and David
  Duvenaud.
\newblock Ffjord: Free-form continuous dynamics for scalable reversible
  generative models.
\newblock In \emph{International Conference on Learning Representations}, 2018.

\bibitem[Greff et~al.(2019)Greff, Kaufman, Kabra, Watters, Burgess, Zoran,
  Matthey, Botvinick, and Lerchner]{greff2019multi}
Klaus Greff, Rapha{\"e}l~Lopez Kaufman, Rishabh Kabra, Nick Watters,
  Christopher Burgess, Daniel Zoran, Loic Matthey, Matthew Botvinick, and
  Alexander Lerchner.
\newblock Multi-object representation learning with iterative variational
  inference.
\newblock In \emph{International Conference on Machine Learning}, 2019.

\bibitem[Haussmann \& Pardoux(1986)Haussmann and Pardoux]{haussmann1986time}
Ulrich~G Haussmann and Etienne Pardoux.
\newblock Time reversal of diffusions.
\newblock \emph{The Annals of Probability}, 14\penalty0 (3):\penalty0
  1188--1205, 1986.

\bibitem[Heng et~al.(2020)Heng, Bishop, Deligiannidis, and
  Doucet]{heng2020controlled}
Jeremy Heng, Adrian~N Bishop, George Deligiannidis, and Arnaud Doucet.
\newblock Controlled sequential {M}onte {C}arlo.
\newblock \emph{The Annals of Statistics}, 48\penalty0 (5):\penalty0
  2904--2929, 2020.

\bibitem[Hessel et~al.(2020)Hessel, Budden, Viola, Rosca, Sezener, and
  Hennigan]{optax2020github}
Matteo Hessel, David Budden, Fabio Viola, Mihaela Rosca, Eren Sezener, and Tom
  Hennigan.
\newblock Optax: composable gradient transformation and optimisation, in jax.,
  2020.

\bibitem[Ho et~al.(2020)Ho, Jain, and Abbeel]{ho2020denoising}
Jonathan Ho, Ajay Jain, and Pieter Abbeel.
\newblock Denoising diffusion probabilistic models.
\newblock In \emph{Advances in Neural Information Processing Systems}, 2020.

\bibitem[Hoogeboom \& Salimans(2022)Hoogeboom and
  Salimans]{hoogeboom2022blurring}
Emiel Hoogeboom and Tim Salimans.
\newblock Blurring diffusion models.
\newblock \emph{arXiv preprint arXiv:2209.05557}, 2022.

\bibitem[Hyv{\"a}rinen(2005)]{Hyvarinen:2005a}
Aapo Hyv{\"a}rinen.
\newblock Estimation of non-normalized statistical models by score matching.
\newblock \emph{The Journal of Machine Learning Research}, 6:\penalty0
  695--709, 2005.

\bibitem[Jing et~al.(2022)Jing, Corso, Chang, Barzilay, and
  Jaakkola]{jing2022torsional}
Bowen Jing, Gabriele Corso, Jeffrey Chang, Regina Barzilay, and Tommi Jaakkola.
\newblock Torsional diffusion for molecular conformer generation.
\newblock \emph{arXiv preprint arXiv:2206.01729}, 2022.

\bibitem[Kappen \& Ruiz(2016)Kappen and Ruiz]{kappen2016adaptive}
Hilbert~J Kappen and Hans~Christian Ruiz.
\newblock Adaptive importance sampling for control and inference.
\newblock \emph{Journal of Statistical Physics}, 162\penalty0 (5):\penalty0
  1244--1266, 2016.

\bibitem[Kappen et~al.(2012)Kappen, G{\'o}mez, and Opper]{kappen2012optimal}
Hilbert~J Kappen, Vicen{\c{c}} G{\'o}mez, and Manfred Opper.
\newblock Optimal control as a graphical model inference problem.
\newblock \emph{Machine learning}, 87\penalty0 (2):\penalty0 159--182, 2012.

\bibitem[Karras et~al.(2022)Karras, Aittala, Aila, and
  Laine]{karras2022elucidating}
Tero Karras, Miika Aittala, Timo Aila, and Samuli Laine.
\newblock Elucidating the design space of diffusion-based generative models.
\newblock In \emph{Advances in Neural Information Processing Systems}, 2022.

\bibitem[Kingma \& Ba(2015)Kingma and Ba]{Kingma:2014}
Diederik~P. Kingma and Jimmy Ba.
\newblock Adam: {A} method for stochastic optimization.
\newblock In \emph{International Conference on Learning Representations}, 2015.

\bibitem[Kingma \& Welling(2013)Kingma and Welling]{kingma2013auto}
Diederik~P Kingma and Max Welling.
\newblock Auto-encoding variational {B}ayes.
\newblock \emph{arXiv preprint arXiv:1312.6114}, 2013.

\bibitem[Klebaner(2012)]{klebaner2012introduction}
Fima~C Klebaner.
\newblock \emph{Introduction to Stochastic Calculus with Applications}.
\newblock Imperial College Press, 2012.

\bibitem[LeCun \& Cortes(2010)LeCun and Cortes]{mnist}
Yann LeCun and Corinna Cortes.
\newblock {MNIST} handwritten digit database.
\newblock http://yann.lecun.com/exdb/mnist/, 2010.
\newblock URL \url{http://yann.lecun.com/exdb/mnist/}.

\bibitem[LeCun et~al.(1998)LeCun, Bottou, Bengio, and
  Haffner]{lecun1998gradient}
Yann LeCun, L{\'e}on Bottou, Yoshua Bengio, and Patrick Haffner.
\newblock Gradient-based learning applied to document recognition.
\newblock \emph{Proceedings of the IEEE}, 86\penalty0 (11):\penalty0
  2278--2324, 1998.

\bibitem[Lee et~al.(2022)Lee, Lu, and Tan]{lee2022convergencenocondition}
Holden Lee, Jianfeng Lu, and Yixin Tan.
\newblock Convergence of score-based generative modeling for general data
  distributions.
\newblock \emph{arXiv preprint arXiv:2209.12381}, 2022.

\bibitem[Leimkuhler \& Matthews(2016)Leimkuhler and
  Matthews]{leimkuhler2016molecular}
Ben Leimkuhler and Charles Matthews.
\newblock \emph{Molecular Dynamics with Deterministic and Stochastic Numerical
  Methods}.
\newblock Springer, 2016.

\bibitem[L{\'e}onard(2014{\natexlab{a}})]{leonard2013survey}
Christian L{\'e}onard.
\newblock A survey of the {S}chr{\"o}dinger problem and some of its connections
  with optimal transport.
\newblock \emph{Discrete and Continuous Dynamical Systems-Series A},
  34\penalty0 (4):\penalty0 1533--1574, 2014{\natexlab{a}}.

\bibitem[L{\'e}onard(2014{\natexlab{b}})]{leonard2014some}
Christian L{\'e}onard.
\newblock Some properties of path measures.
\newblock In \emph{S{\'e}minaire de Probabilit{\'e}s XLVI}, pp.\  207--230.
  Springer, 2014{\natexlab{b}}.

\bibitem[Metz et~al.(2020)Metz, Maheswaranathan, Sun, Freeman, Poole, and
  Sohl-Dickstein]{metz2020using}
Luke Metz, Niru Maheswaranathan, Ruoxi Sun, C~Daniel Freeman, Ben Poole, and
  Jascha Sohl-Dickstein.
\newblock Using a thousand optimization tasks to learn hyperparameter search
  strategies.
\newblock \emph{arXiv preprint arXiv:2002.11887}, 2020.

\bibitem[M{\o}ller et~al.(1998)M{\o}ller, Syversveen, and
  Waagepetersen]{Moller:1998}
Jesper M{\o}ller, Anne~Randi Syversveen, and Rasmus~Plenge Waagepetersen.
\newblock Log {G}aussian {C}ox processes.
\newblock \emph{Scandinavian Journal of Statistics}, 25\penalty0 (3):\penalty0
  451--482, 1998.

\bibitem[Neal(2001)]{neal2001annealed}
Radford~M Neal.
\newblock Annealed importance sampling.
\newblock \emph{Statistics and Computing}, 11\penalty0 (2):\penalty0 125--139,
  2001.

\bibitem[Neal(2003)]{Neal:2003}
Radford~M. Neal.
\newblock Slice sampling.
\newblock \emph{The Annals of Statistics}, 31:\penalty0 705--767, 06 2003.

\bibitem[Nichol \& Dhariwal(2021)Nichol and Dhariwal]{nichol2021improved}
Alexander~Quinn Nichol and Prafulla Dhariwal.
\newblock Improved denoising diffusion probabilistic models.
\newblock In \emph{International Conference on Machine Learning}, 2021.

\bibitem[Papamakarios et~al.(2021)Papamakarios, Nalisnick, Rezende, Mohamed,
  and Lakshminarayanan]{papamakarios2019normalizing}
George Papamakarios, Eric Nalisnick, Danilo~Jimenez Rezende, Shakir Mohamed,
  and Balaji Lakshminarayanan.
\newblock Normalizing flows for probabilistic modeling and inference.
\newblock \emph{The Journal of Machine Learning Research}, 22\penalty0
  (1):\penalty0 2617--2680, 2021.

\bibitem[Parmas et~al.(2018)Parmas, Rasmussen, Peters, and
  Doya]{parmas2018pipps}
Paavo Parmas, Carl~Edward Rasmussen, Jan Peters, and Kenji Doya.
\newblock Pipps: Flexible model-based policy search robust to the curse of
  chaos.
\newblock In \emph{International Conference on Machine Learning}, 2018.

\bibitem[Sohl-Dickstein et~al.(2015)Sohl-Dickstein, Weiss, Maheswaranathan, and
  Ganguli]{sohl2015deep}
Jascha Sohl-Dickstein, Eric Weiss, Niru Maheswaranathan, and Surya Ganguli.
\newblock Deep unsupervised learning using nonequilibrium thermodynamics.
\newblock In \emph{International Conference on Machine Learning}, 2015.

\bibitem[Song et~al.(2021{\natexlab{a}})Song, Meng, and
  Ermon]{song2020denoising}
Jiaming Song, Chenlin Meng, and Stefano Ermon.
\newblock Denoising diffusion implicit models.
\newblock In \emph{International Conference on Learning Representations},
  2021{\natexlab{a}}.

\bibitem[Song et~al.(2021{\natexlab{b}})Song, Durkan, Murray, and
  Ermon]{song2021maximum}
Yang Song, Conor Durkan, Iain Murray, and Stefano Ermon.
\newblock Maximum likelihood training of score-based diffusion models.
\newblock In \emph{Advances in Neural Information Processing Systems},
  2021{\natexlab{b}}.

\bibitem[Song et~al.(2021{\natexlab{c}})Song, Sohl-Dickstein, Kingma, Kumar,
  Ermon, and Poole]{song2020score}
Yang Song, Jascha Sohl-Dickstein, Diederik~P Kingma, Abhishek Kumar, Stefano
  Ermon, and Ben Poole.
\newblock Score-based generative modeling through stochastic differential
  equations.
\newblock In \emph{International Conference on Learning Representations},
  2021{\natexlab{c}}.

\bibitem[Sottinen \& S{\"a}rkk{\"a}(2008)Sottinen and
  S{\"a}rkk{\"a}]{sottinen2008application}
Tommi Sottinen and Simo S{\"a}rkk{\"a}.
\newblock Application of {G}irsanov theorem to particle filtering of discretely
  observed continuous-time non-linear systems.
\newblock \emph{Bayesian Analysis}, 3\penalty0 (3):\penalty0 555--584, 2008.

\bibitem[Sountsov et~al.(2020)Sountsov, Radul, and
  contributors]{inferencegym2020}
Pavel Sountsov, Alexey Radul, and contributors.
\newblock Inference gym, 2020.
\newblock URL \url{https://pypi.org/project/inference_gym}.

\bibitem[Thalmeier et~al.(2020)Thalmeier, Kappen, Totaro, and
  G\'omez]{thalmeier2020adaptive}
Dominik Thalmeier, Hilbert~J Kappen, Simone Totaro, and Vicenc G\'omez.
\newblock Adaptive smoothing for path integral control.
\newblock \emph{Journal of Machine Learning Research}, 21:\penalty0 1--37,
  2020.

\bibitem[Thin et~al.(2021)Thin, Kotelevskii, Durmus, Moulines, Panov, and
  Doucet]{Thin:2021}
Achille Thin, Nikita Kotelevskii, Alain Durmus, Eric Moulines, Maxim Panov, and
  Arnaud Doucet.
\newblock Monte {C}arlo variational auto-encoders.
\newblock In \emph{International Conference on Machine Learning}, 2021.

\bibitem[Tzen \& Raginsky(2019)Tzen and Raginsky]{tzen2019theoretical}
Belinda Tzen and Maxim Raginsky.
\newblock Theoretical guarantees for sampling and inference in generative
  models with latent diffusions.
\newblock In \emph{Conference on Learning Theory}, 2019.

\bibitem[Vargas et~al.(2021)Vargas, Ovsianas, Fernandes, Girolami, Lawrence,
  and N{\"u}sken]{vargas2021bayesian}
Francisco Vargas, Andrius Ovsianas, David Fernandes, Mark Girolami, Neil
  Lawrence, and Nikolas N{\"u}sken.
\newblock Bayesian learning via neural {S}chr{\"o}dinger--{F}{\"o}llmer flows.
\newblock \emph{arXiv preprint arXiv:2111.10510}, 2021.

\bibitem[Vicol et~al.(2021)Vicol, Metz, and Sohl-Dickstein]{vicol2021unbiased}
Paul Vicol, Luke Metz, and Jascha Sohl-Dickstein.
\newblock Unbiased gradient estimation in unrolled computation graphs with
  persistent evolution strategies.
\newblock In \emph{International Conference on Machine Learning}, 2021.

\bibitem[Vincent(2011)]{vincent2011connection}
Pascal Vincent.
\newblock A connection between score matching and denoising autoencoders.
\newblock \emph{Neural Computation}, 23\penalty0 (7):\penalty0 1661--1674,
  2011.

\bibitem[Wainwright \& Jordan(2008)Wainwright and Jordan]{Wainwright:2008}
Martin~J. Wainwright and Michael~I. Jordan.
\newblock Graphical {Models}, {Exponential} {Families}, and {Variational}
  {Inference}.
\newblock \emph{Foundations and Trends{\textregistered} in Machine Learning},
  1\penalty0 (1--2):\penalty0 1--305, November 2008.

\bibitem[Wu et~al.(2020)Wu, K{\"o}hler, and No{\'e}]{wunoe2020stochastic}
Hao Wu, Jonas K{\"o}hler, and Frank No{\'e}.
\newblock Stochastic normalizing flows.
\newblock In \emph{Advances in Neural Information Processing Systems}, 2020.

\bibitem[Zhang et~al.(2021)Zhang, Sahai, and Marzouk]{zhangmarzouk2021sampling}
Benjamin Zhang, Tuhin Sahai, and Youssef Marzouk.
\newblock Sampling via controlled stochastic dynamical systems.
\newblock In \emph{I (Still) Can't Believe It's Not Better! NeurIPS 2021
  Workshop}, 2021.

\bibitem[{Zhang} et~al.(2021){Zhang}, {Hsu}, {Li}, {Finn}, and
  {Grosse}]{ZhangAIS2021}
Guodong {Zhang}, Kyle {Hsu}, Jianing {Li}, Chelsea {Finn}, and Roger {Grosse}.
\newblock Differentiable annealed importance sampling and the perils of
  gradient noise.
\newblock In \emph{Advances in Neural Information Processing Systems}, 2021.

\bibitem[Zhang \& Chen(2022)Zhang and Chen]{zhangyongxinchen2021path}
Qinsheng Zhang and Yongxin Chen.
\newblock Path integral sampler: a stochastic control approach for sampling.
\newblock In \emph{International Conference on Learning Representations}, 2022.

\bibitem[Zhou et~al.(2016)Zhou, Johansen, and Aston]{zhou2016toward}
Yan Zhou, Adam~M Johansen, and John~AD Aston.
\newblock Toward automatic model comparison: an adaptive sequential monte carlo
  approach.
\newblock \emph{Journal of Computational and Graphical Statistics}, 25\penalty0
  (3):\penalty0 701--726, 2016.

\end{thebibliography}
\bibliographystyle{iclr2023_conference}

\appendix

\section{Appendix}
\subsection{Proof of Proposition \ref{prop:RNcontinuous}}
The Radon--Nikodym derivative expression (\ref{eq:RadonNykodyn}) directly follows from an application of Girsanov theorem (see e.g. \cite{klebaner2012introduction}).
The path measures $\mathcal{P}$ and $\mathcal{P}^{\textup{ref}}$ are induced by two diffusions following the same dynamics but with different initial conditions so that
\begin{align*}
    \mathcal{P}(\omega)&=\mathcal{P}(\omega|y_T)\pi(y_T)\\
                     &=\mathcal{P}^{\textup{ref}}(\omega|y_T)\pi(y_T)\\
                     &=\mathcal{P}^{\textup{ref}}(\omega|y_T)\mathcal{N}(y_T;0,\sigma^2I)\frac{\pi(y_T)}{\mathcal{N}(y_T;0,\sigma^2I)}\\ &=\mathcal{P}^{\textup{ref}}(\omega)\frac{\pi(y_T)}{\mathcal{N}(y_T;0,\sigma^2I)}.
\end{align*}
So it follows directly that  $\KL(\mathcal{Q}^\theta||\mathcal{P})=\KL(\mathcal{Q}^\theta||\mathcal{P}^{\textup{ref}})+\mathbb{E}_{y_T \sim q^{\theta}_0}\Bigr[\ln \left(\frac{p^{\textup{ref}}_0(y_T)}{p_0(y_T)}\right)\Bigr]$. Note we apply Girsanov theorem to the time reversals of the path measures $\mathcal{Q}^\theta$ and $\mathcal{P}$ using that the Radon--Nikodym between two path measures and their respective reversals are the same. As $W_t$ is a Brownian motion under $\mathcal{Q}^\theta$, the final expression (\ref{eq:KLpathintegral}) for $\KL(\mathcal{Q}^\theta||\mathcal{P})$ follows. 

\subsection{Comparing DDS and PIS Drifts for Gaussian targets}\label{app:comparisonDDSPIS}
The optimal drift for both PIS and DDS can be expressed in terms of logarithmic derivative of the value function plus additional terms:
\begin{align}\label{eq:drifteq}
     b_{\mathrm{DDS}}(x,t) = -\beta_{T-t}(x-2 \sigma^2 \nabla_x \ln \phi_{T-t}(x) ),  \quad  b_{\mathrm{PIS}}(x,t) = \sigma^2 \nabla_x \ln \phi_{T-t}(x)
\end{align}
where $\phi_t$ is the corresponding value function for DDS or PIS respectively.

Recall that $\log \phi_t(x)=\log p_t(x)-\log p^{\textup{ref}}_{t}(x)$. For a target $\pi(x)=\mathcal{N}(x;\mu,\Sigma)$, $p^{\textup{ref}}_{t}(x)=\mathcal{N}(x_t;\mu^{\textup{ref}}_t,b^{\textup{ref}}_t I)$  and $p_{t|0}(x_t|x_0)=\mathcal{N}(x_t;a_t x_0,b_t I)$, we obtain  
\begin{equation*}
    p_t(x)=\int \pi(x_0) p_{t|0}(x|x_0) \mathrm{d}x_0 =\mathcal{N}(x;a_t \mu, a^2_t \Sigma+b_t I)
\end{equation*}
so 
\begin{equation}
    \nabla \log \phi_t(x)=-(a^2_t \Sigma+b_t I)^{-1}(x-a_t \mu)+(b^{\textup{ref}}_{t})^{-1}(x-\mu^{\textup{ref}}_t).
\end{equation}

\begin{corollary}\label{drift:dds}
For DDS, we have $\mu^{\textup{ref}}_t=0$, $b^{\textup{ref}}_t=\sigma^2 I$, $a_t=\sqrt{1-\lambda_t}$, $b_t=\sigma^2 \lambda_t I$ with  $\lambda_t=1-\exp(-2\int^t_0\beta_s \mathrm{d}s)$.
So, for example, for $\sigma=\beta_t=1$ and $\Sigma=I$, we have 
\begin{equation}
      \nabla \log \phi_t(x)=\mu \exp(-t),\quad  b_{\mathrm{DDS}}(x,t) = -x+2 \mu \exp(-(T-t)).
\end{equation}
\end{corollary}
\begin{corollary}\label{drift:pis}
For PIS, we have a reference process which is is a pinned Brownian motion running backwards in time with $\mu^{\textup{ref}}_t=0$, $b^{\textup{ref}}_t=\sigma^2 (T-t)I$, $a_t=\frac{T-t}{T}$, $b_t=\sigma^2 \frac{t(T-t)}{T}$. For $\sigma=1, \Sigma=I$, we obtain 
\begin{equation}
     b_{\mathrm{PIS}}(x,t)= -(a^2_{T-t}+b_{T-t})^{-1}(x-a_{T-t} \mu)+\frac{x}{t}.
\end{equation}
In particular
$b_{\mathrm{PIS}}(x,t) \approx \frac{x}{t}-(x-\mu)$ as $t \rightarrow 0$.
\end{corollary}

Hence for PIS the drift function explodes close to the origin compared to DDS. This explosion holds for any target $\pi$ as it is only related to the fact that $p^{\textup{ref}}_{t}$ concentrates to $\delta_0$ as $t\rightarrow T$. This makes it harder to approximate for a neural network. Additionally when discretizing the resulting diffusion, this means that smaller discretization steps should be used close to $t=0$.

\subsection{Proof of Proposition \ref{prop:notelbo}}

\begin{proof}
Let us express the discrete time version of the reference process as
\begin{align}
    p^{\mathrm{ref}} (y_{0:K}) =   p_0^{\mathrm{ref}}(y_0) \prod_{k=0}^{K-1}  p^{\mathrm{ref}}_{k +1| k} (y_{k+1}| y_k)
\end{align}
and denote by $p^{\mathrm{ref}}_k$ the corresponding marginal density of $y_k$ which satisfies 
\begin{equation}
    p_{k+1}(y_{k+1}) = \int  p^{\mathrm{ref}}_{k +1| k} (y_{k+1}| y_{k})~~ p_{k}(y_k) \mathrm{d}y_k.
\end{equation}
The backward decomposition of this joint distribution is given by
\begin{align}
    p^{\mathrm{ref}} (y_{0:K}) =p_K^{\mathrm{ref}}(y_K)\prod_{k=0}^{K-1}  p^{\mathrm{ref}}_{k| k+1} (y_{k}| y_{k+1}),
\end{align}
where 
\begin{equation}
     p^{\mathrm{ref}}_{k| k+1} (y_{k}| y_{k+1})=\frac{p^{\mathrm{ref}}_{k +1| k} (y_{k+1}| y_k) p_{k}(y_k) }{ p_{k+1}(y_{k+1})}.
\end{equation}
If our chosen integrator induces a transition kernel $p^{\mathrm{ref}}_{k +1| k} (y_{k+1}| y_k)$ which is such that $p^{\mathrm{ref}}_K(y_K) = p_0^{\mathrm{ref}}(y_K)$, then 
\begin{align}
     p^{\mathrm{ref}} (y_{0:K})  \frac{\pi(y_K)}{p_0^{\mathrm{ref}}(y_K)} =  \pi(y_K) \prod_{k=0}^{K-1}  p^{\mathrm{ref}}_{k| k+1} (y_{k}| y_{k+1})
\end{align}
is a valid (normalised) probability density. Hence it follows that 
\begin{align}
    &\KL(q^\theta(y_{0:K})||{p}^{\textup{ref}}(y_{0:K}))+\mathbb{E}_{y_K \sim q^{\theta}_0}\Bigr[\ln \left(\frac{p^{\textup{ref}}_0(y_K)}{\pi(y_K)}\right)\Bigr]  \\
    =&  \KL\left(q^\theta(y_{0:K})\Big|\Big|{p}^{\textup{ref}}(y_{0:K}) \frac{\pi(y_K)}{p_0^{\mathrm{ref}}(y_K)} \right) \geq 0. \nonumber
\end{align}
If the integrator does not preserve the marginals we have that
\begin{align}
     p^{\mathrm{ref}} (y_{0:K})  \frac{\pi(y_K)}{p_0^{\mathrm{ref}}(y_K)} =  \frac{p_K^{\mathrm{ref}}(y_K)}{p_0^{\mathrm{ref}}(y_K)}\pi(y_K)\prod_{k=0}^{K-1}  p^{\mathrm{ref}}_{k| k+1} (y_{k}| y_{k+1}).
\end{align}
This is not a probability density and thus the objective is no longer guaranteed to be positive and consequently the expectation of our estimator of $\ln Z$ will not be necessarily a lower bound for $\ln Z$. Finally simple calculations show that the Euler discretisation does not preserve the invariant distribution of $\gP^{\mathrm{ref}}$ in DDS.
\end{proof}

\subsection{Proof of Proposition \ref{prop:KL}}
We have 
\begin{equation}\label{eq:logRNsum}
    \log \left( \frac{q^{\theta}(y_{0:K})}{p(y_{0:K})}\right)=\log  \left(\frac{q^{\theta}(y_{0:K})}{p^{\textup{ref}}(y_{0:K})}\right)+\log  \left( \frac{\mathcal{N}(y_0;0,\sigma^2 I)}{\pi(y_0)}\right).
\end{equation}
Now by construction, we have from (\ref{eq:OUrefintegrated}) that  $p^{\textup{ref}}_{k-1|k}(y_{K-k+1}|y_{K-k})=\mathcal{N}(y_{K-k+1};\sqrt{1-\alpha_k}y_{K-k},\sigma^2 \alpha_k I)$ and from (\ref{eq:integratorapproxrevverse}) we obtain $q^{\theta}_{k-1|k}(y_{K-k+1}|y_{K-k})=\mathcal{N}(y_{K-k+1};\sqrt{1-\alpha_k}y_{K-k}+2\sigma^2(1-\sqrt{1-\alpha_k})f_{\theta}(k,y_{K-k}),\sigma^2 \alpha_k I)$.

It follows that 
\begin{align}
    &\log  \left(\frac{q^{\theta}(y_{0:K})}{p^{\textup{ref}}(y_{0:K})}\right)  \label{eq:rationvareps1} \\
    =&\log \left( \frac{q_{K}(y_0)}{p^{\textup{ref}}_{K}(y_0)} \right)+\sum_{k=1}^K 
    \log \left( \frac{q^{\theta}_{k-1|k}(y_{K-k+1}|y_{K-k})}{p^{\textup{ref}}_{k-1|k}(y_{K-k+1}|y_{K-k})} \right) \nonumber \\
    =& \sum_{k=1}^K  \frac{1}{2\alpha_k \sigma^2} \Bigr[ ||y_{K-k+1}-\sqrt{1-\alpha_k}y_{K-k}||^2 \nonumber\\
    &- ||y_{K-k+1}-\sqrt{1-\alpha_k}y_{K-k}-2\sigma^2(1-\sqrt{1-\alpha_k})f(k,y_{K-k})||^2 \Bigr], \nonumber
\end{align}
where we have exploited the fact that $q_{K}(y_0)=p^{\textup{ref}}_{K}(y_0)=\mathcal{N}(y_0;0,\sigma^2 I)$.

Now using $\vepsilon_k:=\frac{1}{\sigma \sqrt{\alpha_k}}(y_{K-k+1}-\sqrt{1-\alpha_k}y_{K-k}-2\sigma^2(1-\sqrt{1-\alpha_k})f(k,y_{K-k}))$, we can rewrite (\ref{eq:rationvareps1}) as 
\begin{align}\label{eq:logDTRN2_2}
   &\log  \left(\frac{q^{\theta}(y_{0:K})}{p^{\textup{ref}}(y_{0:K})}\right) \\
   =&\frac{1}{2} \sum_{k=1}^K \bigr[ ||\vepsilon_k+2 \sigma \frac{(1-\sqrt{1-\alpha_k})}{\sqrt{\alpha_k}}f_\theta(k,y_{K-k})||^2 - ||\vepsilon_k||^2 \bigr] \nonumber \\
    =&2 \sigma^2 \sum_{k=1}^K \frac{(1-\sqrt{1-\alpha_k})^2}{\alpha_k}||f_\theta(k,y_{K-k})||^2 +2 \sigma \sum_{k=1}^K \frac{(1-\sqrt{1-\alpha_k})}{\sqrt{\alpha_k}}f_\theta(k,y_{K-k})^\top \vepsilon_k. \nonumber
\end{align}
Now (\ref{eq:logDTRN}) follows directly from (\ref{eq:logRNsum}) and (\ref{eq:logDTRN2}). Note we have for $\delta \ll 1$ that $2 \sigma^2 \frac{(1-\sqrt{1-\alpha_k})^2}{\alpha_k} \approx \sigma^2 \beta_k  \delta$ and $2 \sigma \frac{(1-\sqrt{1-\alpha_k})}{\sqrt{\alpha_k}} \approx \sqrt{2\beta_{k\delta} \delta}$ as expected from (\ref{eq:RadonNykodyn}).

Finally the final expression (\ref{eq:KLneat}) of the KL follows now from the fact that $\mathbb{E}_{q^\theta}[f_\theta(k,y_{K-k})^\top \vepsilon_k]=\mathbb{E}_{q^\theta_{k}}[\mathbb{E}_{q^\theta_{k-1|k}}[f_\theta(k,y_{K-k})^\top \vepsilon_k |y_{K-k}]]=\mathbb{E}_{q^\theta_{k}}[f_\theta(k,y_{K-k})^\top \mathbb{E}_{q^\theta_{k-1|k}}[ \vepsilon_k |y_{K-k}]]=0$.

\subsection{Alternative Kullback--Leibler decomposition}\label{app:alternativeKL}
A KL decomposition similar in spirit to the one developed for DDPM \citep{ho2020denoising} can be be derived. It leverages the fact that 
\begin{equation}
p_{k-1|k,0}(x_{k-1}|x_{k},x_0)=\mathcal{N}(x_k;\tilde{\mu}_{k}(x_{k},x_0),\sigma^2 \tilde{\beta}_{k} I)
\end{equation}
for $\tilde{\mu}_{k}(x_{k},x_0)=\frac{\sqrt{\bar{\alpha}_{k-1}}\beta_{k}}{1-\bar{\alpha}_{k}}x_0+\frac{\sqrt{\alpha_{k}}(1-\bar{\alpha}_{k-1})}{1-\bar{\alpha}_{k}}x_{k},~\tilde{\beta}_{k}=\frac{1-\bar{\alpha}_{k-1}}{1-\bar{\alpha}_{k}}\beta_k$.
\begin{proposition}
The reverse Kullback--Leibler discrepancy $\KL(q^{\theta}||p)$ satisfies
\begin{align*}
\KL(q^{\theta}||p)=\mathbb{E}_{q^{\theta}}[&\KL(q_K(x_K)||p_{K|0}(x_K|x_0))
+\sum_{k=2}^{K} \KL(q^{\theta}_{k-1|k}(x_{k-1}|x_k)||p_{k-1|0,k}(x_{k-1}|x_0,x_k))\\
&+   \KL(q^{\theta}_{0|1}(x_0|x_1)||\pi(x_0))].
\end{align*}
So for $q^\theta_{k-1|k}(x_{k-1}|x_k)=\mathcal{N}(x_{k-1};\sqrt{1-\beta_k}x_k+ \sigma^2 \beta_k f_\theta(k,x_k),\sigma^2 \beta_k I)$, the terms $ \KL(q^{\theta}_{k-1|k}(x_{k-1}|x_k)||p_{k-1|0,k}(x_{k-1}|x_0,x_k))$ are KL between two Gaussian distributions and can be calculated analytically.
\end{proposition} 
We found this decomposition to be numerically unstable and prone to diverging in our reverse KL setting. 

\begin{proof}
The reverse KL can be decomposed as follows
\begin{align*}
\KL(q^{\theta}||p)&=\mathbb{E}_{q^{\theta}}\Big[\log \Bigl(\frac{q^{\theta}(x_{0:K})}{p(x_{0:K})}\Bigl)\Big]=\mathbb{E}_{q^{\theta}}\Big[\log \Bigl(\frac{q^{\theta}(x_{0:K})}{\pi(x_0)p(x_{1:K}|x_0)}\Bigl)\Big]\\
&=\mathbb{E}_{q^{\theta}}\Big[\log \Bigl(\frac{q^{\theta}(x_{0:K})}{p(x_{1:K}|x_0)}\Bigl)\Big]-\mathbb{E}_{q^{\theta}}[\log \pi(x_0)]\\
&=L(\theta)-\mathbb{E}_{q^{\theta}}[\log \pi(x_0)]
\end{align*}
where
\begin{align*}
L(\theta)&=\mathbb{E}_{q^{\theta}}\Big[\log \Bigl(\frac{q^{\theta}(x_{0:K})}{p(x_{1:K}|x_0)}\Bigl)\Big]=\mathbb{E}_{q^{\theta}}\Big[\log q_K(x_K) +\sum_{k=1}^K
\log \Bigl(\frac{q^{\theta}_{k-1|k}(x_{k-1}|x_{k})}{p_{k|k-1}(x_{k}|x_{k-1})}\Bigl)\Big]
\end{align*}
Now using the identity for $k \geq 2$
\begin{align*}
    p_{k-1,k|0}(x_{k-1},x_k|x_0)&= p_{k-1|0}(x_{k-1}|x_0) p_{k|k-1}(x_k|x_{k-1})\\
    &=p_{k-1|0,k}(x_{k-1}|x_0,x_k)p_{k|0}(x_k|x_0),
\end{align*}\
we can rewrite $L(\theta)$ as
\begin{align*}
L(\theta)&=\mathbb{E}_{q^{\theta}}\Bigr[\log q_K(x_K) +\sum_{k=2}^{K}
\log \Bigl(\frac{q^{\theta}_{k-1|k}(x_{k-1}|x_{k})}{p_{k-1|0,k}(x_{k-1}|x_0,x_k)}.\frac{p_{k-1|0}(x_{k-1}|x_0)}{p_{k|0}(x_{k}|x_0)}\Bigl)+ \log \Bigl(\frac{q^{\theta}_{0|1}(x_0|x_1)}{p_{1|0}(x_1|x_0)}\Bigl) \Bigr]\\
&=\mathbb{E}_{q^{\theta}}\Big[\log \Bigl(\frac{q_K(x_K)}{p_{K|0}(x_{K}|x_0)}\Bigl) +\sum_{k=2}^{K}
\log  \Bigl(\frac{q^{\theta}_{k-1|k}(x_{k-1}|x_{k})}{p_{k-1|0,k}(x_{k-1}|x_0,x_k)} \Bigl)+ \log q^{\theta}_{0|1}(x_0|x_1) \Big]\\
&=\mathbb{E}_{q^{\theta}}[\KL(q_K(x_K)||p_{K|0}(x_K|x_0))]
+\sum_{k=2}^{K} \KL(q^{\theta}_{k-1|k}(x_{k-1}|x_k)||p_{k-1|0,k}(x_{k-1}|x_0,x_k))\\
&\quad + \log q^{\theta}_{0|1}(x_0|x_1) \Big]
\end{align*}
The result now follows directly. 
\end{proof}

\section{Underdamped Langevin Dynamics}\label{sec:underdampeddetailsSM}
In the generative modeling context, it has been proposed to extend the original state $x \in \mathbb{R}^d$ by a momentum variable $m \in \mathbb{R}^d$. One then diffuses in this extended space using an underdamped Langevin dynamics \citep{Dockhorn2022} targeting $\mathcal{N}(x;0,\sigma^2I)\mathcal{N}(m;0,M)$. It was demonstrated empirically that the resulting scores one needs to estimate are smoother and this led to improved performance. We adapt this approach to Monte Carlo sampling. This adaptation is non trivial and in particular requires to design carefully numerical integrators.
\subsection{Continuous time}
We now consider an augmented target distribution $\pi(x) \mathcal{N}(m;0,M)$ where $M$ is a positive definite mass matrix. We then diffuse this extended target using the following underdamped Langevin dynamics, i.e.
\begin{align}\label{eq:forwardunderdampedSM}
\mathrm{d}x_t&=M^{-1}m_t\mathrm{d}t,\\
\mathrm{d}m_t&=-\frac{x_t}{\sigma^2}\mathrm{d}t -\beta_t m_t\mathrm{d}t +\sqrt{2\beta_t}M^{1/2}\mathrm{d}B_t,\nonumber
\end{align}
where $x_0 \sim \pi, m_0\sim \mathcal{N}(0,M)$. The resulting path measure on $[0,T]$ is denoted again $\mathcal{P}$.
From \citep{haussmann1986time}, the time-reversal process is also a diffusion satisfying
\begin{align}\label{eq:reverseunderdampedSM}
\mathrm{d}y_t&=-M^{-1}n_t\mathrm{d}t,\\
\mathrm{d}n_t&=\frac{y_t}{\sigma^2}\mathrm{d}t + \beta_{T-t} n_t\mathrm{d}t +2\beta_{T-t}M \nabla_{n_t} \log \eta_{T-t}(y_t,n_t)\mathrm{d}t+\sqrt{2\beta_{T-t}}M^{1/2}\mathrm{d}W_t,\nonumber
\end{align}
for $(y_0,n_0) \sim \eta_T$ where $\eta_t$ denotes the density of $(x_t,m_t)$ under (\ref{eq:forwardunderdampedSM}). 

Now consider a reference process $\mathcal{P}^{\textup{ref}}$ on $[0,T]$ defined by the forward process (\ref{eq:forwardunderdampedSM}) initialized using $x_0 \sim \mathcal{N}(0,\sigma^2), m_0\sim \mathcal{N}(0,M)$. In this case one can check that $\eta^{\textup{ref}}_t(x_t,m_t)=\mathcal{N}(x_t;0,\sigma^2I)\mathcal{N}(m_t;0,M)$ and the time-reversal process of $\mathcal{P}^{\textup{ref}}$ satisfies
\begin{align}\label{eq:reverseunderdampedrefSM}
\mathrm{d}y_t&=-M^{-1}n_t\mathrm{d}t,\\
\mathrm{d}n_t&=\frac{y_t}{\sigma^2}\mathrm{d}t + \beta_{T-t} n_t\mathrm{d}t - 2\beta_{T-t} n_t\mathrm{d}t+\sqrt{2\beta_{T-t}}M^{1/2}\mathrm{d}W_t \nonumber\\
&=\frac{y_t}{\sigma^2}\mathrm{d}t -\beta_{T-t} n_t\mathrm{d}t+\sqrt{2\beta_{T-t}}M^{1/2}\mathrm{d}W_t, \nonumber
\end{align}
as $\nabla_n \log \eta_t^{\textup{ref}}(y,n) =\nabla_n \log  (\mathcal{N}(y;0,\sigma^2)\mathcal{N}(n;0,M))=-M^{-1}n$. 

Hence it follows that the time-reversal (\ref{eq:reverseunderdampedSM}) of $\mathcal{P}$ can be also be written as
\begin{align}\label{eq:reverseunderdampedphiSM}
\mathrm{d}y_t&=-M^{-1}n_t\mathrm{d}t,\\
\mathrm{d}n_t&=
\frac{y_t}{\sigma^2}\mathrm{d}t - \beta_{T-t} n_t\mathrm{d}t +2\beta_{T-t}M \nabla_{n_t} \log \phi_{T-t}(y_t,n_t)\mathrm{d}t+\sqrt{2\beta_{T-t}}M^{1/2}\mathrm{d}W_t,  \nonumber
\end{align}
where $\phi_t(x,m):=\eta_t(x,m)/\eta_t^{\textup{ref}}(x,m)$. 

To approximate $\mathcal{P}$, we consider a parameterized path measure $\mathcal{Q}^\theta$ whose time reversal is defined for $(y_0,n_0)\sim \mathcal{N}(y_0;0,I)\mathcal{N}(n_0;0,M)$ by $\mathrm{d}y_t=-M^{-1}n_t\mathrm{d}t$ and 
\begin{align}\label{eq:reverseunderdampedQSM}
\mathrm{d}n_t
=\frac{y_t}{\sigma^2}\mathrm{d}t - \beta_{T-t} n_t\mathrm{d}t +2\beta_{T-t}M f_{\theta}(T-t,y_t,n_t)\mathrm{d}t+\sqrt{2\beta_{T-t}}M^{1/2}\mathrm{d}W_t.
\end{align}

\subsection{Learning the time-reversal through KL minimization \label{appdx:udmpcont}}
To approximate $\mathcal{P}$, we will consider a parameterized diffusion whose time reversal is defined by
\begin{align}\label{eq:reverseunderdampedQSM}
\mathrm{d}y_t&=-M^{-1}n_t\mathrm{d}t,\\
\mathrm{d}n_t&
=\frac{y_t}{\sigma^2}\mathrm{d}t - \beta_{T-t} n_t\mathrm{d}t +2\beta_{T-t}M f_{\theta}(T-t,y_t,n_t)\mathrm{d}t+\sqrt{2\beta_{T-t}}M^{1/2}\mathrm{d}W_t \nonumber
\end{align}
for $(y_0,n_0)\sim \mathcal{N}(y_0;0,I)\mathcal{N}(n_0;0,M)$
inducing a path measure $\mathcal{Q}^\theta$ on the time interval $[0,T]$.

We can now compute the Radon-Nikodym derivative between $\mathcal{Q}^\theta$ and $\mathcal{P}^{\textup{ref}}$ using an extension of Girsanov theorem  (Theorem A.3. in  \cite{sottinen2008application})
\begin{align}\nonumber
    \log \frac{\mathrm{d}\mathcal{Q}^{\theta}}{\mathrm{d}\mathcal{P}^{\textup{ref}}}&=\int^T_0 \sqrt{2\beta_{T-t}}M^{1/2} f_{\theta}(T-t,y_t,n_t)^\top \mathrm{d}W_t+\frac{1}{2}\int^T_0 ||2\beta_{T-t}M  f_{\theta}(T-t,y_t,n_t)||^2_{(2\beta_{T-t}M)^{-1}} \mathrm{d}t\\
\end{align}

To summarize, we have the following proposition.

\begin{proposition}\label{prop:RNcontinuouunderdamped}The Radon-Nikodym derivative $\frac{\mathrm{d}\mathcal{Q}^{\theta}}{\mathrm{d}\mathcal{P}^{\textup{ref}}}(y_{[0,T]},n_{[0,T]})$ satisfies under $\mathcal{Q}^{\theta}$
\begin{equation}
  \log \left(\frac{\mathrm{d}\mathcal{Q}^{\theta}}{\mathrm{d}\mathcal{P}^{\textup{ref}}}\right)
    =\scaleobj{.8}{\int^T_0} \beta_{T-t}|| f_{\theta}(T-t,y_t,n_t)||^2_{M} \mathrm{d}t+\scaleobj{.8}{\int^T_0} \sqrt{2\beta_{T-t}}M^{1/2} f_{\theta}(T-t,y_t,n_t)^\top \mathrm{d}W_t \label{eq:RNunderdamped}.
\end{equation}
From $\KL(\mathcal{Q}^{\theta}||\mathcal{P})=\KL(\mathcal{Q}^{\theta}||\mathcal{P}^{\textup{ref}})+\mathbb{E}_{Q^{\theta}}\Bigr[\log \left(\frac{p^{\textup{ref}}_0(y_T,n_T)}{\eta_0(y_T,n_T)}\right)\Bigr]$, it follows that
\begin{align}
\KL(\mathcal{Q}^{\theta}||\mathcal{P})
 &=\mathbb{E}_{\mathcal{Q}^{\theta}}\Bigr[\scaleobj{.8}{\int^T_0} \beta_{T-t}|| f_{\theta}(T-t,y_t,n_t)||^2_{M} \mathrm{d}t +\scaleobj{.8}{\ln \left(\frac{\gN(y_T; 0, \sigma^2 I)}{\pi(y_T)}\right)}\Bigr]\label{eq:KLpathintegralunderdamped}.
\end{align}
\end{proposition}
The second term on the r.h.s. of (\ref{eq:KLpathintegralunderdamped}) follows from the fact that $\log (p^{\textup{ref}}_0(y_T,n_T)/p_0(y_T,n_T))=\log (\mathcal{N}(y_T;0,\sigma^2I)/\pi(y_T))$.

\subsection{Normalizing flow through ordinary differential equation \label{app:NFunderdamped}}
The following ODE gives exactly the same marginals $\eta_T$ as the SDE (\ref{eq:forwardunderdampedSM}) defining $\mathcal{P}$
\begin{align}\label{eq:ODEunderdampedSM}
    \mathrm{d}x_t&=M^{-1}m_t\mathrm{d}t,\\
    \mathrm{d}m_t&=-\frac{x_t}{\sigma^2}\mathrm{d}t -\beta_t m_t\mathrm{d}t -\beta_t M \nabla_{m_t} \log \eta_{t}(x_t,m_t) \mathrm{d}t,\nonumber\\
     &=-\frac{x_t}{\sigma^2}\mathrm{d}t-\beta_t M \nabla_{m_t} \log \phi_{t}(x_t,m_t) \mathrm{d}t, \nonumber
\end{align}
Thus if we integrate its time reversal from $0$ to $T$ starting from $(y_0,q_0)\sim \eta_T$
\begin{align}\label{eq:reversalODEunderdampedSM}
    \mathrm{d}y_t&=-M^{-1}n_t\mathrm{d}t,\\
    \mathrm{d}n_t&= \frac{y_t}{\sigma^2}\mathrm{d}t+\beta_{T-t} M \nabla_{n_t} \log \phi_{T-t}(y_t,n_t) \mathrm{d}t,\nonumber
\end{align}
then we would obtain at time $T$ a sample $(y_T,n_T)\sim  \pi(y_T)\mathcal{N}(n_T;0,M)$. 

In practice, this suggests that once we have learned an approximation $f_{\theta^*}(t,x,m) \approx \nabla_m \log \phi_t(x,m)$ by minimization of the ELBO then we can construct a proposal using 
\begin{align}\label{eq:reversalODEunderdampedproposalSM}
    \mathrm{d}y_t&=-M^{-1}n_t\mathrm{d}t,\\
    \mathrm{d}n_t&= \frac{y_t}{\sigma^2}\mathrm{d}t+\beta_{T-t} M f_{\theta^*}(T-t,y_t,n_t) \mathrm{d}t,\nonumber
\end{align}
for $(y_0,n_0)\sim \mathcal{N}(y_0;0,\sigma^2I)\mathcal{N}(n_0;0,M)$. The resulting sample $(y_T,n_T) \sim \bar{\eta}_0$ will have distribution close to $\pi \times \mathcal{N}(0,M)$. Again it is possible to compute pointwise the distribution of this sample to perform an importance sampling correction.

\subsection{From continuous time to discrete time}\label{subsec:integratorunderdamped}
We now need to come up with discrete-time integrator for the time-reversal of $\mathcal{P}^{\textup{ref}}$ given by (\ref{eq:reverseunderdampedrefSM}) and the time-reversal of $\mathcal{Q}^\theta$ given by (\ref{eq:reverseunderdampedQSM}).
Let us start with  (\ref{eq:reverseunderdampedrefSM}). We split it into the two components
\begin{equation}\label{eq:reverseunderdampedrefP1SM}
\mathrm{d}y_t=-M^{-1}n_t\mathrm{d}t, \qquad \mathrm{d}n_t=\frac{y_t}{\sigma^2}\mathrm{d}t,
\end{equation}
and 
\begin{equation}\label{eq:reverseunderdampedrefP2SM}
\mathrm{d}y_t=0,\qquad \mathrm{d}n_t=-\beta_{T-t} n_t\mathrm{d}t+\sqrt{2\beta_{T-t}}M^{1/2}\mathrm{d}W_t.
\end{equation}
We will compose these transitions. To obtain $(y_{k+1},n_{k+1})$ from $(y_k,n_k)$, we first integrate (\ref{eq:reverseunderdampedrefP1SM}). To do this, consider the Hamiltonian equation $\mathrm{d}y_t=M^{-1}n_t\mathrm{d}t, \mathrm{d}n_t=-\frac{y_t}{\sigma^2}\mathrm{d}t$ which preserves  $\mathcal{N}(y;0,\sigma^2I)\mathcal{N}(n;0,M)$ as invariant distribution. We can integrate this ODE exactly over an interval of length $\delta$ and denote its solution $\Phi(y,n)$; see Section \ref{sec:exactHamiltonianequation} for details. We use its inverse $\Phi^{-1}(y,n)=\Phi_{\textup{flip}}(y,n)\circ\Phi(y,n)\circ\Phi_{\textup{flip}}(y,n)$ where $\Phi_{\textup{flip}}(y,n)=(y,-n)$ so that $(y_{k+1},n'_k)=\Phi^{-1}(y_k,n_k)$. Then we integrate exactly (\ref{eq:reverseunderdampedrefP2SM}) using
\begin{equation}\label{eq:integratorOUmomentumSM}
   \qquad n_{k+1}=\sqrt{1-\alpha_{K-k}}n'_k+\sqrt{\alpha_{k-1}} M^{1/2} \vepsilon_k.
\end{equation}
We have thus design a transition of the form 
\begin{equation}\label{eq:transprefunderdamped}
    p^{\textup{ref}}(y_{k+1},n_{k+1},n'_k|y_k,n_k)=\delta_{\Phi^{-1}(y_k,n_k)}(y_{k+1},n'_k)\mathcal{N}(n_{k+1};\sqrt{1-\alpha_{K-k}}n'_k; \alpha_{K-k}M).
\end{equation}
Now to integrate (\ref{eq:reverseunderdampedQSM}), we split it in three parts. We first integrate (\ref{eq:reverseunderdampedrefP1SM}) using exactly the same integrator $(y_{k+1},n'_k)=\Phi^{-1}(y_k,n_k)$. Then we integrate 
\begin{equation}
\mathrm{d}n_t=2\beta_{T-t}M f_{\theta}(T-t,y_t,n_t)\mathrm{d}t
\end{equation}
using 
\begin{equation}
    n''_k=n'_k+2 (1-\sqrt{1-\alpha_{K-k}}) M f_{\theta}(K-k,y_{k+1},n'_k),
\end{equation}
where we abuse notation and write $f_{\theta}(K-k,y_{k+1},n'_k)$ instead of $f_{\theta}((K-k)\delta,y_{k+1},n'_k)$.
Finally, we integrate the OU part using (\ref{eq:integratorOUmomentumSM}) but replacing $n'_k$ by with $n''_k$. So the final transition we get is 
\begin{align*}
    q^\theta(y_{k+1},n_{k+1},n'_k,n''_k|y_k,n_k)&=\delta_{\Phi^{-1}(y_k,n_k)}(y_{k+1},n'_k)\delta_{n'_k+2 (1-\sqrt{1-\alpha_{K-k}})M f_{\theta}(K-k,y_{k+1},n'_k)})(n''_k)\\
    &\times \mathcal{N}(n_{k+1};\sqrt{1-\alpha_{K-k}}n''_k;\alpha_{K-k}M).
\end{align*}
However, we can integrate $n''_k$ analytically to obtain
\begin{align}\label{eq:transqthetaunderdamped}
    &q^\theta(y_{k+1},n_{k+1},n'_k|y_k,n_k)=\delta_{\Phi^{-1}(y_k,n_k)}(y_{k+1},n'_k) \nonumber\\
    \times&\mathcal{N}(n_{k+1};\sqrt{1-\alpha_{K-k}}(n'_k+2(1-\sqrt{1-\alpha_{K-k}})M f_{\theta}(K-k,y_{k+1},n'_k));\alpha_{K-k}M).
\end{align}

\subsection{Exact solution of Hamiltonian equation}\label{sec:exactHamiltonianequation}
Consider the Hamiltonian equation and $M=I$, the solution of
\begin{equation}
\mathrm{d}x_t=m_t\mathrm{d}t,\qquad \mathrm{d}m_t=-\frac{x_t}{\sigma^2}\mathrm{d}t
\end{equation}
can be exactly written as $(x_t,m_t)=\Phi_t(x_0,m_0)$ defined through
\begin{align}
    x_t=x_0 \cos(t/\sigma)+m_0 \sigma \sin(t/\sigma),\quad 
    m_t=-\frac{x_0}{\sigma} \sin(t/\sigma)+m_0 \cos(t/\sigma).
\end{align}
This is the so-called harmonic oscillator.

Now the inverse of the Hamiltonian flow satisfies $\Phi^{-1}(x,m)=\Phi_{\textup{flip}}(x,m)\circ\Phi(x,m)\circ\Phi_{\textup{flip}}(x,m)$ (see e.g. \cite{leimkuhler2016molecular})
so we can simply integrate
\begin{equation}
\mathrm{d}y_t=-n_t\mathrm{d}t,\qquad \mathrm{d}n_t=\frac{y_t}{\sigma^2}\mathrm{d}t,
\end{equation}
using $(y_{k+1},n'_{k})=\Phi_{\tau}^{-1}(y_k,n_k)$.

This gives
\begin{align}
    y_{k+1}&=y_k \cos(\tau/\sigma)-n_k \sigma \sin(\tau/\sigma)\\
    n'_k&=-(-\frac{y_k}{\sigma} \sin(\tau/\sigma)-n_k \cos(t/\sigma))\nonumber\\
    &=\frac{y_k}{\sigma} \sin(\tau/\sigma)+n_k \cos(t/\sigma).
\end{align}

Obviously, if we have $M=\alpha I$ then
\begin{equation}
\mathrm{d}x_t=\frac{m_t}{\alpha}\mathrm{d}t,\qquad \mathrm{d}m_t=-\frac{x_t}{\sigma^2}\mathrm{d}t,
\end{equation}
then we use a reparameterization $\tilde{m}_t=m_t/\alpha$ and
\begin{equation}
\mathrm{d}x_t=\tilde{m}_t \mathrm{d}t,\qquad \mathrm{d}\tilde{m}_t=-\frac{1}{\alpha}\frac{x_t}{\sigma^2}\mathrm{d}t,
\end{equation}
and the solution is as above with $\sigma$ being replaced by $\sigma \sqrt{\alpha}$.

So for $\tilde{n}_k=n_k/\alpha$, we have 
\begin{align*}
    y_{k+1}&=y_k \cos(\tau/(\sqrt{\alpha}\sigma))- \tilde{n}_k \sigma\sqrt{\alpha} \sin(\tau/(\sqrt{\alpha}\sigma)),\\
    \tilde{n}'_k&=\frac{y_k}{\sqrt{\alpha}\sigma} \sin(\tau/(\sqrt{\alpha}\sigma))+\tilde{n}_k \cos(t/(\sqrt{\alpha}\sigma)),
\end{align*}
so as $n_k=\alpha\tilde{n}_k$ and similarly $n'_k=\alpha\tilde{n}'_k$, we finally obtain
\begin{align}
   y_{k+1}&=y_k \cos(\tau/(\sqrt{\alpha}\sigma))- n_k \frac{\sigma}{\sqrt{\alpha}} \sin(\tau/(\sqrt{\alpha}\sigma)),\\
   n'_k&=\frac{\sqrt{\alpha}}{\sigma}y_k \sin(\tau/(\sqrt{\alpha}\sigma))+n_k \cos(\tau/(\sqrt{\alpha}\sigma)).
\end{align}

\subsection{Derivation of Underdampened KL in Discrete Time \label{apdx:KLunderdampedDT}}

In this section we derive the discrete-time KL for the underdamped noising dynamics.
\begin{proposition}\label{prop:KLunderdampedDT}
The log density ratio $\textup{lr}=\log (q^{\theta}(y_{0:K},n_{0:K},n'_{0:K})/ p^{\textup{ref}}(y_{0:K},n_{0:K},n'_{0:K}))$ for $y_{0:K},n_{0:K},n'_{0:K}\sim q^{\theta}(\cdot)$ equals
\begin{equation}\label{eq:logDTRN2}
\textup{lr}=\sum_{k=1}^K 2\Bigr[\frac{\kappa^2_k}{\alpha_k}||f_{\theta}(k,y_{K-k+1},n'_{K-k})||^2_{M}+M^{1/2}\frac{\kappa_k}{\sqrt{\alpha_k}}f_{\theta}(k,y_{K-k+1},n'_{K-k})^\top \varepsilon_k  \Bigr]
\end{equation}
where $\kappa_k:=\sqrt{1-\alpha_k}(1-\sqrt{1-\alpha_k})$ and $\varepsilon_k$ is obtained as a function of $n_{k+1}$ and $n''_k$ when describing the integrator for $\mathcal{Q}^{\theta}$. In particular, we have $\varepsilon_k \overset{\textup{i.i.d.}}{\sim} \mathcal{N}(0,I)$ and one obtains
\begin{align}
    \KL(q^{\theta}||p)=2\E_{q^\theta}\left[\sum_{k=1}^K \frac{\kappa^2_k}{\alpha_k}\Bigr[||f_{\theta}(k,y_{K-k+1},n'_{K-k})||^2_{M} + \scaleobj{.8}{\ln \left(\frac{\gN(y_K; 0, \sigma^2 I)}{\pi(y_K)}\right)}\right].\label{eq:KLneatunderdamped}
\end{align}
\end{proposition}

The integrator has been designed so that the ratio between the transitions of the proposal (\ref{eq:transqthetaunderdamped}) and the reference process (\ref{eq:transprefunderdamped}) is well-defined as the deterministic parts are identical in the two transitions so cancel and 
\begin{align}\label{eq:discreteRNunderdampedSM}
    &\frac{q^\theta(y_{k+1},n_{k+1},n'_k|y_k,n_k)}{p^{\textup{ref}}(y_{k+1},n_{k+1},n'_k|y_k,n_k)}\\
    =&\frac{\mathcal{N}(n_{k+1};\sqrt{1-\alpha_{K-k}}(n'_k+2(1-\sqrt{1-\alpha_{K-k}})M f_{\theta}(K-k,y_{k+1},n'_k)); \alpha_{K-k}M)}{\mathcal{N}(n_{k+1};\sqrt{1-\alpha_{K-k}}n'_k; \alpha_{K-k}M)}.\nonumber
\end{align}

The calculations to compute the Radon--Nikodym derivative are now very similar to what we did in the proof of Proposition \ref{prop:KL}.

We have 
\begin{align}
    &\log  \left(\frac{q^{\theta}(y_{0:K},n_{0:K},n'_{0:K})}{p^{\textup{ref}}(y_{0:K},n_{0:K},n'_{0:K})}\right)  \label{eq:rationvareps1SM} \\
    =&\log \left( \frac{q_{K}(y_0,n_0)}{p^{\textup{ref}}_{K}(y_0,n_0)} \right)+\sum_{k=1}^K 
    \log \left( \frac{q^{\theta}_{k-1|k}(y_{K-k+1},n_{K-k+1},n'_{K-k}|y_{K-k},n_{K-k})}{p^{\textup{ref}}_{k-1|k}(y_{K-k+1},n_{K-k+1},n'_{K-k}|y_{K-k},n_{K-k})} \right) \nonumber \\
    =& \sum_{k=1}^K \Bigr[ ||n_{K-k+1}-\sqrt{1-\alpha_k}n'_{K-k}||^2_{(2\alpha_k M)^{-1}} \nonumber \\
    &- ||n_{K-k+1}-\sqrt{1-\alpha_{k}}(n'_{K-k}+2(1-\sqrt{1-\alpha_{k}})M f_{\theta}(k,y_{K-k+1},n'_{K-k}))||^2_{(2\alpha_k M)^{-1}} \Bigr] \nonumber
\end{align}
where we have exploited the fact that $q_{K}(y_0,n_0)=p^{\textup{ref}}_{K}(y_0,n_0)=\mathcal{N}(y_0;0,\sigma^2 I)\mathcal{N}(n_0;0,M)$.

Now let us introduce 
\begin{equation}
    \vepsilon_k:=\frac{M^{-1/2}}{\sqrt{\alpha_k}}(n_{K-k+1}-\sqrt{1-\alpha_k}n'_{K-k}-2\sqrt{1-\alpha_k}(1-\sqrt{1-\alpha_k})Mf_{\theta}(k,y_{K-k},n'_{K-k}))
\end{equation}
hence 
\begin{align*}
     &\frac{M^{-1/2}}{\sqrt{\alpha_k}}(n_{K-k+1}-\sqrt{1-\alpha_k}n'_{K-k})\\
     =& \vepsilon_k+ \frac{2M^{-1/2}}{\sqrt{\alpha_k}}\sqrt{1-\alpha_k}(1-\sqrt{1-\alpha_k})Mf_{\theta}(k,y_{K-k+1},n'_{K-k})\\
     =&\vepsilon_k+ \frac{2M^{1/2}}{\sqrt{\alpha_k}}\sqrt{1-\alpha_k}(1-\sqrt{1-\alpha_k})f_{\theta}(k,y_{K-k+1},n'_{K-k}).
\end{align*}
We can rewrite (\ref{eq:rationvareps1SM}) as 
\begin{align}\label{eq:logDTRN2underdampedSM}
    &\log  \left(\frac{q^{\theta}(y_{0:K},n_{0:K},n'_{0:K})}{p^{\textup{ref}}(y_{0:K},n_{0:K},n'_{0:K})}\right)   \\
    =& \frac{1}{2} \sum_{k=1}^K \Bigr[||\vepsilon_k+ \frac{2M^{1/2}}{\sqrt{\alpha_k}}\sqrt{1-\alpha_k}(1-\sqrt{1-\alpha_k})f_{\theta}(k,y_{K-k+1},n'_{K-k}))||^2-||\vepsilon_k||^2  \Bigr] \nonumber\\
    =& \frac{1}{2} \sum_{k=1}^K \Bigr[||f_{\theta}(k,y_{K-k+1},n'_{K-k}))||^2_{4M\frac{(1-\alpha_k)(1-\sqrt{1-\alpha_k})^2}{\alpha_k}}  \nonumber\\
    &+\frac{4M^{1/2}}{\sqrt{\alpha_k}}\sqrt{1-\alpha_k}(1-\sqrt{1-\alpha_k})f_{\theta}(k,y_{K-k+1},n'_{K-k}))^\top \varepsilon_k  \Bigr] \nonumber\\
    =& \sum_{k=1}^K \Bigr[||f_{\theta}(k,y_{K-k+1},n'_{K-k}))||^2_{2M\frac{(1-\alpha_k)(1-\sqrt{1-\alpha_k})^2}{\alpha_k}} \nonumber\\
    &+2M^{1/2}\frac{\sqrt{1-\alpha_k}(1-\sqrt{1-\alpha_k})}{\sqrt{\alpha_k}}f_{\theta}(k,y_{K-k+1},n'_{K-k}))^\top \varepsilon_k  \Bigr] \nonumber
\end{align}
where we note that $\frac{\sqrt{1-\alpha_k}(1-\sqrt{1-\alpha_k})}{\sqrt{\alpha_k}}\approx \alpha_k/2 \approx \beta_{k\delta}\delta$.

\section{Experiments}

In this section we incorporate additional experiments and ablations as well as further experimental detail. For both PIS and DDS we created a grid with $\delta = \frac{K}{T}= 0.05$ and values of $T \in \{3.4, 6.4, 12.8, 25.6\}$ and corresponding number of steps $K \in \{64, 128, 128, 256\}$. When tuning PIS we explore different values of $\sigma$ which in practice, when discretised, amounts to the same effect as changing $\delta$. For DDS we apply the cosine schedule directly on the number of steps however we ensure that $\sum_k \alpha_k = \alpha_{\max} T$  such that we have a similar scaling as in PIS. Finally both PIS and DDS where trained with Adam \citep{Kingma:2014} to at most 11000 iterations, although in most cases most both converged in less than 6000 iterations. Across all experiments used a sample size of 300 to estimate the ELBO at train time and of 2000 to for the reported IS estimator of $Z$ for all methods.

Finally for the ground truth where available we use the true normalizing constant as is the case with the Funnel distribution and the Normalising Flows used in the NICE \cite{dinh2014nice} task. For the rest of the tasks we follow prior work \citep{zhangyongxinchen2021path,arbel2021annealed} use a long run SMC chain (1000 temperatures, 30 seeds with 2000 samples each).

\subsection{Network Parametrisation}\label{apdx:nnparam}

In order to improve numerical stability we re-scale the neural network drift parametrisation as follows:
\begin{align*}
    f_{\theta}(k, y) &= 2^{-1}\lambda_K^{-1} \alpha_k  \tilde{f}_\theta(k, y), \\
    \tilde{f}_\theta(k, y) &=\mathrm{NN}_1(k, y;\theta) + \mathrm{NN}_2(k;\theta) \odot \nabla \ln  \pi(y).
\end{align*}
This was done such that the reciprocal term  $\alpha_{K-k}^{-1}$ in the DDS objective is cancelled as the term reaches very small values in the boundaries causing the overall objective to be large. Finally as $\lambda_k = 1- \sqrt{1-\alpha_k} \approx \frac{\alpha_k}{2}$ it follows that our proposed re-parametrisation converges to the same SDE as before, but now with stabler and simpler updates:
\begin{align}
 y_{k+1,n}&\!=\!\sqrt{1\!\!-\!\alpha_{K\!-\!k}} y_{k,n}+\sigma^2\alpha_{K-k}\tilde{f}_\theta(K\!\!-k,y_{k,n}\!) +\sigma \!\sqrt{\alpha_{K-k}} \varepsilon_{k,n},\; \varepsilon_{k,n}\! \overset{\textup{i.i.d.}}{\sim}\! \mathcal{N}(0,I),\\
r_{k+1,n} &= r_{k,n} + {2^{-1}\sigma^2 \alpha_{K-k}} || \tilde{f}_\theta(K-k,y_{k,n})||^2.
\end{align}

\subsection{Tuning Hyper-parameters}\label{apdx:tune}

For both DDS and PIS we explore a grid of 25 hyper-parameter values. For PIS we search for the best performing value of the volatility coefficient over $25$ values, depending on the task we vary the end points of the drift however we noticed that PIS leads to numerical instabilities for $\gamma > 4$ thus we never search for values larger than this. For DDS we searched across $5$ values for $\sigma$ and $5$ values for $\alpha_{\max}$, this led to a total of $25$ combinations we explored for each experiment. Finally for SMC we searched over 3 of its different step sizes  leading to a total of $5^3=125$.

\subsection{Brownian Motion Model With Gaussian Observation Noise} \label{apdx:brownian}

The model is given by
\begin{align*}
    \alpha_{\mathrm{inn}} &\sim \mathrm{LogNormal}(0,2),\\
    \alpha_{\mathrm{obs}} &\sim \mathrm{LogNormal}(0,2) ,\\
    x_1 &\sim \gN(0, \alpha_{\mathrm{inn}}),\\
    x_i &\sim \gN(x_{i-1}, \alpha_{\mathrm{inn}}), \quad i=2,\hdots 20,\\
    y_i &\sim \gN(x_{i}, \alpha_{\mathrm{obs}}), \quad i=1,\hdots 30.
\end{align*}

The goal is to perform inference over the variables $\alpha_{\mathrm{inn}}, \alpha_{\mathrm{obs}}$ and $\{x_i\}_{i=1}^{30}$ given the observations $\{y_i\}_{i=1}^{10} \cup \{y_i\}_{i=20}^{30}$.

\subsection{Drift Magnitude}
\begin{figure*}[t!]
\centering
    \includegraphics[width=0.6\linewidth ]{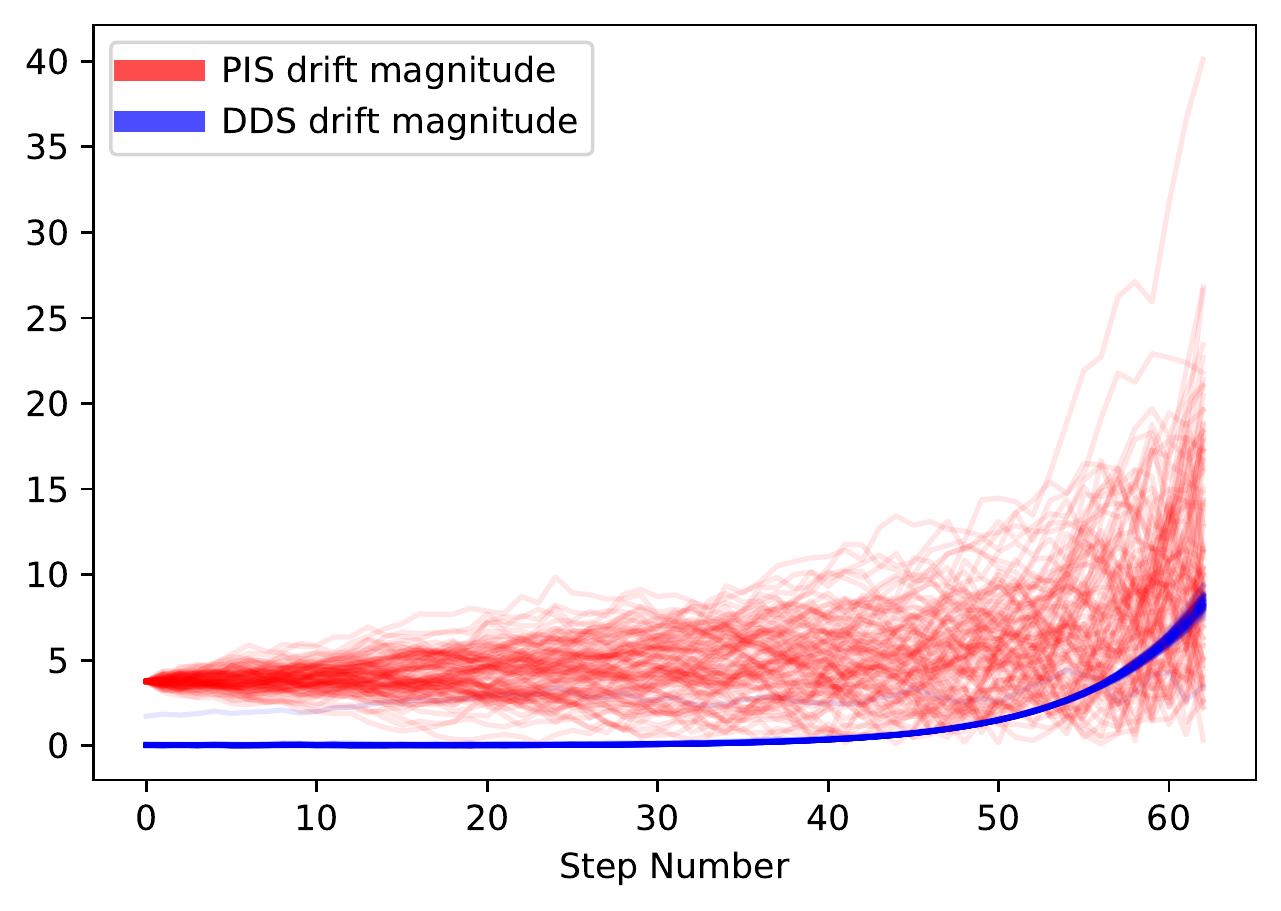}
    \caption{Magnitude of the learnt neural net approximation of the drift $\nabla_x\ln \phi_{t}(x)$ (see (\ref{eq:drifteq})) as a function of $t$.} \label{fig:driftmag}
\end{figure*}

In Figure \ref{fig:driftmag} we compare the magnitudes of the NN drifts learned by PIS and DDS on a scalar target $\gN(6, \sigma^2)$ target. The drifts where sampled on randomly evaluated trajectories from each of the learned samplers. We observe the PIS drift to have large magnitudes with higher variance more prone to numerical instabilities, whilst the OU drift has a notably lower magnitude with significantly less variance. We conjecture this is the reason that PIS becomes unstable at training time for a large number of steps across several of our PIS experiments, many of which did not converge due numerical instabilities. For simplicity in this experiment we set $\alpha_k$ to be uniform.

\subsection{Training Time} \label{sec:training_time}

We evaluate the training times of DDS and PIS on an 8-chip TPU circuit and average over 99 runs. Table \ref{fig:training_time} shows running times per training iteration. In total we trained for 11000 iterations thus the total training times are  19 minutes; 37 minutes; 1 hour and 14 minutes; 2 hours and 26 minutes for $k=64,128,256,512$ respectively.
\begin{table}[]
\centering
\begin{tabular}{@{}ccccc@{}}
\toprule
    & \multicolumn{4}{c}{Training Times per Iteration - seconds (LGCP)}                                \\ \cmidrule(l){1-5} 
    & $k=64$             & $k=128$            & $k=256$            & $k=512$                           \\ \cmidrule(l){2-5} 
DDS & $0.104 \pm 0.0004$ & $0.205 \pm 0.0004$ & $0.406 \pm 0.0004$ & $0.809 \pm 0.0004$ \\
PIS & $0.104 \pm 0.0004$ & $0.205 \pm 0.0003$ & $0.406 \pm 0.0005$ & $0.808 \pm 0.0006$                \\ \bottomrule
\end{tabular} 
\caption{Training time per ELBO gradient update on 300 samples. \label{fig:training_time}}
\end{table}

\subsection{Variational Encoder in Latent Space} \label{sec:vae_exp}

\begin{figure*}[t!]
    \centering
    \includegraphics[width=0.46\linewidth ]{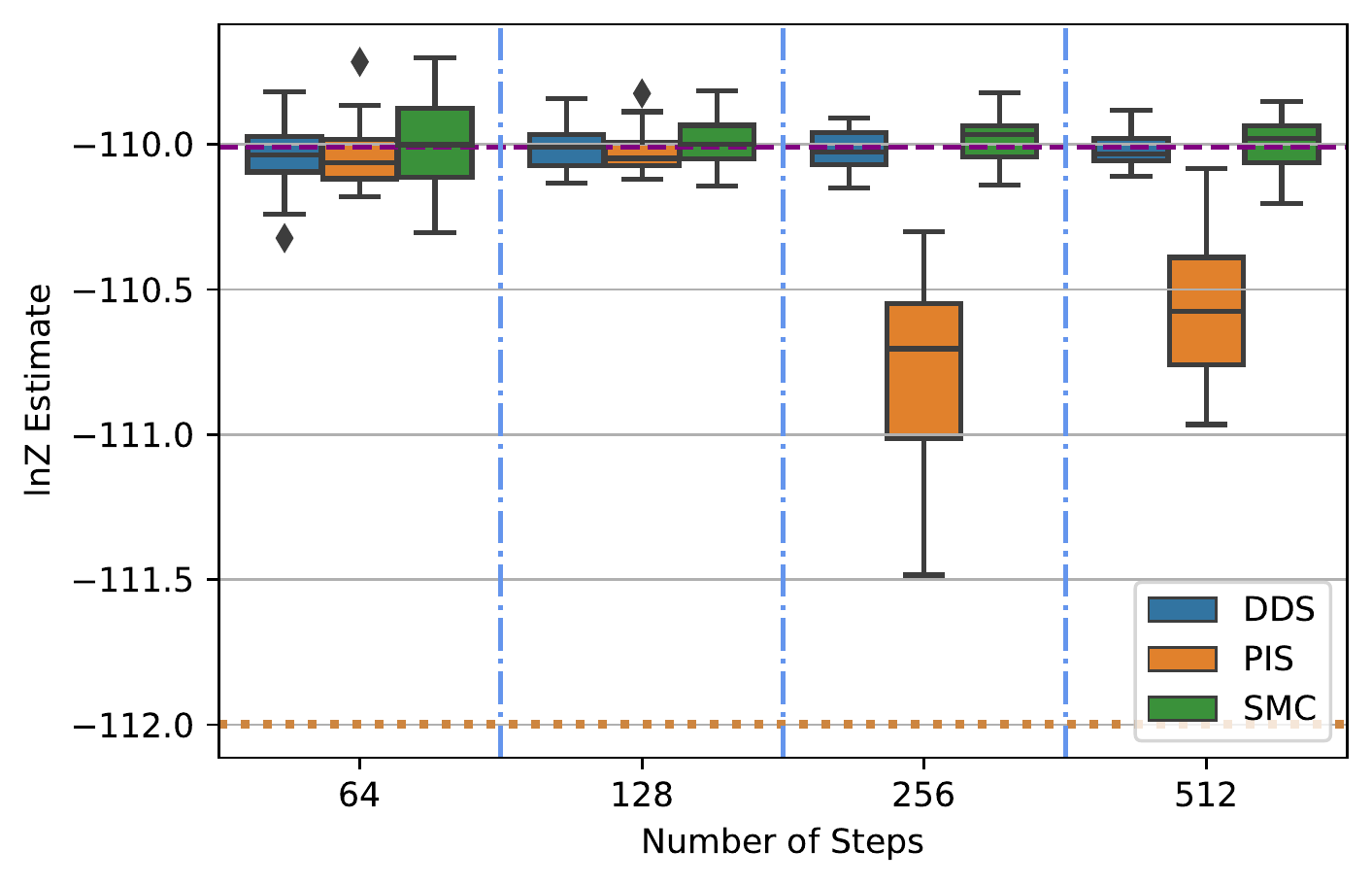}
    \caption{Results on pretrained VAE from \cite{arbel2021annealed}.}
 \label{fig:vae}
\end{figure*}

Following \cite{arbel2021annealed} we explore estimating the posterior of a pre-trained Variational Auto Encoder (VAE) \citep{kingma2013auto}. We found all approaches to reach a very small error for $\ln Z$ very quickly however as seen with PIS across experiments results became unstable for large $T$ and thus many hyper-parameters when tuning the volatility for PIS led to the loss exploding.

\subsection{Generated Images}\label{app:mnist}

In this section we provide some of the generated samples for the normalising flow evaluation task. In Figure \ref{fig:mnist_img} we can observe how both PIS and DDS are able to generate more image classes than SMC due to mode collapse. Out of the 3 approaches we can see that DDS mode collapses the least whilst SMC generates the higher quality images. Using a neural network with inductive biases such as a convolutional neural network \citep{lecun1998gradient} should improve the image quality over the simple feed-forward networks we have used, we believe this is the reason behind the low quality in the images as both minimising the path KL and sampling high quality images is a difficult task for a small feed-forward network.

\begin{figure*}[t!]
    \centering
    \begin{minipage}{0.32\linewidth }
    \centering
    \includegraphics[width=\linewidth ]{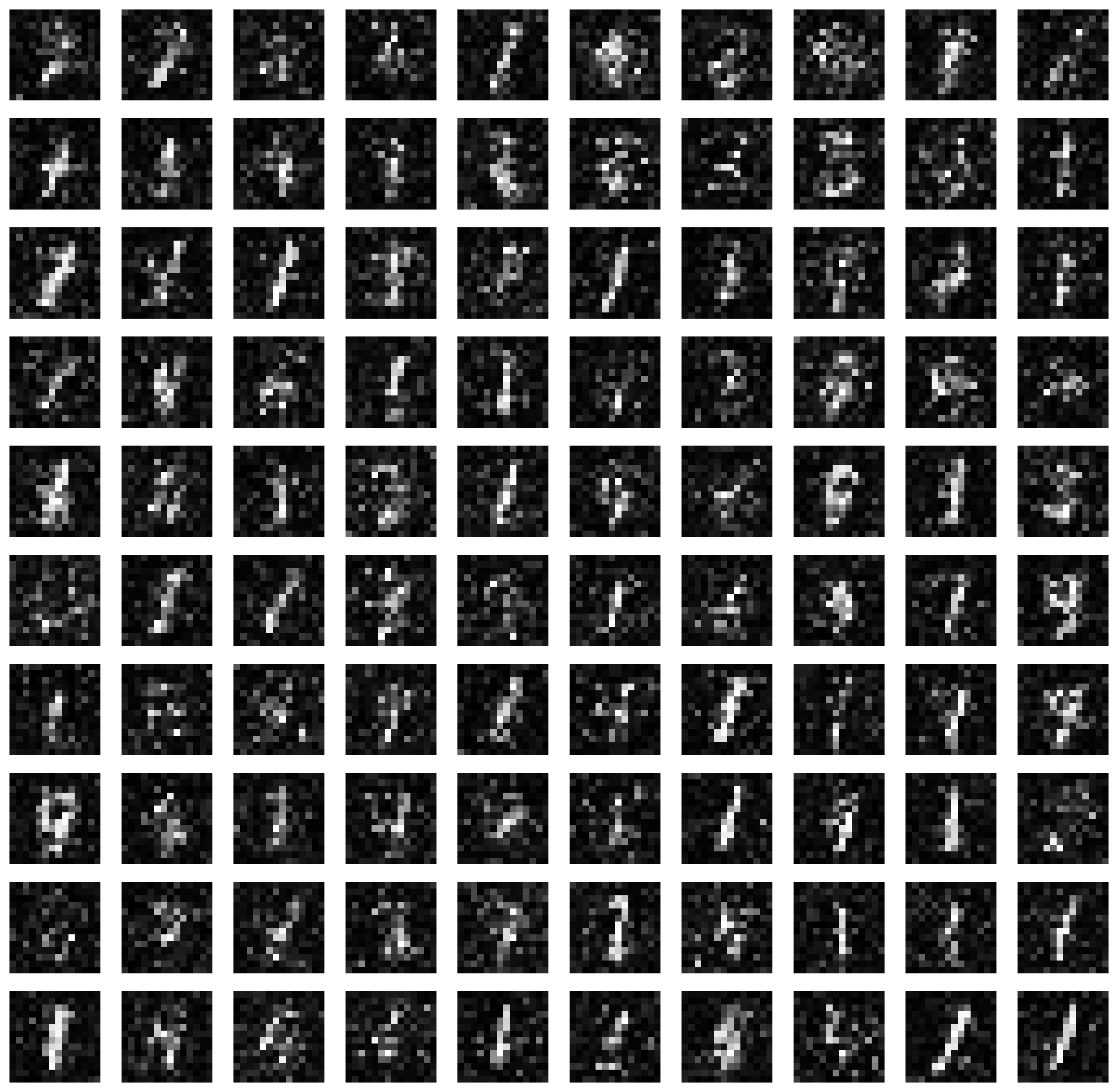}
    \end{minipage}
    \hspace*{\fill}\begin{minipage}{0.32\linewidth }
    \centering
    \includegraphics[width=\linewidth ]{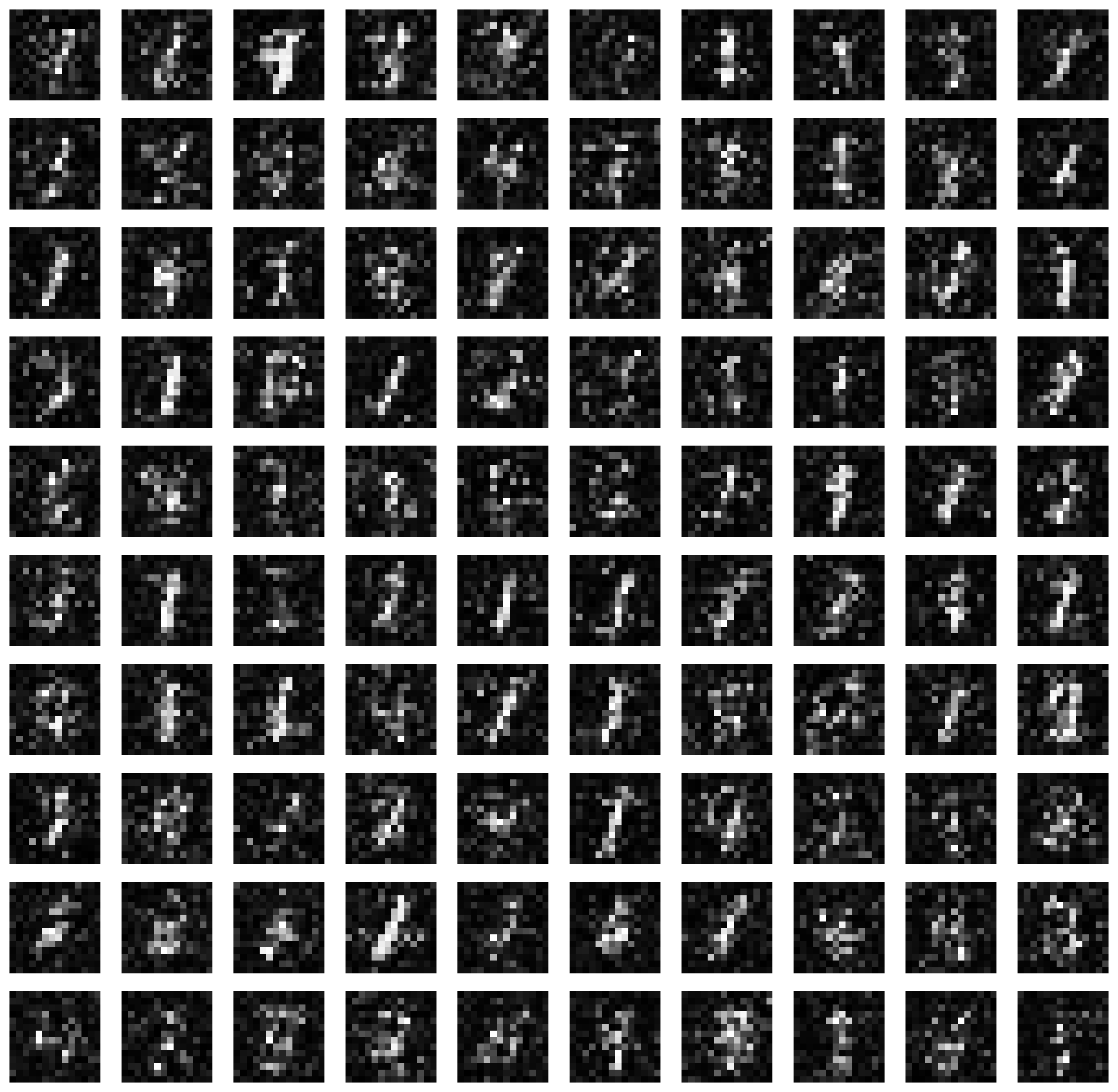}
    \end{minipage}
    \hspace*{\fill}\begin{minipage}{0.32\linewidth }
    \centering
    \includegraphics[width=\linewidth ]{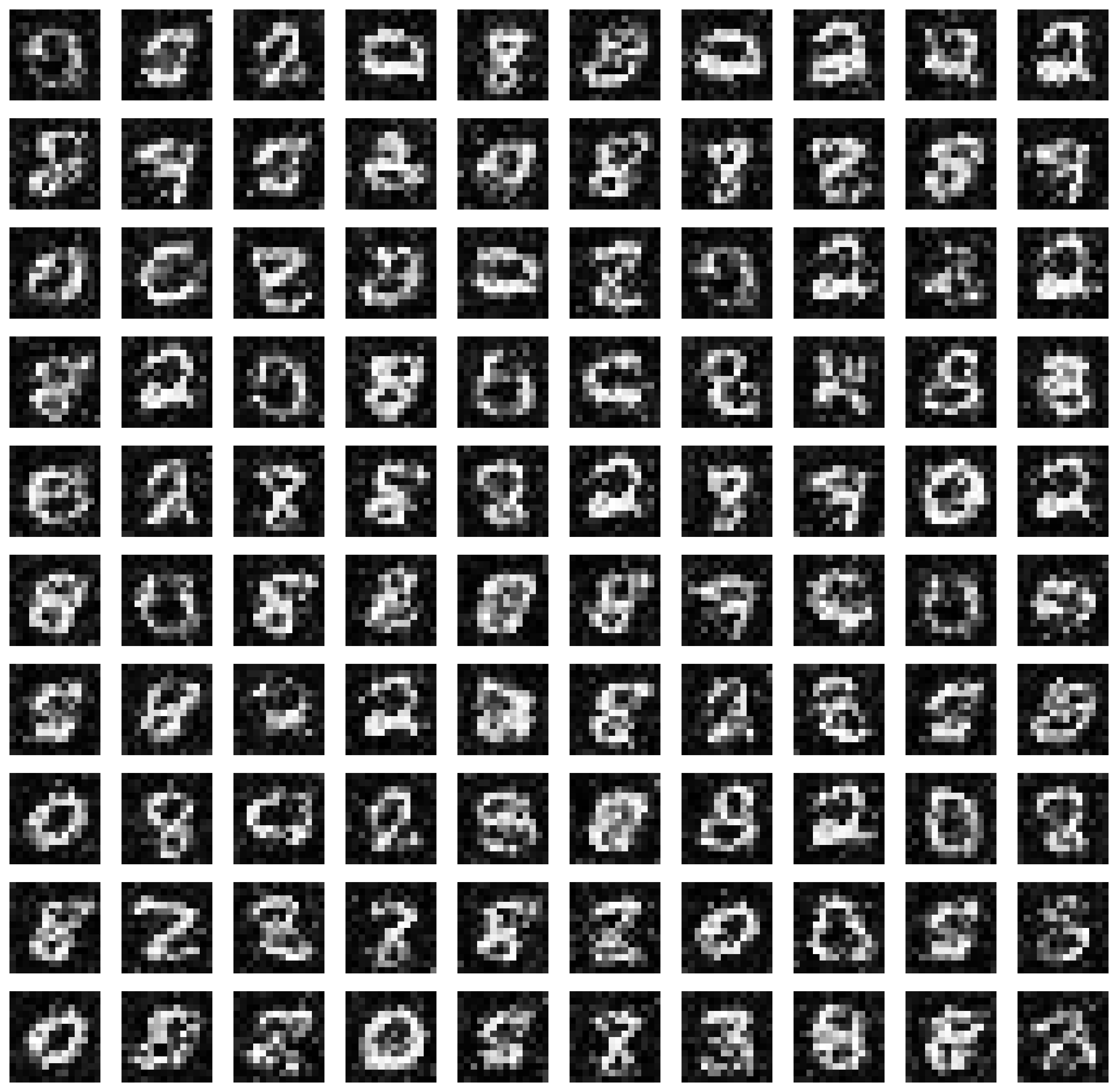}
    \end{minipage}
    \caption{a) DDS , b) PIS , c) SMC. We can see how SMC mode collapses significantly as it only seems to have 5 image classes. For PIS we identified 7 classes whilst for DDS we found 9 out of 10.} \label{fig:mnist_img}
\end{figure*}

\subsection{Euler--Mayurama vs Exponential Integrator}

In this section we demonstrate empirically how using the Euler Mayurama integrator directly on the DDS objective leads to overestimating $\ln Z$. We compare on the Funnel dataset for which we know $\ln Z$ to be $0$ and the Ionosphere dataset which is small and for which we have a reliable estimate of $\ln Z$ via a long run SMC. Results in Table \ref{fig:euvsexp} show case how for lower values of $k=64,128$ the Euler--Maruyama (EM) based approach significantly overestimates $\ln Z$ and overall does not perform well.  

\begin{table}[]
\adjustbox{max width=\textwidth}{
\begin{tabular}{@{}llccccc@{}}
\toprule
\multirow{2}{*}{} & \multicolumn{3}{c}{Funnel}                                                                             & \multicolumn{3}{c}{Ionosphere}                                                                                      \\ \cmidrule(l){2-7} 
                  & Exponential Integrator                 & Euler Mayurama                         & Ground Truth         & \multicolumn{1}{l}{Exponential Integrator} & Euler Mayurama                           & Ground Truth                \\ \midrule
$k=64$            & $-0.206 \pm 0.059$                     & $2.425 \pm 0.266$                      & \multirow{4}{*}{$0$} & $-111.693 \pm 0.169$                       & $-106.364 \pm 0.371$                     & \multirow{4}{*}{$-111.560$} \\
$k=128$           & $-0.176\pm0.099$ & $2.011\pm 0.532$                       &                      & $-111.587 \pm 0.136$                       & $-111.158\pm 0.399$                      &                             \\
$k=256$           & $-0.221 \pm 0.076$                     & \multicolumn{1}{l}{$-0.086 \pm 0.124$} &                      & $-111.575\pm 0.163$    & \multicolumn{1}{l}{$-112.271 \pm 0.366$} &                             \\
$k=512$           & $-0.176 \pm 0.068$                     & $-0.154 \pm 0.066$                     &                      & $-111.582\pm 0.136$                        & $-112.912 \pm 0.353$                     &                             \\ \bottomrule
\end{tabular}
}
\caption{Comparing EM vs Exponential integrators on Funnel and Ionosphere datasets. We can see how EM significantly overestimates $\ln Z$.\label{fig:euvsexp}}
\end{table}

\subsubsection{Detached Gradient Ablations} \label{sec:detach_exp}

In this section we perform an ablation over our proposed modification of the PIS-Grad network architecture. We compare the same architecture with and without the gradient attached. We found that detaching gradient led to a favourable performance as well as more stable training. Results are presented in table \ref{tab:detach} and were carried out for $K=128$ and using the best found diffusion coefficient from each task.  

We chose to explore this feature as it is known that optimization through unrolled computation graphs can be chaotic and introduce exploding/vanishing gradients~\citep{parmas2018pipps, metz2020using} which can numerically introduce bias. This has been successfully applied in related prior work~\citep{greff2019multi} where detached scores are provided as features to the recognition networks. Alternative approaches exist based on smoothing the chaotic loss-landscape~\citep{vicol2021unbiased} but we leave this for future work.

\if\cossqablation1
\subsection{Cosine based decay for PIS}

Unlike DDS, with PIS it is not immediately clear how to schedule the step-sizes in a similar manner to \cite{nichol2021improved}. The simplest approach would be to schedule $\delta$ (the discretisation step). We perform an ablation with the cosine schedule on the discretisation step with PIS.  Results are displayed in Table \ref{tab:cos_sq}, we can see that overall the cosine discretisation decreases or preserves the performance of PIS with the exception of the Sonar data-set in which it increases. Whilst there may be a way to improve PIS with a bespoke discretisation we were unable to find such in this work. Additionally we can see from Table \ref{tab:cos_sq} how DDS improves significantly across every task when using the cosine squared schedule compared to uniform.

\begin{table}[t]
\centering
\adjustbox{max width=\textwidth}{
\begin{tabular}{@{}llllllll@{}}
\toprule
Method          & Funnel             & LGCP                & VAE                   & Sonar                & Ionosphere           & Brownian           & NICE               \\ \midrule
PIS Cos decay & $-0.256 \pm 0.094$  & $501.060 \pm 0.881$ & $-110.025 \pm  0.0720$  & $-108.843 \pm 0.242$ & $-111.618 \pm 0.127$ & $1.035 \pm 0.097$  & $-4.212 \pm 0.620$ \\
PIS$_{\mathrm{uniform}}$ & $-0.268 \pm 0.07$  & $502.002 \pm 0.957$ & $-110.025 \pm0.070$  & $-109.349 \pm 0.356$ & $-111.602 \pm 0.145$ & $0.960 \pm 0.145$  & $-3.985 \pm 0.349$                 \\ 
DDS & $-0.176 \pm  0.098$  & $497.944 \pm 0.994$ & $-110.012 \pm 0.071$  & $-108.903 \pm 0.226$ & $-111.587 \pm 0.136$ & $1.047 \pm 0.156$  & $-3.490 \pm 0.568$                 \\ 
DDS$_{\mathrm{uniform}}$ & $-0.275 \pm  0.127$  & $477.680 \pm 1.338$ & $-110.245 \pm 0.188$  & $-109.556 \pm 0.468$ & $-111.736 \pm 0.161$ & $0.475 \pm 0.166$  & $-7.047 \pm 0.691$                 \\ \bottomrule
\end{tabular}\label{tab:cos_sq}
}
\caption{Ablation for $\cos^2$ scheduling.}
\end{table}

\fi

\begin{table}[t]
\centering
\adjustbox{max width=\textwidth}{
\begin{tabular}{@{}llllllll@{}}
\toprule
Method          & Funnel             & LGCP                & VAE                   & Sonar                & Ionosphere           & Brownian           & NICE               \\ \midrule
PIS Grad Detach & $-0.268 \pm 0.07$  & $502.002 \pm 0.957$ & $-110.025 \pm0.0704$  & $-109.349 \pm 0.356$ & $-111.602 \pm 0.145$ & $0.960 \pm 0.145$  & $-3.985 \pm 0.349$ \\
PIS Grad        & $-0.268 \pm  0.09$ & $501.273 \pm 0.798$ & $-110.033 \pm 0.0647$ & $-109.377 \pm 0.361$ & $-111.705 \pm 0.190$ & $-5.793 \pm 0.499$ & $-3.927 \pm  0.666$                 \\ \bottomrule
\end{tabular}
}
\caption{Results for PIS grad with and without detaching the gradient. \label{tab:detach}}
\end{table}

\subsection{Probability Flow ODE}\label{app:probaflowODE}

Figures \ref{fig:prob_flow_mix} and \ref{fig:flow_ode_simple} show samples obtained from the probability flow ODE of a trained DDS model. We can see that the probability flow ODE is able to perfectly sample from a uni-modal Gaussian, matters became more challenging with multi-modal distributions.  Initially we discretised the ODE (\ref{eq:ODEproposal}) using the same type of integrators as the SDEs to obtain $ y_{k+1} = y_{k} +  \delta \sigma^2 (1 - \sqrt{1-\alpha_{K-k} }) f_\theta (K-k, y_{k})$ for $y_0 \sim \mathcal{N}(0;\sigma^2 I)$. Unfortunately under this discretisation we found the probability flow ODE to become stiff with more complex distributions such as the mixture of Gaussians, resulting in strange effects in the samples (matching the modes but completely wrong shapes). By using a Heun integrator as proposed in \cite{karras2022elucidating} we were able to obtain better results, however we also found we needed to increase the network size to improve the results as well as train with a learning rate decay of $0.99$. With these extra features we were able to simulate a probability flow ODE that roughly matched the marginal densities of the SDE. However even with these fixes we still found the probability flow based estimator of $\ln Z$ to drastically overestimate, this is not entirely surprising as in discrete time the expectation of this estimator is not guaranteed to be an ELBO.

\subsection{Underdamped OU Results} \label{apdx:undexp}

We perform some additional experiments with the underdamped OU reference process. Similar to the damped setting we parametrise the network as:
\begin{align}
    f_\theta(k, x, p) = \mathrm{NN}_1(k, x, p; \theta)  + \mathrm{NN}_2(k; \theta) \nabla \ln \pi(x).
\end{align}
Unfortunately, unlike in DDS and PIS we cannot directly aid the update/proposal for $x_t$ with $ \nabla \ln \pi(x)$, thus this naive parametrisation is not the ideal inductive bias.

Results in Figures \ref{fig:benchmark_targets_udmp} and \ref{fig:logreg_targets_udmp} show some experiments for this approach. The results perform worse than DDS for a standard OU reference process and PIS. However they are still better than VI. We believe that future work exploring better inductive biases for the network $f_\theta(k, x, p)$ could narrow the performance gap in this approach. Additionally we highlight that in Figure \ref{fig:benchmark_targets_udmp} (Funnel Distribution) we can see the underdamped approach overestimates $\ln Z$ for $k=64$. We believe this may be due to numerical error when generating a sample.

\begin{figure}
    \centering
    \includegraphics[width=\linewidth]{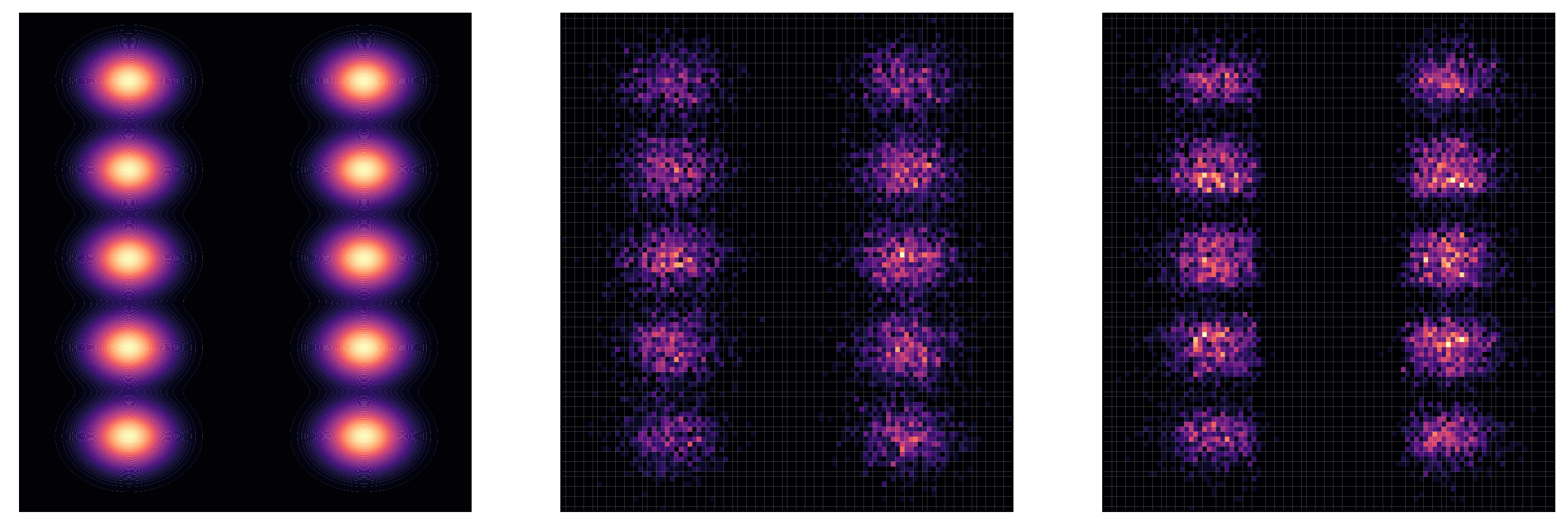}
    \caption{a) Target distribution (MoG), b) DDS SDE samples, c) trained probability flow ODE.}
    \label{fig:prob_flow_mix}
\end{figure}

\subsection{Box Plots for Main Results}

For completeness in this section we also present our results via box plots as can be seen in Figures \ref{fig:base_targets_old} and \ref{fig:logreg_targets_old}.

\section{Denoising Diffusion Samplers in Discrete-time: Further details} \label{apdx:discrete}

  \subsection{Forward Process and Its Time Reversal}\label{subsec:DDS-DT}
In discrete-time, we consider the following discrete-time ``forward'' Markov process 
\begin{equation}\label{eq:forwardprocessP}
    p(x_{0:K})=\pi(x_0) \prod_{k=1}^K p_{k|k-1}(x_k|x_{k-1}).
\end{equation}
Following (\ref{eq:OUrefintegrated}), we use $p_{k|k-1}(x_k|x_{k-1})=\mathcal{N}(x_k;\sqrt{1-\alpha_{k}}x_{k-1},\sigma^2 \alpha_k I)$. The parameters $(\alpha_{k})_{k=1}^K$ are such that $p_K(x_K) \approx \mathcal{N}(x_K;0,\sigma^ 2 I)$.

We propose here to sample approximately from $\pi$ by approximating the ancestral sampling scheme for (\ref{eq:forwardprocessP}) corresponding to the backward decomposition
\begin{equation}\label{eq:backwardprocessP}
    p(x_{0:K})=p_K(x_K)\prod_{k=1}^{K} p_{k-1|k}(x_{k-1}|x_{k}),~~\text{with}~~ p_{k-1|k}(x_{k-1}|x_{k})=\frac{p_{k-1}(x_{k-1})p_{k|k-1}(x_{k}|x_{k-1})}{p_{k}(x_{k})},
\end{equation}
where $p_0(x_0)=\pi(x_0)$. If we could thus sample $x_K \sim p_K(\cdot)$ then $x_{k-1} \sim p_{k-1|k}(\cdot |x_{k})$ for $k=K,...,1$ then $x_0$ would indeed be a sample from $\pi$. However as the marginal densities $(p_k)_{k=1}^K$ are not available in closed-form, we cannot implement exactly this ancestral sampling procedure. 

As we have by design $p_K(x_K) \approx \mathcal{N}(x_K;0,\sigma^2 I)$, we can simply initialize the ancestral sampling procedure by a Gaussian sample. The approximation of the backward Markov kernels $p_{k-1|k}(x_{k-1}|x_{k})$ is however more involved.

\subsection{Reference Process and Bellman Recursion}
We introduce the ``simple'' reference process $p^{\textup{ref}}(x_{0:K})$ defined by
\begin{equation}\label{eq:processPref}
    p^{\textup{ref}}(x_{0:K})=\mathcal{N}(x_0;0,\sigma^2 I)p(x_{1:K}|x_0)= \mathcal{N}(x_0;0,\sigma^2 I)\prod_{k=1}^K p_{k|k-1}(x_k|x_{k-1})
\end{equation}
which is designed by construction to admits the marginal distributions $ p^{\textup{ref}}_k(x_k)=\mathcal{N}(x_k;0,\sigma^2 I)$ for all $k$. It can be easily verified using the chain rule for KL that the extended target process $p$ is the distribution minimizing the reverse (or forward) KL discrepancy w.r.t. $p^{\textup{ref}}$ over the set of path measures $q(x_{0:K})$ with marginal $q_0(x_0)=\pi(x_0)$ at the initial time, i.e.
\begin{figure}
    \centering
    \includegraphics[width=0.9\linewidth]{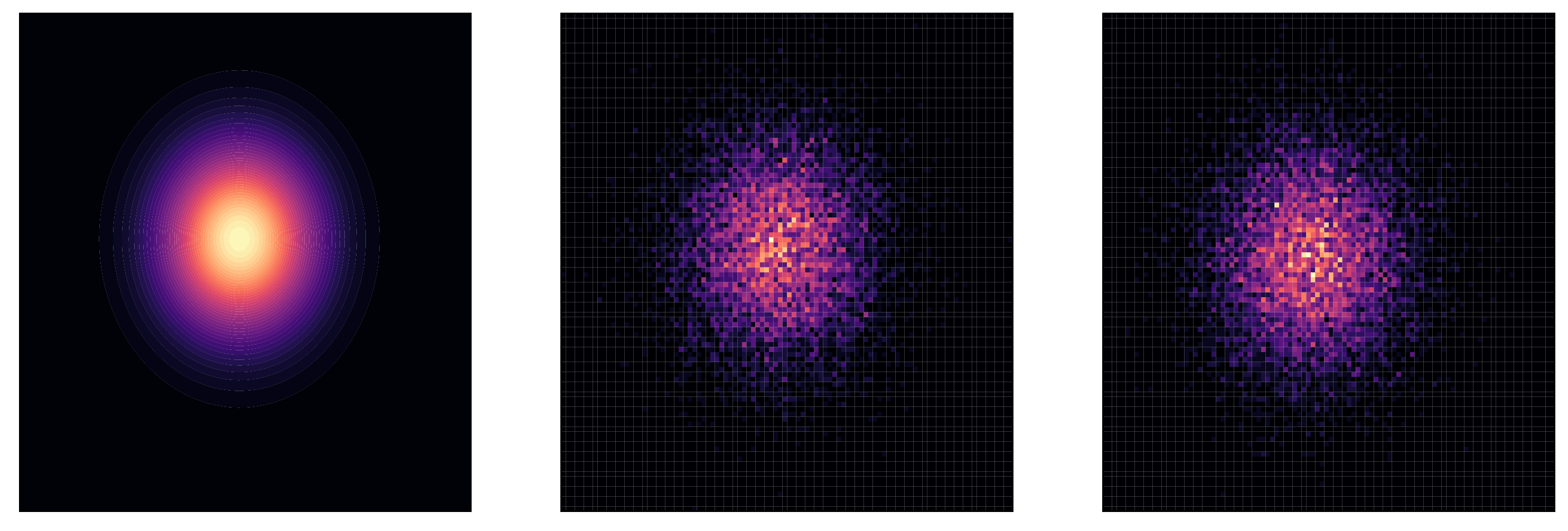}
    \caption{
    a) Target distribution $\gN(6,1)$, b) DDS SDE c) Trained DDS probability flow ODE.}
    \label{fig:flow_ode_simple}
\end{figure}
\begin{figure*}[t!]
    \centering
    \begin{minipage}{0.32\linewidth }
    \centering
    \includegraphics[width=\linewidth ]{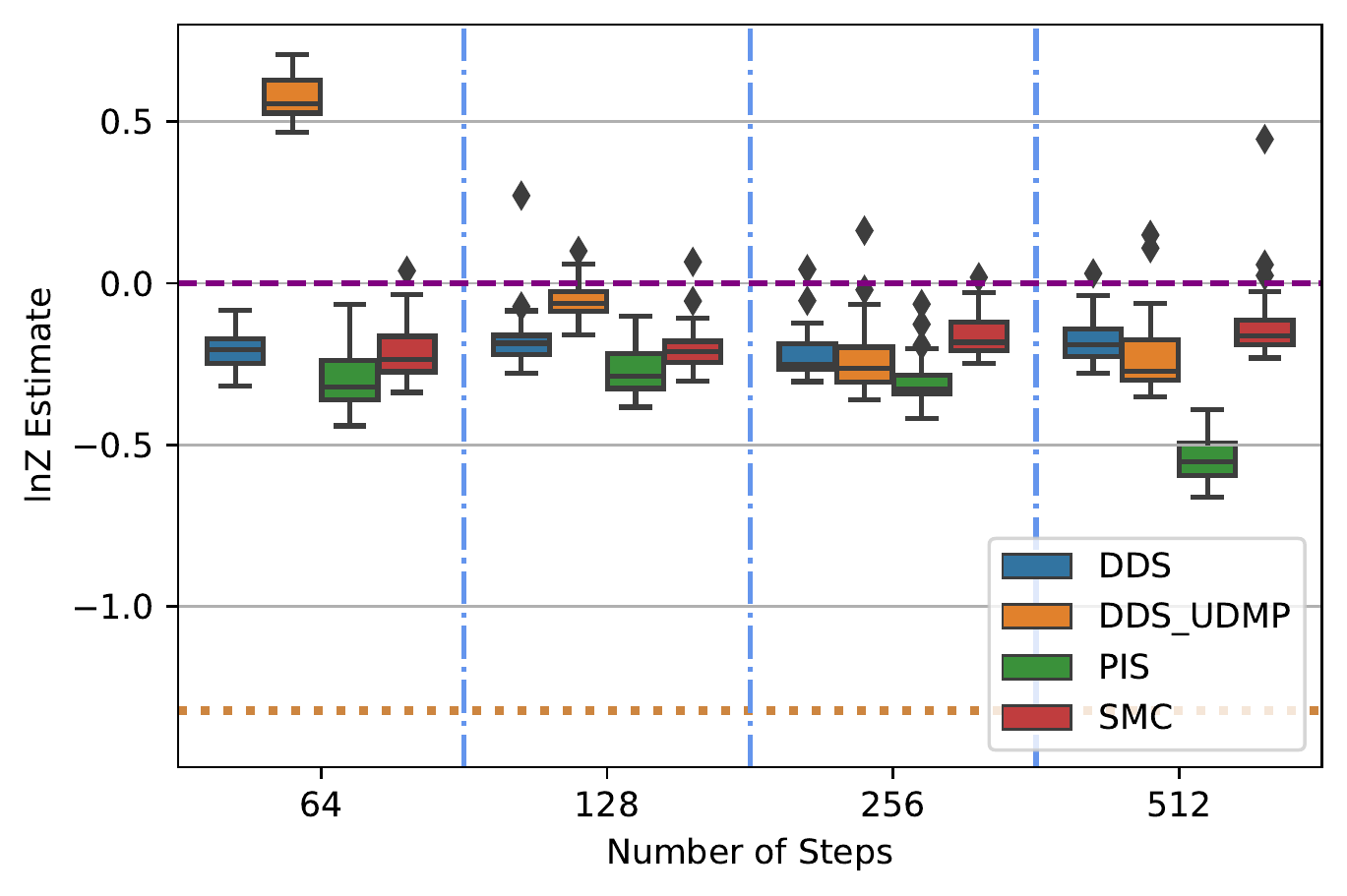}
    \end{minipage}
    \hspace*{\fill}\begin{minipage}{0.32\linewidth }
    \centering
    \includegraphics[width=\linewidth ]{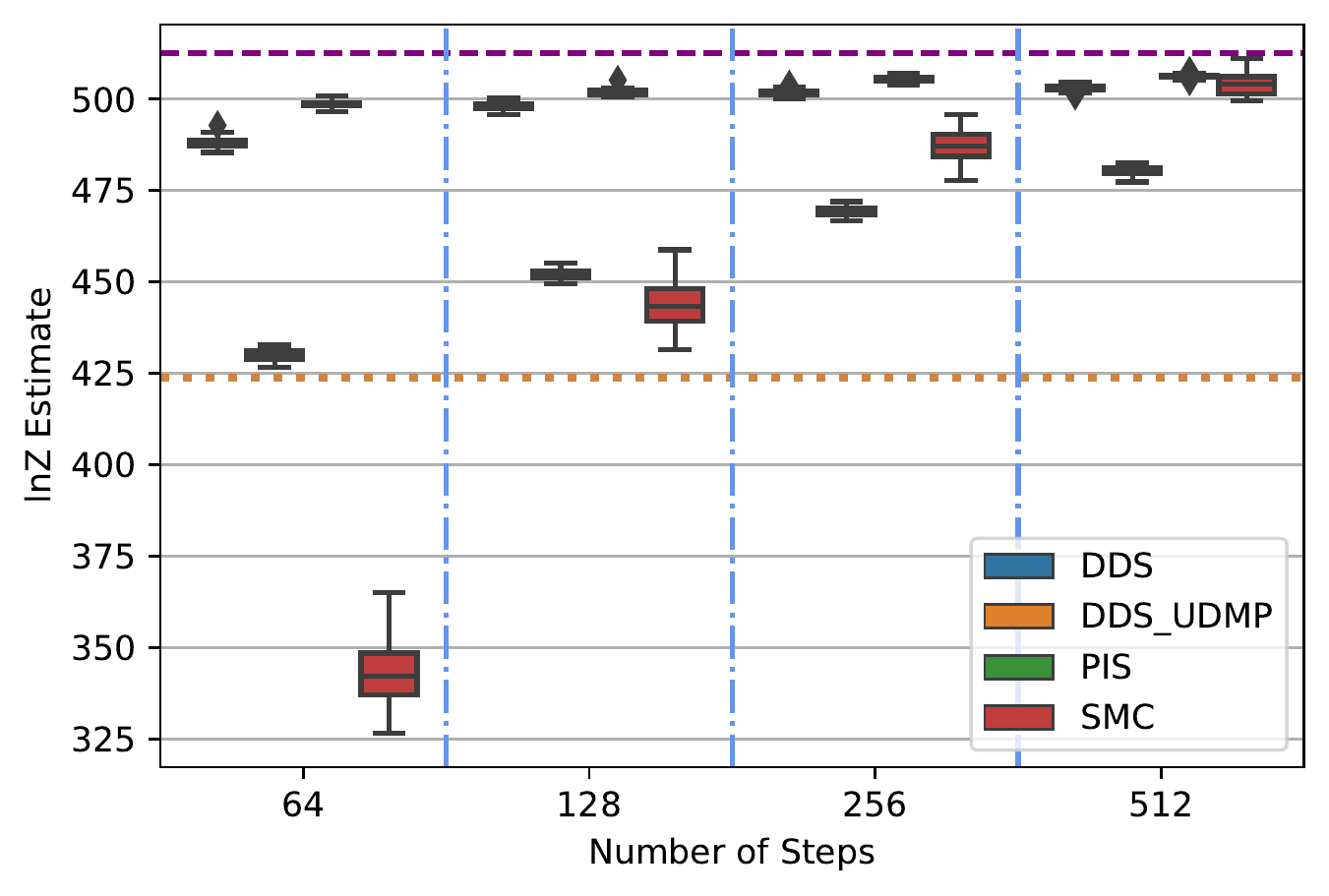}
    \end{minipage}
    \hspace*{\fill}\begin{minipage}{0.32\linewidth }
    \centering
    \includegraphics[width=\linewidth ]{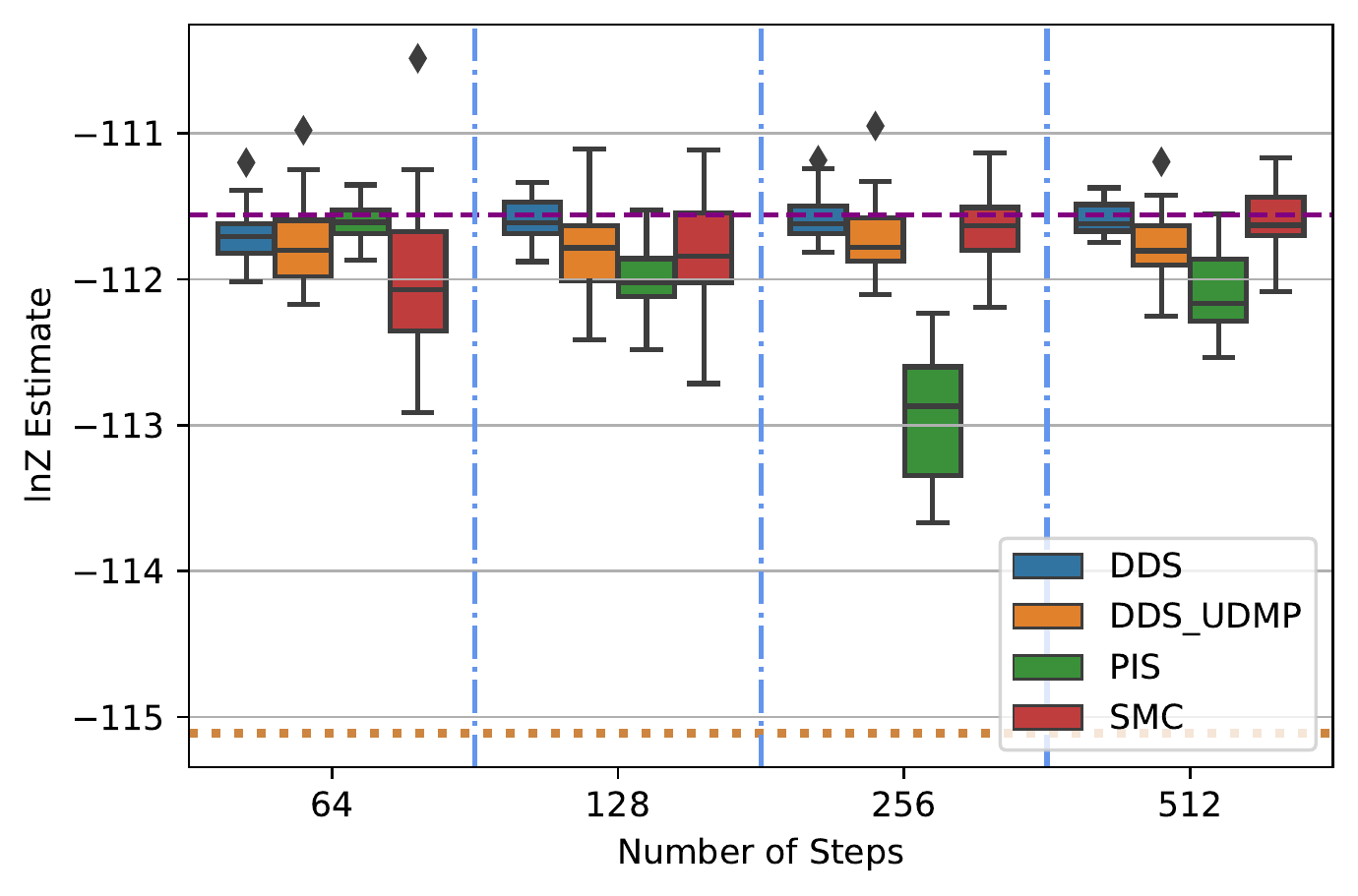}
    \end{minipage}
    \caption{$\log Z$ estimate as a function of number of steps $K$ - a) Funnel , b) LGCP, c) Logistic Ionosphere dataset. Yellow Dotted line is MF-VI and dashed magenta is the gold standard.} \label{fig:benchmark_targets_udmp}
\end{figure*}
\begin{figure*}[t!]
    \centering
    \begin{minipage}{0.33\linewidth }
    \centering
    \includegraphics[width=\linewidth ]{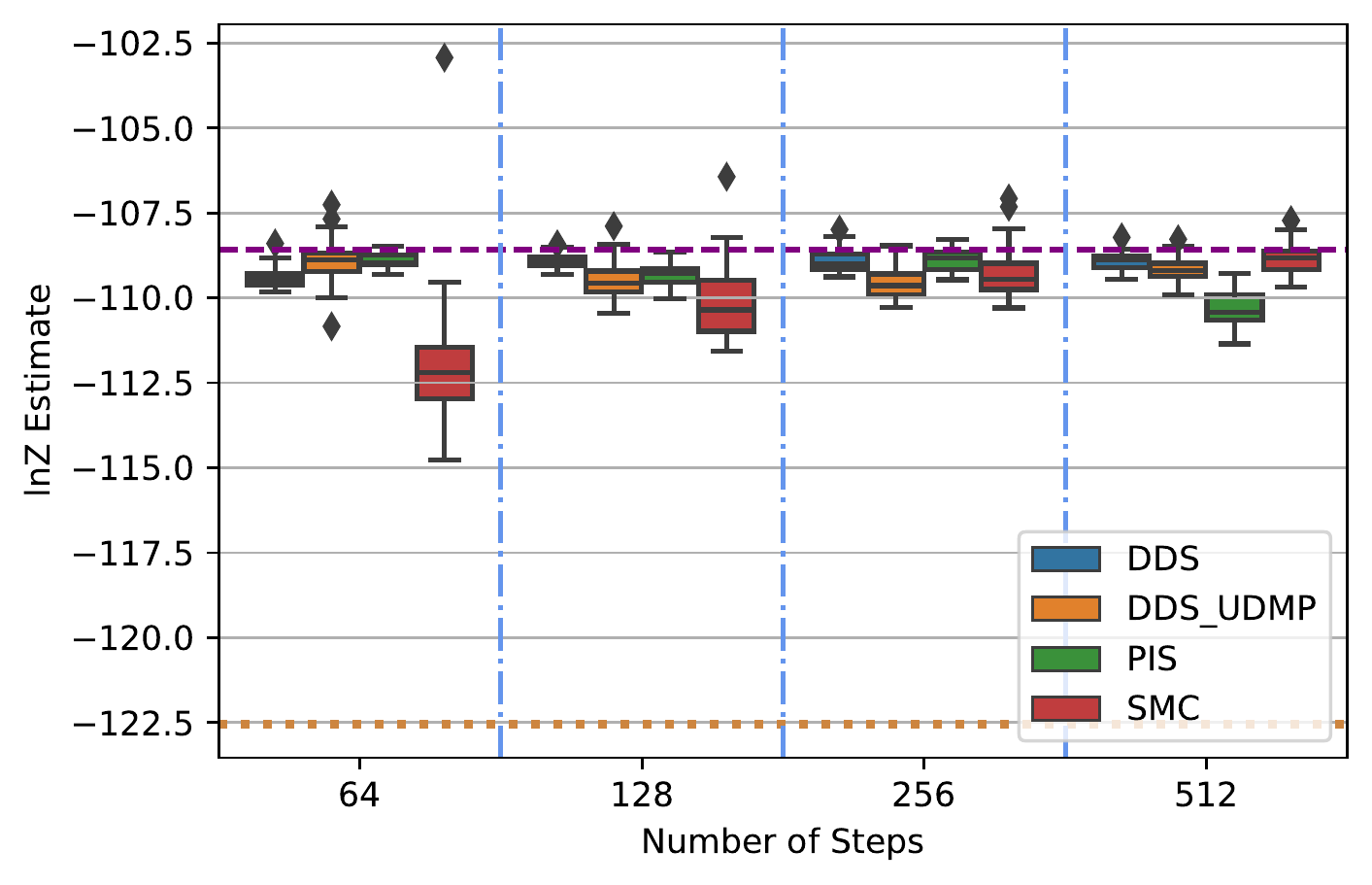}
    \end{minipage}
    \hspace*{\fill}\begin{minipage}{0.32\linewidth }
    \centering
    \includegraphics[width=\linewidth ]{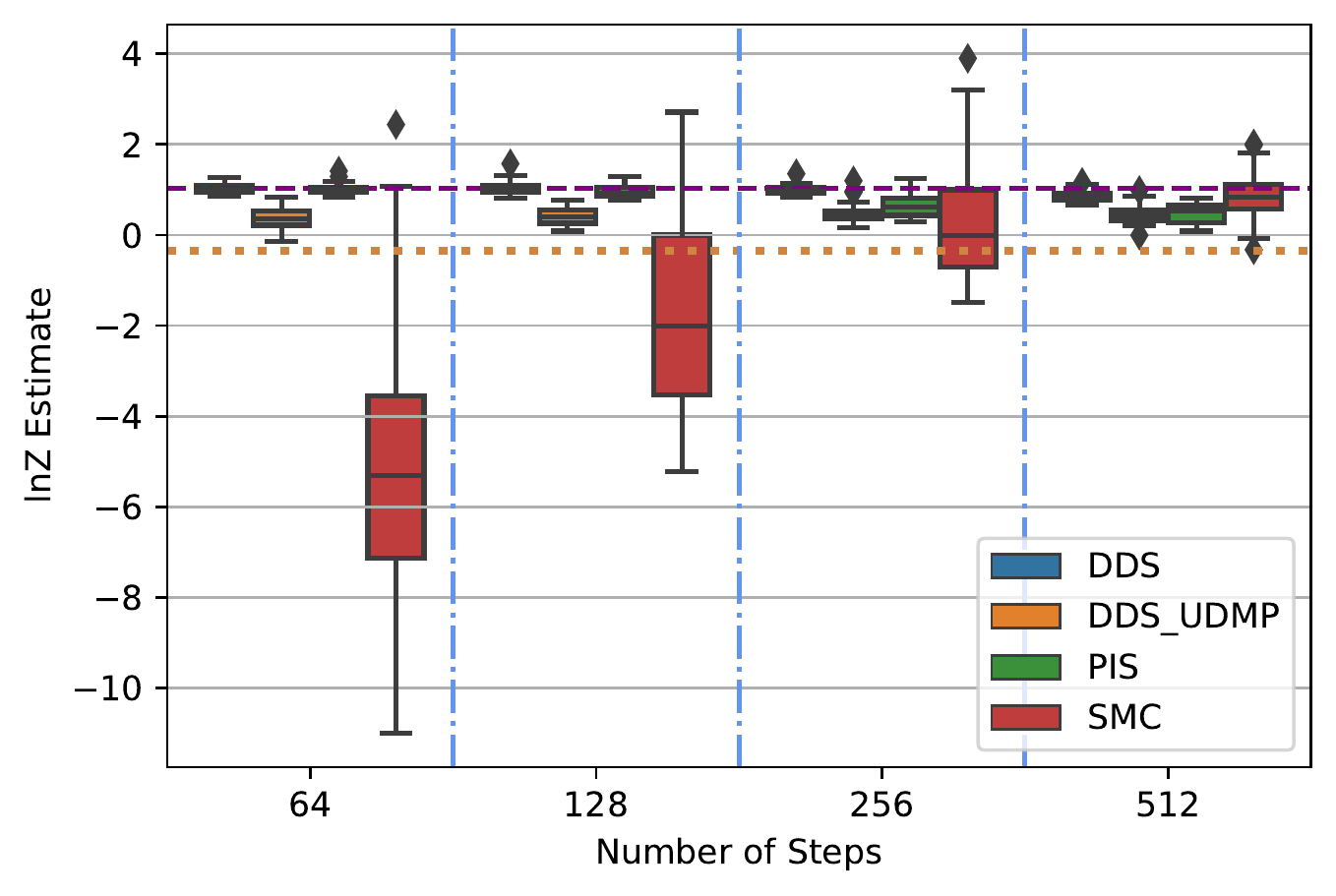}
    \end{minipage}
     \hspace*{\fill}
    \hspace*{\fill}\begin{minipage}{0.32\linewidth }
    \centering
    \includegraphics[width=\linewidth ]{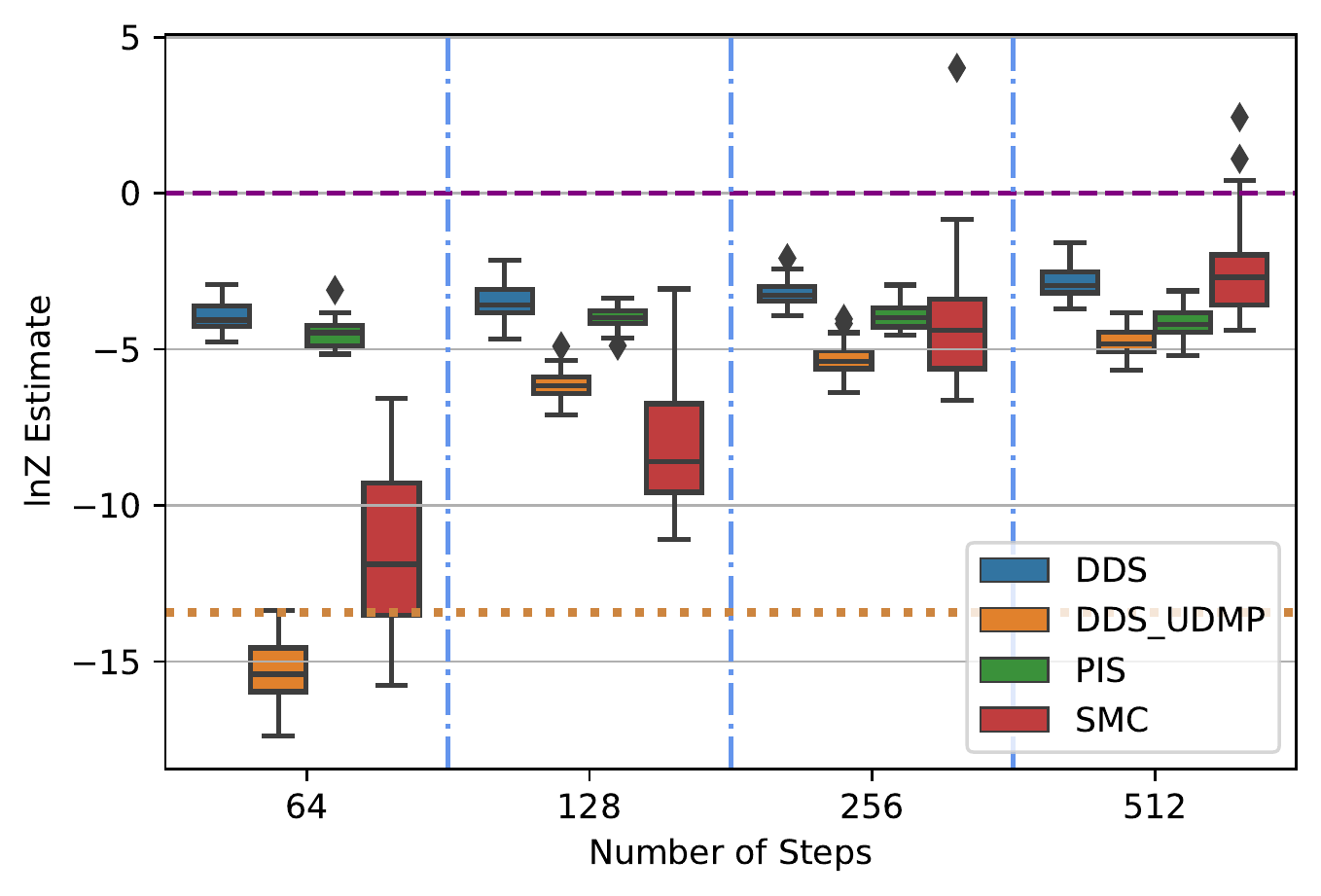}
    \end{minipage}
    
    \caption{$\log Z$ estimate as a function of number of steps $K$ - a) Logistic Sonar dataset, b) Brownian motion, c) NICE. Yellow dotted line is MF-VI and dashed magenta is the gold standard.} \label{fig:logreg_targets_udmp}
\end{figure*}
\begin{equation*}
    p = \argmin_q \Bigl\{\KL(q||p^{\textup{ref}}): q_0=\pi \Bigl\}.
\end{equation*}
The intractable backward Markov densities $p_{k-1|k}$ one wants to approximate can also be rewritten as twisted versions of the tractable backward Markov densities $p^{\textup{ref}}_{k-1|k}$ by some value functions $(\phi_k)_{k=0}^K$.
\begin{proposition}\label{prop:valuefunction}
We have for $k=1,...,K$
\begin{equation}\label{eq:backwardkernelvaluefunctions}
    p_{k-1|k}(x_{k-1}|x_{k})=\frac{\phi_{k-1}(x_{k-1}) p^{\textup{ref}}_{k-1|k}(x_{k-1}|x_{k})}{\phi_{k}(x_{k})},\quad\text{for}\quad \phi_k(x_k)=\frac{p_k(x_k)}{p^{\textup{ref}}_k(x_k)}.
\end{equation}
The value functions $(\phi_k)_{k=1}^K$ satisfy a forward Bellman type equation
\begin{align}\label{eq:valuefunctionBellman}
    \phi_k(x_k)=\int \phi_{k-1}(x_{k-1})p^{\textup{ref}}_{k-1|k}(x_{k-1}|x_{k})\mathrm{d}x_{k-1},\qquad \phi_0(x_0)=\frac{\pi(x_0)}{p^{\textup{ref}}_0(x_0)}.
\end{align}
It follows that
\begin{align}\label{eq:recursionvaluefunction}
    \phi_k(x_k)&= \int \phi_0(x_0)~~p^{\textup{ref}}_{0|k}(x_{0}|x_{k})\mathrm{d}x_{0}.
\end{align}
\end{proposition}

\begin{proof}
To establish (\ref{eq:backwardkernelvaluefunctions}), we use Bayes' rule
\begin{align*}
    p_{k-1|k}(x_{k-1}|x_{k})&=\frac{p_{k-1}(x_{k-1})p_{k|k-1}(x_{k}|x_{k-1})}{p_{k}(x_{k})}\\
    &=\frac{\phi_{k-1}(x_{k-1})p^{\textup{ref}}_{k-1}(x_{k-1})p_{k|k-1}(x_{k}|x_{k-1})}{\phi_{k}(x_{k})p^{\textup{ref}}_{k}(x_{k-1})}\\
    &=\frac{\phi_{k-1}(x_{k-1})p^{\textup{ref}}_{k-1|k}(x_{k-1}|x_{k})}{\phi_{k}(x_{k})},
\end{align*}

where we have used the fact that $p_k(x_k)=\phi_k(x_k)p^{\textup{ref}}_{k}(x_{k})$, $p^{\textup{ref}}_{k|k-1}(x_{k}|x_{k-1})=p_{k|k-1}(x_{k}|x_{k-1})$ and Bayes' rule again.
Now we have $\int p_{k-1|k}(x_{k-1}|x_{k}) \mathrm{d}x_{k-1}=1$ for any $x_k$, so it follows directly from the expression of this transition kernel that the value function satisfies
 \begin{align}\label{eq:valuefunctionBellman2}
    \phi_k(x_k)=\int \phi_{k-1}(x_{k-1})p^{\textup{ref}}_{k-1|k}(x_{k-1}|x_{k})\mathrm{d}x_{k-1},\qquad \phi_0(x_0)=\frac{\pi(x_0)}{p^{\textup{ref}}_0(x_0)}.
\end{align}
By iterating this recursion, it follows that
\begin{align}\label{eq:recursionvaluefunction2}
    \phi_k(x_k)&= \int \phi_0(x_0)~~p^{\textup{ref}}_{0|k}(x_{0}|x_{k})\mathrm{d}x_{0}.
\end{align}
\end{proof}

\begin{figure*} 
    \centering
    \begin{minipage}{0.32\linewidth }
    \centering
    \includegraphics[width=\linewidth ]{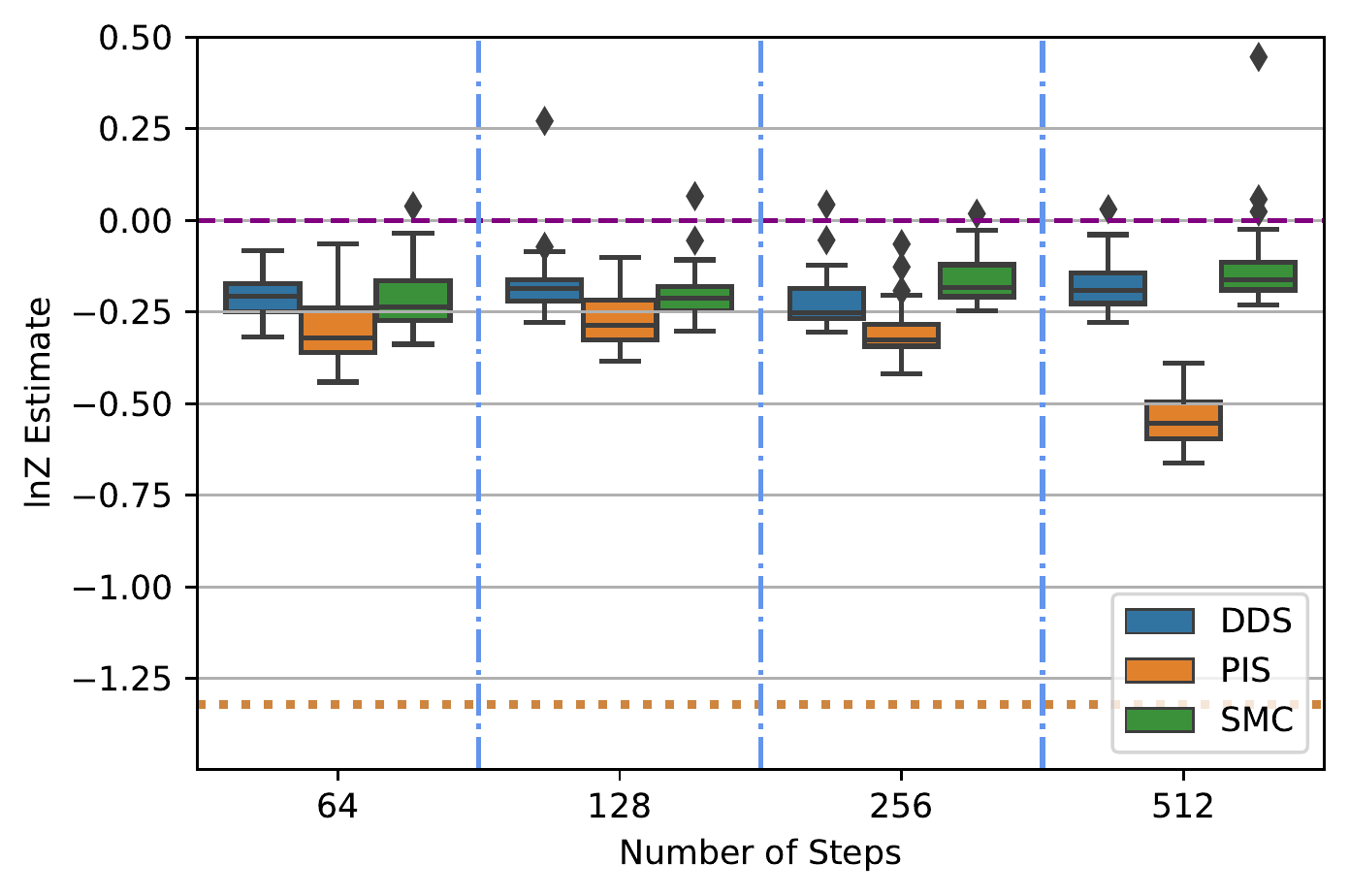}
    \end{minipage}
    \hspace*{\fill}\begin{minipage}{0.32\linewidth }
    \centering
    \includegraphics[width=\linewidth ]{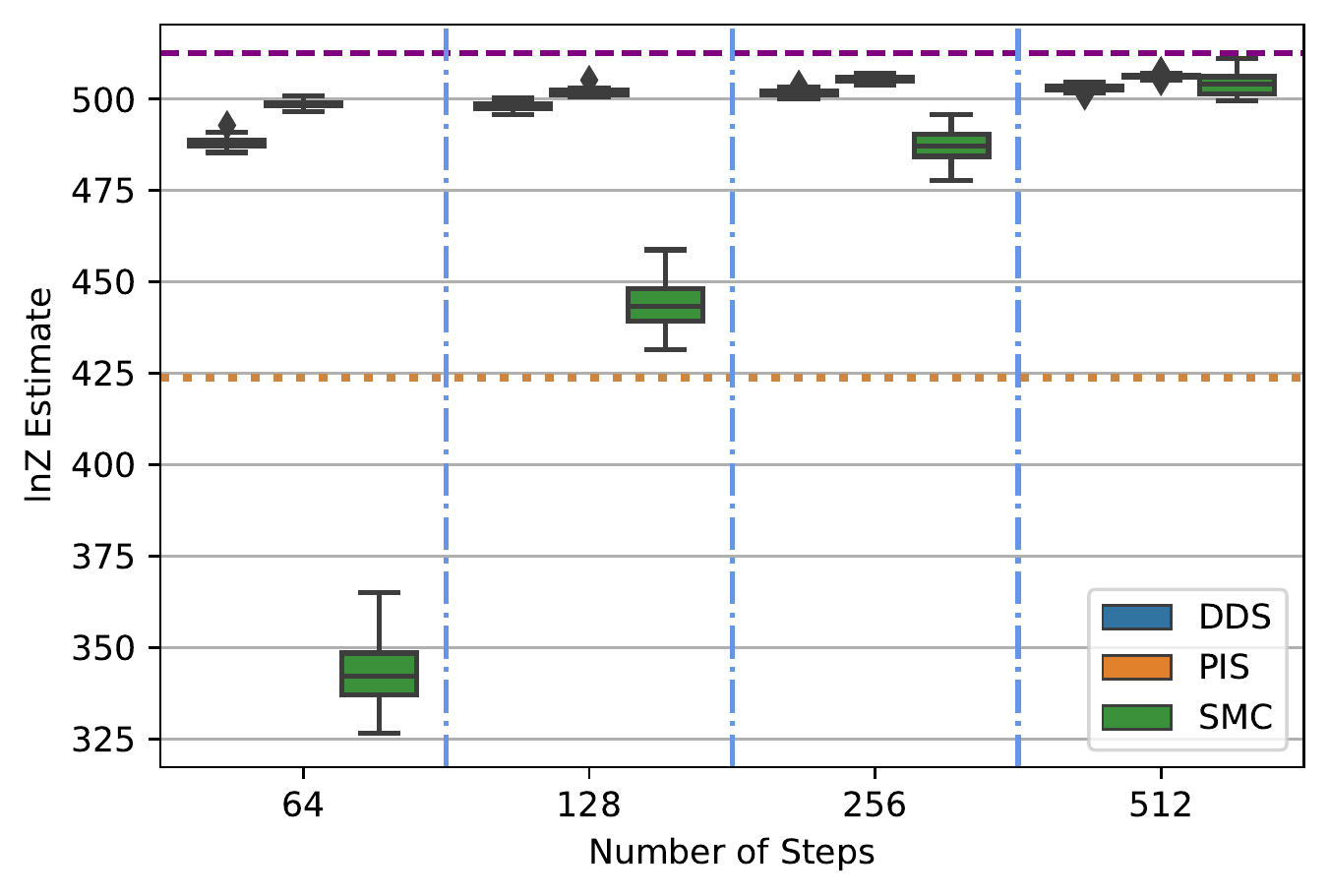}
    \end{minipage}
    \hspace*{\fill}\begin{minipage}{0.32\linewidth }
    \centering
    \includegraphics[width=\linewidth ]{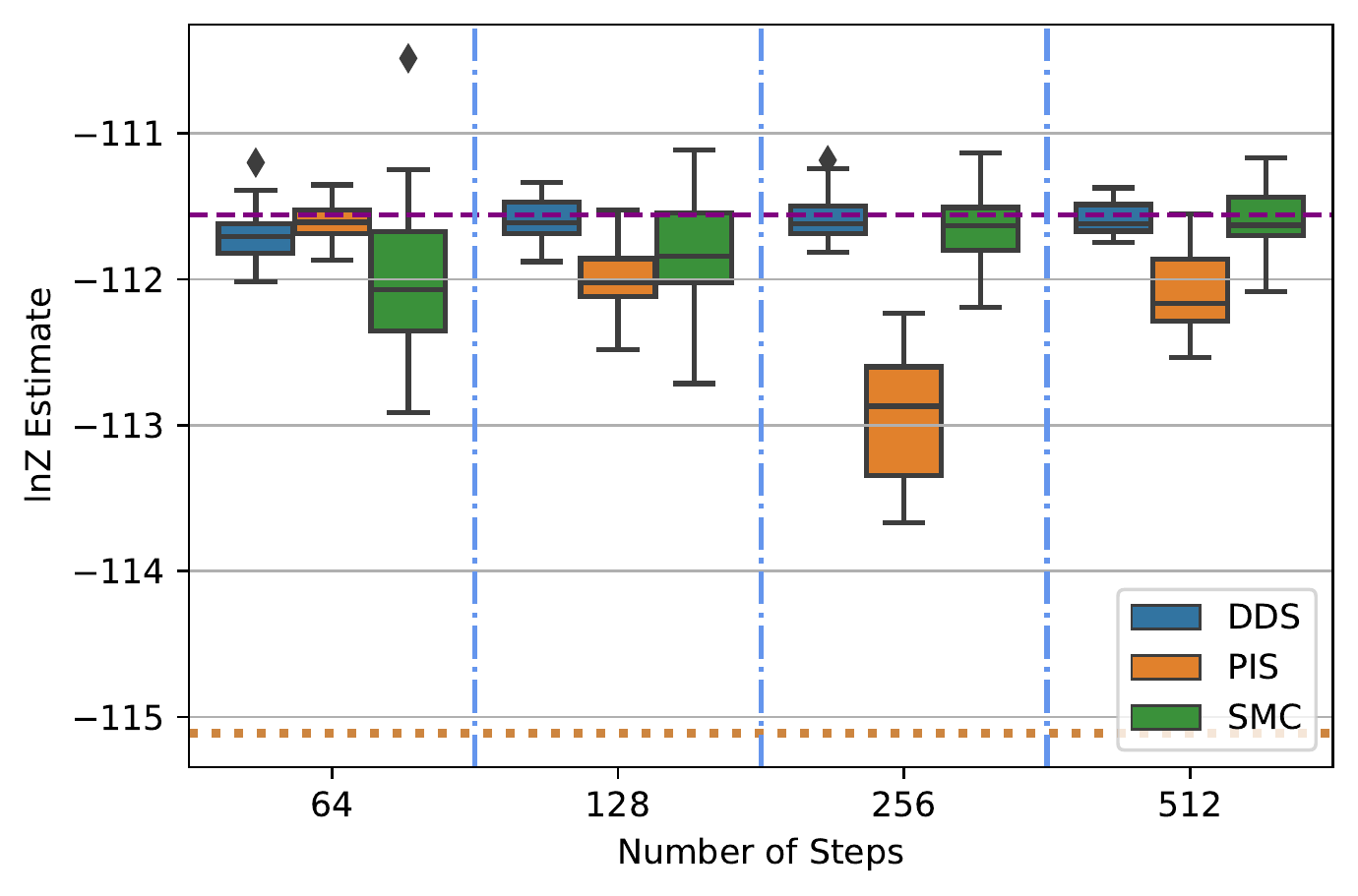}
    \end{minipage}
    \caption{$\log Z$ estimate as a function of number of steps $K$ - a) Funnel , b) LGCP, c) Logistic Ionosphere dataset. Yellow Dotted line is MF-VI and dashed magenta is the gold standard.} \label{fig:base_targets_old}
\end{figure*}
\begin{figure*}  
    \centering
    \begin{minipage}{0.32\linewidth }
    \centering
    \includegraphics[width=\linewidth ]{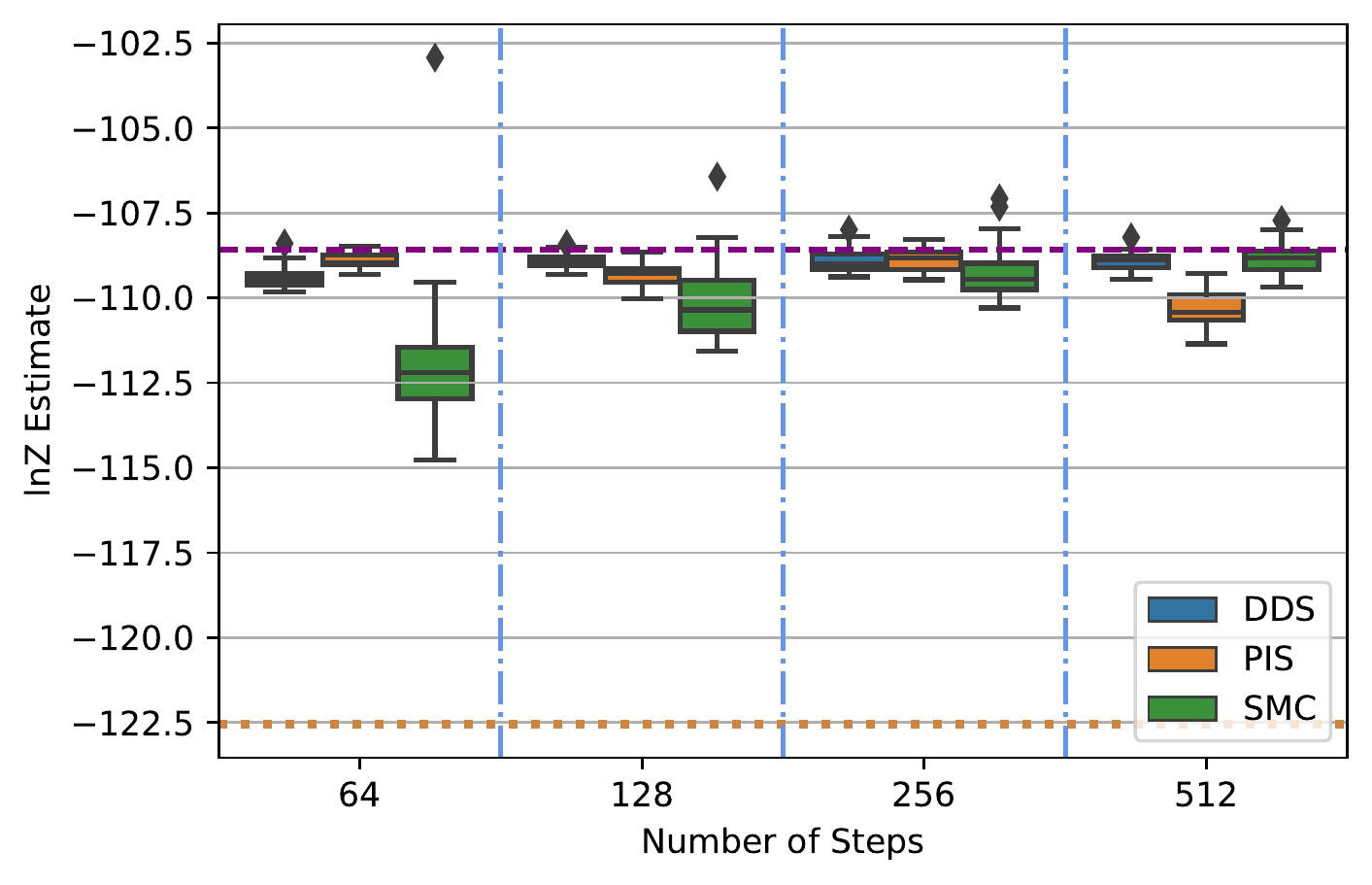}
    \end{minipage}
    \hspace*{\fill}\begin{minipage}{0.32\linewidth }
    \centering
    \includegraphics[width=\linewidth ]{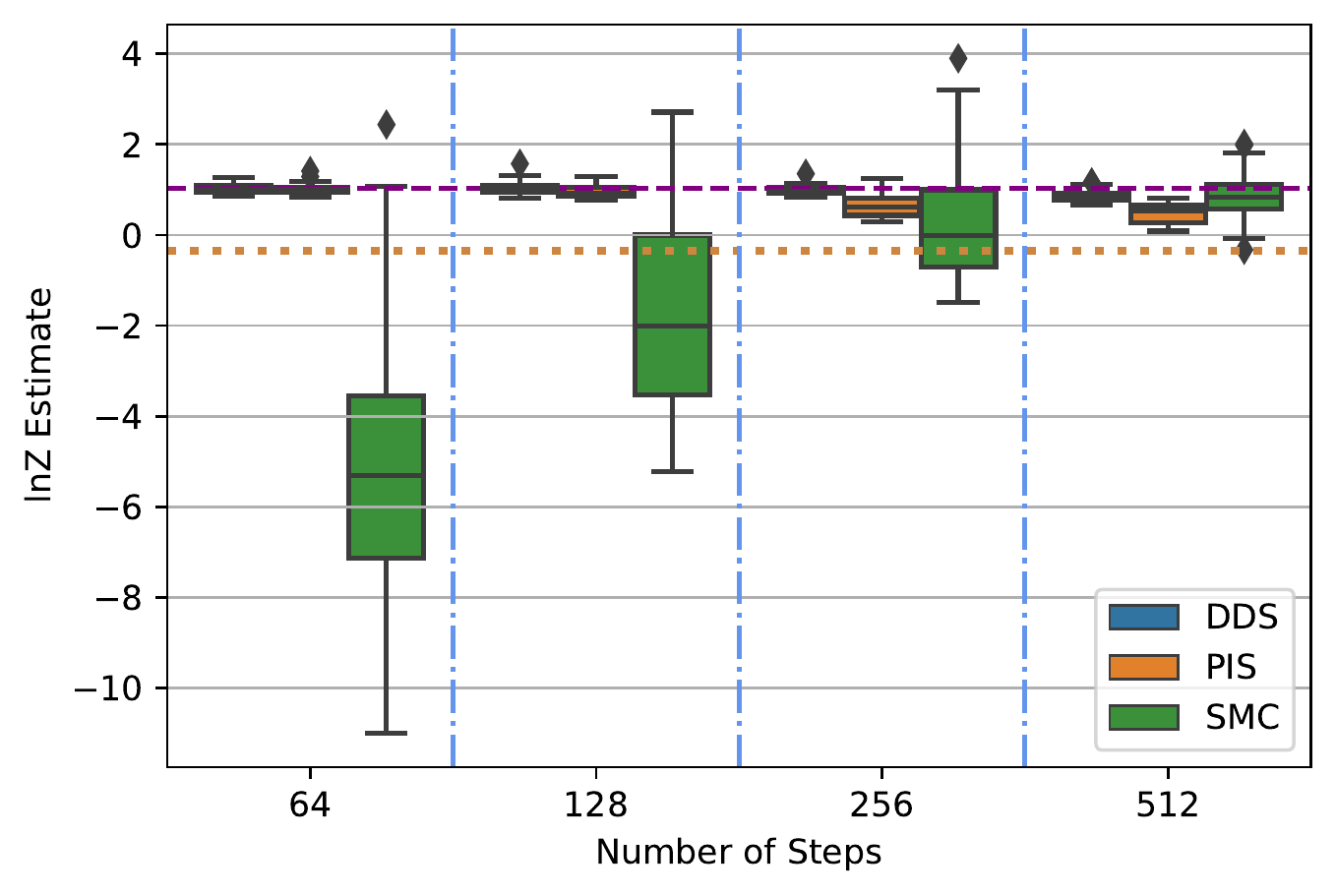}
    \end{minipage}
     \hspace*{\fill}
    \hspace*{\fill}\begin{minipage}{0.32\linewidth }
    \centering
    \includegraphics[width=\linewidth ]{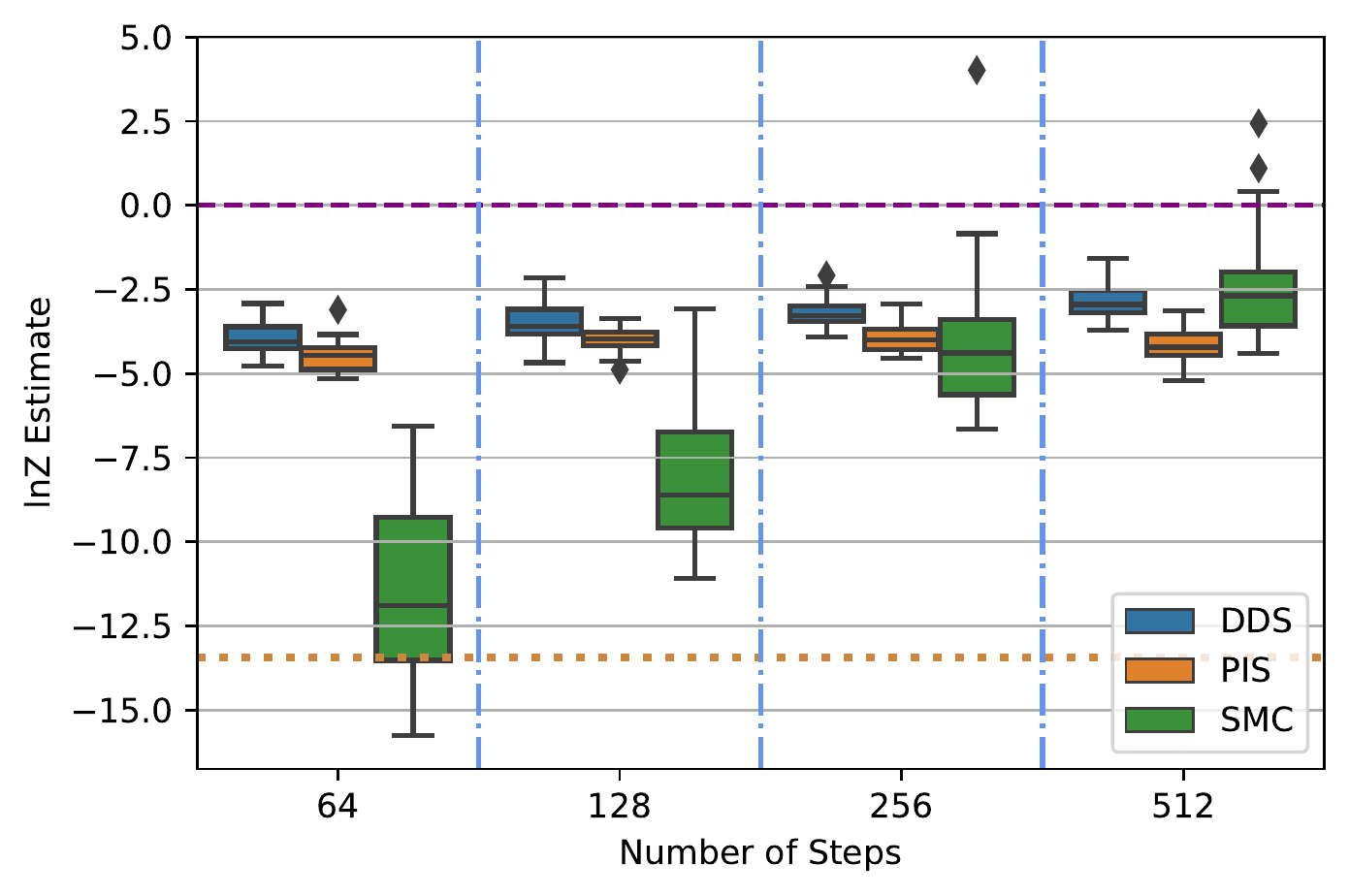}
    \end{minipage}
    
    \caption{$\log Z$ estimate as a function of number of steps $K$ - a) Logistic Sonar dataset, b) Brownian motion, c) NICE. Yellow dotted line is MF-VI and dashed magenta is the gold standard.} \label{fig:logreg_targets_old}
\end{figure*}

\subsection{Approximating the Backward kernels}
We want to approximate the backward Markov transitions $p_{k-1|k}$.  From (\ref{eq:backwardkernelvaluefunctions}), this would require not only to approximate the value functions - Monte Carlo estimates based on (\ref{eq:recursionvaluefunction}) are high variance - but also to sample from the resulting twisted kernel which is difficult.
However, if we select $\beta_k \approx 0$, then we have $\phi_{k-1}(x)\approx \phi_k(x)$ and by a Taylor expansion of $\nabla \log \phi_k$ around $x_k$ we obtain
\begin{align}
     p_{k-1|k}(x_{k-1}|x_{k})&=p^{\textup{ref}}_{k-1|k}(x_{k-1}|x_{k}) \exp(\log \phi_{k-1}(x_{k-1})- \log \phi_{k}(x_{k})) \nonumber\\
     &\approx \mathcal{N}(x_{k-1};\sqrt{1-\alpha_k}x_k+\sigma^2(1-\sqrt{1-\alpha_k}) \nabla \log \phi_k(x_k))
     \label{eq:approxDTreversalvalue}.
\end{align}
By differentiating the identity (\ref{eq:recursionvaluefunction}) w.r.t. $x_k$, we obtain the following expression for $\phi_k$
\begin{equation}\label{eq:derivativevaluefunction}
    \nabla \log \phi_k(x_k)=\int \nabla \log p^{\textup{ref}}_{0|k}(x_0|x_{k})~~  p_{0|k}(x_{0}|x_{k})\mathrm{d}x_{0}.
\end{equation}
Alternatively, we can also use $\nabla \log \phi_k(x_k)=\nabla \log p_k(x_k)+\tfrac{x}{\sigma^2}$ where $\nabla \log p_k(x_k)$ can be easily shown to satisfy 
\begin{equation}\label{eq:scoreidentity}
    \nabla \log p_k(x_k) = \int \nabla \log p_{k|0}(x_k|x_0)~~p_{0|k}(x_0|x_k)\mathrm{d}x_0.
\end{equation}
For DDPM, the conditional expectation in (\ref{eq:scoreidentity}) is reformulated as the solution to a regression problem, leveraging the fact that we can easily obtain samples from $ \pi(x_0)p_{k|0}(x_k|x_0)$ as $\pi$ is the data distribution in this context. In the Monte Carlo sampling context considered here, we could sample approximately from $p_{0|k}(x_0|x_k) \propto \pi(x_0)p_{k|0}(x_k|x_0)$ using MCMC so as to approximate the conditional expectations in (\ref{eq:derivativevaluefunction}) and (\ref{eq:scoreidentity}) but this would defeat the purpose of the proposed methodology. We propose instead to approximate $\nabla \log \phi_k$ by minimizing a suitable criterion.

In practice, we will consider a distribution $q^{\theta}(x_{0:K})$ approximating $p(x_{0:K})$ of the form
\begin{equation}\label{eq:proposalqreversetime}
    q^{\theta}(x_{0:K})=\mathcal{N}(x_K;0,\sigma^2 I)\prod_{k=1}^K q^\theta_{k-1|k}(x_{k-1}|x_k),
\end{equation}
i.e. we select $q^{\theta}_K(x_{K})=q_K(x_{K})=\mathcal{N}(x_K;0,\sigma^2 I)$ as $p_K(x_K)\approx \mathcal{N}(x_K;0,\sigma^2 I)$. We want $q^\theta_{k-1|k}(x_{k-1}|x_k)$ to approximate $p_{k-1|k}(x_{k-1}|x_k)$, i.e. $f_\theta(k,x_k) \approx \nabla \log \phi_k(x_k)$. 
Inspired by (\ref{eq:approxDTreversalvalue}), we will consider an approximation of the form 
\begin{equation} \label{eq:algo_param}
    q^\theta_{k-1|k}(x_{k-1}|x_k)=\mathcal{N}(x_{k-1};\sqrt{1-\alpha_k}x_k+\sigma^2(1-\sqrt{1-\alpha_k}) f_\theta(k,x_k),\sigma^2 \alpha_k I).
\end{equation}
where $f_\theta(k,x_k)$ is differentiable w.r.t. $\theta$.

\subsection{Hyperparameters}

\subsubsection{Fitted Hyperparameters}

To aid reproducibility we report all fitted hyperparameters for each of our methods and PIS across all experiments in Tables \ref{hype:base} and \ref{hype:logreg}.

\subsubsection{Optimisation Hyperparameters}

As mentioned in the experimental section across all experiments modulo the Funnel we use the Adam optimiser with a learning rate of $0.0001$ with no learning decay and $11000$ training iterations, for the rest of the optimisation parameters use the default settings as provided by the Optax library \citep{optax2020github} which are $b_1=0.9, b_2=0.999, \epsilon=10^{-8}$ naming as per \cite{Kingma:2014}.

From the \href{https://github.com/qsh-zh/pis}{github repository} of \cite{zhangyongxinchen2021path} we were only able to find hyperparameters reported for the Funnel distribution. In order to first reproduce their results we used the a learning rate of $0.005$ and a learning rate decay of $0.95$ as per their implementation, their results were initially not reproducible due to a bug in setting $\sigma_f=1$ despite comparing to methods at the less favourable values of $\sigma_f=3$. For $\sigma_f=1$ we were able to reproduce their results. However we report results at $\sigma_f=3$ as this is the traditional value used for this loss. As no other optimisation configuration files were reported we used the more conservative learning rate of $0.0001$ since PIS was very unstable for $0.005$ with decay $0.95$ across many of our tasks. Finally we would like to clarify that the exact same optimiser settings where used for both PIS and DDS in order to ensure a fair comparison.

\subsubsection{Drift and Gradient Clipping}

We follow the same gradient clipping as in \cite{zhangyongxinchen2021path} that is :
\begin{align}
f_\theta(k, x) \!=\!\textcolor{blue}{ \mathrm{clip}\Big(}\mathrm{NN}_1(k, x;\theta) + \mathrm{NN}_2(k;\theta) \odot \textcolor{red}{\mathrm{clip}}\textcolor{red}{\big(}\nabla \ln  \pi(x)\textcolor{red}{,-10^2, 10^2 \big)}  \textcolor{blue}{, -10^4, 10^4  \Big)}
\end{align}

\begin{landscape}

\begin{table}[]
\adjustbox{max totalheight=2cm}{

\begin{tabular}{@{}llllllllll@{}}
\toprule
        & \multicolumn{3}{l}{funnel}                                                               & \multicolumn{3}{l}{lgcp}                                                           & \multicolumn{3}{l}{ion}                                                                    \\ \midrule
        & DDS                            & PIS              & UDMP                                 & DDS                         & PIS              & UDMP                              & DDS                             & PIS              & UDMP                                  \\
$K=64$  & $\sigma=1.075,\alpha=1.075$    & $\sigma=1.068$   & $\sigma= 1.85, \alpha= 1.67, m= 0.9$ & $\sigma= 2.1, \alpha=1.50$  & $\sigma= 1.068$  & $\sigma=1.1, \alpha=2.5, m= 0.4$  & $\sigma = 0.688, \alpha= 1.463$ & $\sigma =0.253$  & $\sigma= 0.6, \alpha= 3.85, m= 0.600$ \\
$K=128$ & $\sigma=1.075, \alpha= 0.6875$ & $\sigma= 0.416$  & $\sigma=1.85, \alpha= 3.7, m=0.9$    & $\sigma= 2.1, \alpha=0.75$  & $\sigma = 0.742$ & $\sigma=1.4, \alpha=2.5, m=0.4$   & $\sigma =0.3, \alpha=1.075$     & $\sigma= 0.09$   & $\sigma=0.6, \alpha =3.85, m = 0.600$ \\
$K=256$ & $\sigma=1.85, \alpha=0.3$      & $\sigma= 0.742$  & $\sigma=1.075, \alpha = 2.5, m= 0.9$ & $\sigma= 2.1, \alpha=0.900$ & $\sigma = 0.579$ & $\sigma=1.4, \alpha=4.5, m = 0.4$ & $\sigma= 0.3, \alpha =0.688$    & $\sigma = 0.416$ & $\sigma= 0.6, \alpha=3.85, m= 1.0$    \\
$K=512$ & $\sigma=1.463, \alpha=0.3$     & $\sigma = 0.253$ & $\sigma =0.688, \alpha=3.7, m= 0.9$  & $\sigma =2.1, \alpha=1.500$ & $\sigma= 0.416$  & $\sigma=1.7, \alpha=4.5, m= 0.4$  & $\sigma =0.688, \alpha =0.688$  & $\sigma= 0.253$  & $\sigma=0.6, \alpha=3.85, m=1.0$      \\ \bottomrule
\end{tabular}
}
\caption{Fitted hyperparameters for funnel, lgcp and ion experiments. \label{hype:base}}
\end{table}

\begin{table}[]
\adjustbox{max totalheight=1.7cm}{
\begin{tabular}{@{}lllllllllllll@{}}
\toprule
        & \multicolumn{3}{l}{lr\_sonar}                                                      & \multicolumn{3}{l}{vae}                                 & \multicolumn{3}{l}{brownian}                                                       & \multicolumn{3}{l}{nice}                                                           \\ \midrule
        & DDS                         & PIS             & UDMP                               & DDS                           & PIS              & UDMP & DDS                        & PIS              & UDMP                               & DDS                          & PIS             & UDMP                              \\
$K=64$  & $\sigma= 0.3, \alpha=1.650$ & $\sigma= 0.253$ & $\sigma= 1.15, \alpha=1.7, m= 2.2$ & $\sigma= 0.61, \alpha=2.2$    & $\sigma=0.253$   & NA   & $\sigma= 0.1, \alpha=2.35$ & $\sigma =0.084$  & $\sigma=0.115, \alpha=4.8, m=2.2$  & $\sigma= 1.5, \alpha =2.125$ & $\sigma =0.75$  & $\sigma=1.2, \alpha= 1.0, m= 0.9$ \\
$K=128$ & $\sigma=0.3, \alpha=1.2$    & $\sigma= 0.253$ & $\sigma= 0.55, \alpha=1.7, m= 3.1$ & $\sigma 0.61, \alpha=1.670$   & $\sigma= 0.2523$ & NA   & $\sigma =0.1, \alpha=1.8$  & $\sigma= 0.093$  & $\sigma=0.115, \alpha=4.8, m=2.2$  & $\sigma= 1.5, \alpha =1.75$  & $\sigma =0.588$ & $\sigma=1.2, \alpha= 1.0 m=1.65$  \\
$K=256$ & $\sigma=0.3, \alpha=0.75$   & $\sigma= 0.253$ & $\sigma=0.55, \alpha= 2.9, m= 2.2$ & $\sigma= 0.61, \alpha=1.140$  & $\sigma= 0.506$  & NA   & $\sigma =0.1, \alpha=2.35$ & $\sigma= 0.043$  & $\sigma=0.115, \alpha=3.75, m=2.2$ & $\sigma= 1.5, \alpha =1.75$  & $\sigma =0.425$ & $\sigma=1.2, \alpha= 2.5, m=0.9$  \\
$K=512$ & $\sigma=0.3, \alpha=0.75$   & $\sigma= 0.253$ & $\sigma= 0.55, \alpha=2.9, m=3.1$  & $\sigma =0.61, \alpha =1.670$ & $\sigma=0.416$   & NA   & $\sigma =0.1, \alpha=1.8$  & $\sigma= 0.0408$ & $\sigma=0.115, \alpha=4.8, m=2.2$  & $\sigma= 1.5, \alpha =2.5$   & $\sigma =0.263$ & $\sigma=1.2, \alpha= 2.5, m=1.65$ \\ \bottomrule
\end{tabular}
}
\caption{Fitted hyperparameters for sonar, brownian, vae and nice experiments. \label{hype:logreg}}
\end{table}

\end{landscape}

\end{document}